\documentclass{article}

\usepackage[preprint]{neurips_2026}

\usepackage{microtype}
\usepackage{graphicx}
\usepackage{subcaption}
\usepackage{booktabs} 
\usepackage{amsthm}
\usepackage{thmtools,thm-restate}
\usepackage{enumitem}
\usepackage{multirow}
\usepackage{subcaption}
\usepackage{graphicx}
\usepackage[table]{xcolor}
\usepackage{bbm}

\usepackage[utf8]{inputenc} 
\usepackage[T1]{fontenc}    
\usepackage{hyperref}       
\usepackage{url}            
\usepackage{booktabs}       
\usepackage{amsfonts}       
\usepackage{nicefrac}       
\usepackage{microtype}      
\usepackage{xcolor}         

\usepackage{amsmath}
\usepackage{amssymb}
\usepackage{mathtools}
\usepackage{float}
\usepackage{algorithm}
\usepackage{algorithmic}
\usepackage{cancel}

\usepackage[capitalize,noabbrev]{cleveref}

\theoremstyle{plain}
\newtheorem{theorem}{Theorem}[section]
\newtheorem{proposition}[theorem]{Proposition}
\newtheorem{lemma}[theorem]{Lemma}
\newtheorem{corollary}[theorem]{Corollary}
\theoremstyle{definition}
\newtheorem{definition}[theorem]{Definition}

\theoremstyle{remark}

\newcommand{\meanadv}{\bar{A}^{\mathrm{info}}}

\usepackage[textsize=tiny]{todonotes}

\newcommand{\methodtitle}{Information-Time Proximal Policy Optimization}
\newcommand{\methodname}{Information-Time Proximal Policy Optimization}
\newcommand{\methodabb}{InfoPPO}

\title{\methodtitle{}}

\author{
  {\bf Yongcheng Zeng$^{1,2}$},  
  {\bf Xinyu Cui$^{1}$},
  {\bf Yan Song$^{3,4}$},
  {\bf Guoqing Liu$^{5}$},\\
  {\bf Hongsheng Xin$^{6}$},
  {\bf Kaike Zhang$^{6}$},
  {\bf Cheng Deng$^{7}$},
  {\bf Kun Zhan$^{6}$},\\
  {\bf Jian Ying$^{6}$},
  {\bf Jian Zhao$^8$},
  {\bf Haifeng Zhang$^{1,\dagger}$},
  {\bf Jun Wang$^{3,\dagger}$} \\
  \vspace{0.05 cm}\\
  {\normalsize $^1$ Institute of Automation, Chinese Academy of Sciences, China}\\
  {\normalsize $^2$ School of Artificial Intelligence, University of Chinese Academy of Sciences, China} \\
  {\normalsize $^3$ University College London} 
  {\normalsize $^4$ AI Lab, The Yangtze River Delta} \\
  {\normalsize $^5$ Microsoft Research AI4Science}
  {\normalsize $^6$ Li Auto}
  {\normalsize $^7$ The University of Edinburgh}\\
  {\normalsize $^8$ Zhongguancun Institute of Artificial Intelligence}
}

\begin{document}

\maketitle

\begin{abstract}
Reinforcement Learning with Verifiable Rewards (RLVR) has substantially improved the reasoning capabilities of Large Language Models (LLMs). However, existing methods typically parameterize temporal progression in the Markov Decision Process (MDP) by token-by-token generation, despite the highly non-uniform information flow along autoregressive trajectories. In this paper, we propose \methodname{} (\methodabb{}), which reparameterizes temporal progression using information density, so that temporal distance is measured by accumulated information rather than raw token count. 
This reparameterization induces a common state-dependent structure for both temporal credit propagation and policy updates. By measuring temporal progression through accumulated information rather than token count, \methodabb{} restores the effectiveness of non-trivial discounting in long-horizon reasoning, retaining effective-horizon contraction in information time while avoiding excessive attenuation of terminal supervision over long token sequences. Moreover, the information-time policy-improvement analysis naturally leads to a state-dependent update constraint, which we implement through adaptive clipping. By adapting the clipping threshold at each token position to the information density of its corresponding state, this mechanism enables more targeted policy
updates while preserving proximal control.
Theoretically, we extend performance-difference and policy-improvement analyses to the information-time MDP, deriving a policy-improvement lower bound when policy changes are regulated by information density. We further connect the general information-time analysis to practical LLM
policy optimization by relating state-wise information density to local policy
movement, while also providing theoretical grounding for the adaptive update
mechanism.
Experiments on Qwen3 models demonstrate consistent gains over competitive
baselines across five challenging competition-style mathematical reasoning benchmarks. \methodabb{} also maintains stable accuracy and response length across non-trivial discount settings under which token-time PPO deteriorates.
\end{abstract}

\section{Introduction}
Although Reinforcement Learning with Verifiable Rewards (RLVR) has substantially improved the performance of Large Language Models (LLMs) on challenging reasoning tasks \citep{guo2025deepseek,yang2025qwen3,jaech2024openai,shao2024deepseekmath,yu2025dapo}, reliably optimizing policies over long reasoning trajectories remains challenging. Existing RLVR methods commonly formulate each generated token as one transition of the Markov Decision Process (MDP), thereby measuring temporal progression by raw token count. This is a natural and widely adopted modeling choice. However, this formulation does not explicitly capture how information content varies across token transitions. This raises the question of whether token count alone provides the most suitable temporal metric for long-horizon LLM reasoning. Under sparse outcome-level supervision, this choice also makes non-trivial discounting difficult to apply: discounting at every token can cause terminal rewards to decay rapidly over long sequences, whereas removing discounting forgoes its control over the effective horizon. This tension motivates us to consider whether temporal progression in LLM reinforcement learning can be parameterized by a measure other than token count alone.

One natural basis for such an alternative is the non-uniform information structure of autoregressive trajectories. In conventional RL domains such as robotic control and games, temporal progression is typically indexed by successive environment interactions, with each transition advancing the discrete-time index by one step \citep{brockman2016openai,schulman2015trust,schulman2017proximal}. When this convention is applied to language generation, each generated token is likewise treated as one unit of temporal progression. Yet token progression need not match information progression: many locally predictable tokens, including common function words such as articles, contribute mainly to local syntactic or semantic continuity, whereas a smaller number of semantic transition points can substantially redirect the subsequent reasoning trajectory \citep{wang2025beyond,qu2025dynamic,fu2025deep}. The resulting information landscape therefore interleaves locally sparse and dense regions. In sparse regions, continuations are relatively constrained, whereas dense regions involve a broader range of plausible continuations and larger increments of information. Consequently, token spans of equal length can accumulate markedly different amounts of reasoning-relevant information, even though token-based discounting applies the same cumulative decay to each. Raw token count thus captures sequence distance, but not the information distance traversed along the sequence. This motivates parameterizing temporal progression by accumulated information rather than token count.

Building on this analysis, we propose \methodname{} (\methodabb{}), 
which preserves the token-level state and action structure while reparameterizing the temporal measure of the MDP using state-wise information density. Temporal distance is therefore measured by accumulated information rather than raw token count. Under this formulation, discounting and trace decay operate over information time, retaining effective-horizon contraction while mitigating the excessive attenuation of terminal supervision over long token sequences. The corresponding policy-improvement analysis further motivates constraining policy movement in proportion to information density, which \methodabb{} implements through adaptive clipping. By adapting the clipping range to local information density, this mechanism improves sample efficiency and accelerates policy learning. The state-wise information structure thus provides a common basis for temporal credit propagation and policy-update regulation.

Theoretically, we extend the performance-difference lemma to the information-time MDP and derive a policy-improvement lower bound for policy updates constrained by information density. We further connect the general information-time analysis to
practical LLM policy optimization by relating state-wise information density to
local policy movement, while also providing theoretical grounding for the
adaptive update mechanism in \methodabb{}.
Empirically, \methodabb{} delivers consistent gains over PPO and DAPO-based baselines across model scales and five challenging competition-style mathematical reasoning benchmarks, while maintaining stable response lengths across training and remaining robust to non-trivial discounting that substantially degrades token-time PPO.

\section{Related Works}
\textbf{Reinforcement Learning with Verifiable Rewards (RLVR).} 
RLVR has emerged as a powerful post-training paradigm for improving the reasoning capabilities of LLMs by using rule-based verification functions to generate binary reward signals \citep{guo2025deepseek,yang2025qwen3,jaech2024openai,lambert2024t}.
Within this paradigm, Proximal Policy Optimization (PPO) \citep{schulman2017proximal} remains a central framework for policy optimization and constitutes the most direct point of comparison for our method.
Beyond PPO, several methods have adapted policy optimization to the specific demands of LLM reasoning.
Group-based methods such as GRPO \citep{shao2024deepseekmath} eliminate the dependency on a separate value network by sampling multiple outputs per query and computing advantages via group-relative normalization, a strategy successfully validated by systems like DeepSeek-R1 \citep{guo2025deepseek}.
Concurrently, DAPO \citep{yu2025dapo} further improves performance and stability through refined engineering implementations. Recent studies have also reconsidered how policy updates should be allocated across a generated trajectory. DAPO-FT \citep{wang2025beyond} shows that a minority of high-entropy tokens account for a disproportionate share of effective policy learning and accordingly restrict gradient updates to these positions. At a coarser granularity, GSPO \citep{zheng2025group} shifts importance weighting and clipping from individual tokens to entire sequences. 
Collectively, these methods refine how learning signals are estimated, sampled, and allocated, while generally retaining token-by-token generation as the temporal parameterization of the underlying MDP. In contrast, InfoPPO preserves the proximal policy-optimization structure of PPO while reparameterizing temporal progression through state-wise information density, thereby enabling effective-horizon control over long reasoning trajectories and supporting more sample-efficient policy learning.

\textbf{Temporal Abstraction and Information Dynamics.}
Reinforcement learning has long explored temporal abstraction as a way to move beyond uniformly indexed primitive transitions.
Foundational frameworks such as Semi-Markov Decision Processes (SMDPs) and the options framework \citep{sutton1999between,precup2000temporal,bacon2017option} introduced temporal abstraction by aggregating primitive transitions into temporally extended actions. 
At the representation level, recent language modeling work, such as Dynamic Large Concept Models (DLCM) \citep{qu2025dynamic}, also explores adaptive semantic abstraction beyond fixed token-level representations.
These approaches typically require higher-level actions, units, or boundaries to be specified or learned, which is difficult in the semantically fluid process of autoregressive generation.
Beyond explicit temporal abstraction, another complementary line of work generalizes geometric discounting through weighted or transition-based discount factors, allowing temporal weighting to vary across states or transitions \citep{feinberg1994markov,white2017unifying}. \methodabb{} is related to this line of work, but grounds non-uniform temporal progression specifically in state-wise information density and extends the same structure to policy-update regulation. Unlike explicit temporal-abstraction methods, \methodabb{} preserves token-level states and actions, capturing non-uniform temporal progression without requiring explicit segmentation or macro-action modeling.

\section{Preliminaries}

\textbf{Markov Decision Process.}
We formulate the language generation process as a discounted Markov Decision Process (MDP), denoted by the tuple $\mathcal{M}=(\mathcal{S},\mathcal{A},\mathcal{P},r,\gamma)$.
The state $s_t \in \mathcal{S}$ represents the generated sequence prefix, while the action $a_t \in \mathcal{A}$ corresponds to the next token selected from the vocabulary $\mathcal{V}$.
The transition dynamics $\mathcal{P}:\mathcal{S} \times \mathcal{A} \times \mathcal{S} \rightarrow \mathbb{R}$ are deterministic, given by $s_{t+1} = s_t \oplus a_t$, where $\oplus$ denotes concatenation. $r:\mathcal{S}\rightarrow \mathbb{R}$ is the reward function and $\gamma$ is the discount factor. For convenience, we write $r_t\coloneqq r(s_t)$. Given a policy $\pi$, the objective is to maximize the expected discounted return $\eta(\pi) = \mathbb{E}_{\tau \sim \pi}[\sum_{t=0}^{\infty} \gamma^t r(s_t)]$.
Accordingly, we define the state-value function $V_\pi(s_t) = \mathbb{E}_\pi[\sum_{k=0}^\infty \gamma^k r(s_{t+k}) | s_t]$ and the action-value function $Q_\pi(s_t, a_t)= \mathbb{E}_\pi[\sum_{k=0}^\infty \gamma^k r(s_{t+k}) | s_t,a_t]$, yielding the advantage function $A_\pi(s, a) \coloneqq Q_\pi(s, a) - V_\pi(s)$.
Finally, we define the (unnormalized) discounted state visitation distribution as $d_\pi(s) \coloneqq \sum_{t=0}^{\infty}\gamma^t P(s_t=s | \pi)$.

In many sparse reward settings, such as reasoning or question-answering tasks in LLMs, the agent only receives a non-zero signal $R_T$ at the terminal timestep $T$, i.e.,
\begin{equation}
    r(s_t) = 
\begin{cases}
R_T, & t=T \\
0, & t<T 
\end{cases}
\label{eq:sparse_reward}
\end{equation}
Here, the terminal reward $R_T$ provides outcome-level supervision for the entire trajectory.

\textbf{Policy Optimization.}
The \textit{Performance Difference Lemma} \citep{kakade2002approximately} provides the theoretical foundation for monotonic policy improvement. It expresses the return difference between two policies, $\pi$ and $\tilde{\pi}$, as:
\begin{equation}
    \eta(\tilde{\pi}) - \eta(\pi) = \mathbb{E}_{\tau \sim \tilde{\pi}} \left[ \sum_{t=0}^\infty \gamma^tA_\pi(s_t, a_t) \right] = \sum_{s}d_{\tilde{\pi}}(s)\sum_{a}\tilde{\pi}(a|s)A_{\pi}(s,a).
\end{equation}

In practice, to render the optimization tractable, we employ a local approximation by replacing $d_{\tilde{\pi}}$ with $d_{\pi}$, yielding the following surrogate objective function:
\begin{align}
L_{\pi}(\tilde{\pi}) = \eta(\pi) + \sum_s d_{\pi}(s) \sum_a \tilde{\pi}(a | s) A_\pi(s,a). \label{eq:adv}
\end{align}
TRPO \citep{schulman2015trust} 
derives the following lower bound for monotonic policy improvement:
\begin{equation}
\begin{aligned}
\eta(\tilde{\pi}) \ge L_{\pi}(\tilde{\pi}) - \frac{4C \gamma  }{(1-\gamma)^2} \alpha^2,
\end{aligned}
\end{equation}
where $C = \max_{s,a}|A_\pi(s,a)|$ and $\alpha=\max\limits_s D_{\text{TV}}(\pi(\cdot|s)\|\tilde{\pi}(\cdot|s))$. In practice,
TRPO constrains the magnitude of the policy update using an expected KL-divergence trust region, optimizing the following surrogate objective:
\begin{equation}
\begin{gathered}
    \max_{\theta} \; \mathbb{E}_{\pi_{\theta_{\text{old}}}} \left[ \frac{\pi_\theta(a_t|s_t)}{\pi_{\theta_{\text{old}}}(a_t|s_t)} {A}_{\theta_{\text{old}}}(s_t,a_t) \right] \\
    \text{s.t.} \quad \mathbb{E}_{\pi_{\theta_{\text{old}}}} \left[ D_{\mathrm{KL}}(\pi_{\theta_{\text{old}}}(\cdot|s_t) \| \pi_\theta(\cdot|s_t)) \right] \le \delta,
\end{gathered}
\end{equation}
where $\pi_\theta$ is the parameterized policy, and $\theta_{\text{old}}$ denotes the policy parameters before the update.

PPO \citep{schulman2017proximal} simplifies this optimization by using a clipped surrogate objective 
to approximately constrain policy updates:
\begin{equation}
\begin{aligned}
   \mathcal{L}^{\mathrm{CLIP}}(\theta) = &\mathbb{E}_{\pi_{\theta_{\text{old}}}} \Big[ \min \big( \omega_t(\theta) {A}_{\theta_{\text{old}}}, \mathrm{clip}(\omega_t(\theta), 1-\epsilon_{\text{low}}, 1+\epsilon_{\text{high}}) {A}_{\theta_{\text{old}}} \big) \Big],
\end{aligned}
\end{equation}
where $\omega_t(\theta) = \frac{\pi_\theta(a_t|s_t)}{\pi_{\theta_{\text{old}}}(a_t|s_t)}$ represents the importance ratio, $\epsilon_{\text{low}}$ and $\epsilon_{\text{high}}$ are the clipping hyperparameters.

\textbf{Advantage Estimation.}
Generalized Advantage Estimation (GAE) \citep{schulman2015high} is used to estimate the advantage from a parameterized value function
$V_\phi$. 
Given the TD residual $\delta_t = r_t + \gamma V_\phi(s_{t+1}) - V_\phi(s_t)$, the GAE estimator is given by $A_{t}^{\text{GAE}(\gamma,\lambda)}(s_t,a_t) = \sum_{l=0}^{\infty} (\gamma \lambda)^l \delta_{t+l}$, where $\lambda \in (0,1]$ controls the bias-variance trade-off.

\section{Methodology}
In this section, we develop \methodname{} (\methodabb{}) in three stages. First, we identify the temporal misalignment induced by uniform token time and formulate an Information-Time MDP that preserves token-level states and actions while measuring trajectory progression through accumulated state-dependent information. Second, we extend the performance-difference and policy-improvement analyses to this non-uniform temporal measure, deriving a general lower bound when policy updates are constrained by information density. Finally, we specialize the framework to LLM policy optimization by using the old policy's uncertainty over next-token continuations as a tractable proxy for local information density. We empirically examine how this proxy varies along reasoning trajectories and theoretically show that the induced policy movement satisfies the update condition required by the general lower bound. Because the proxy is fixed within each update but recomputed across iterations,  
we further bound the effect of this recomputation on the information-time return.
This specialization yields the practical \methodabb{} algorithm, combining information-time advantage estimation with proximal updates guided by information density.

\subsection{Information-Time Markov Decision Process}

\begin{figure*}[t]
\begin{center}
\centerline{\includegraphics[width=\textwidth]{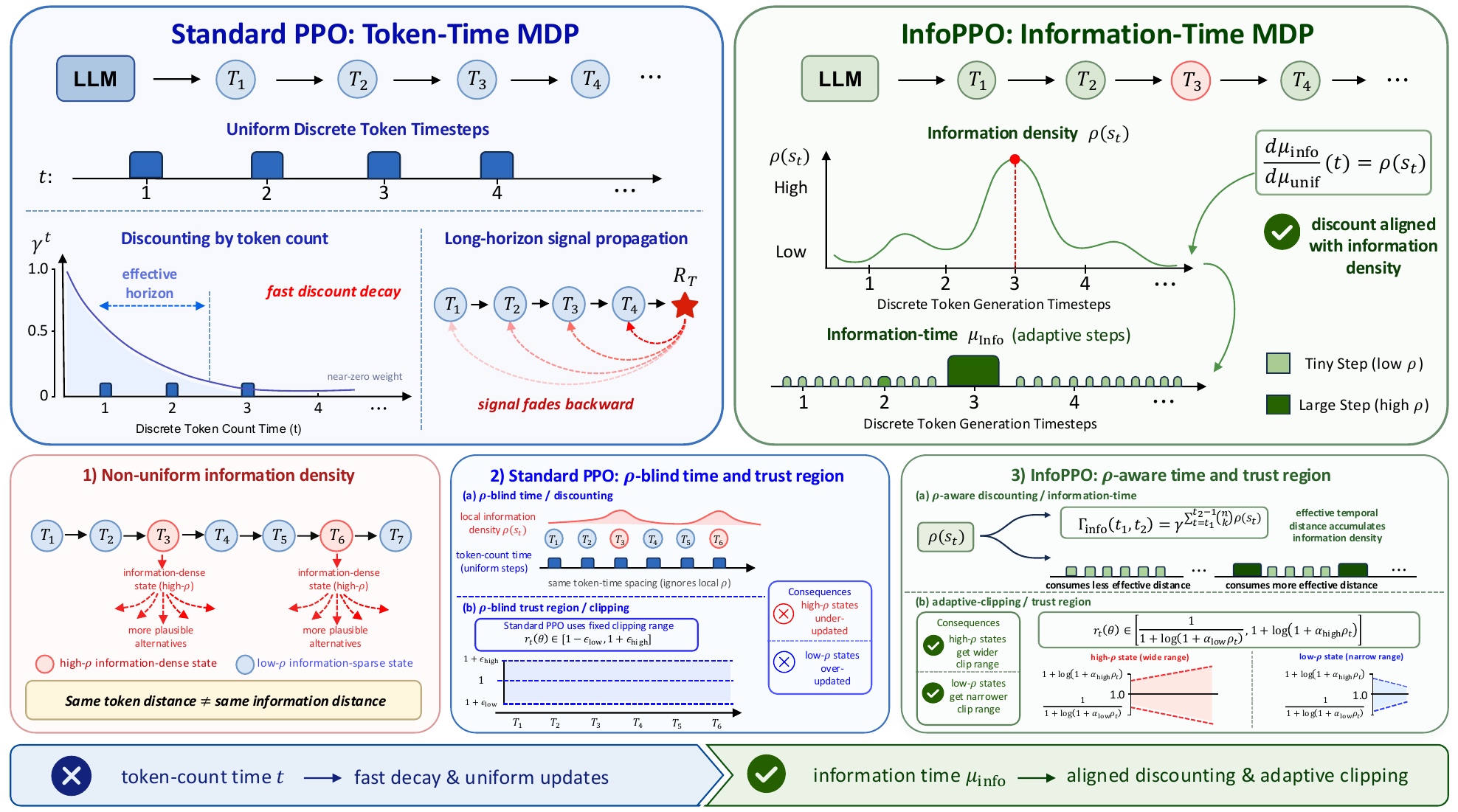}}
\caption{Comparison of standard PPO and \methodabb{}. PPO parameterizes temporal progression by uniform token steps and applies fixed clipping bounds across token positions. In long-horizon reasoning, this creates a tension between long-range credit propagation and effective-horizon contraction. In contrast, \methodabb{} preserves token-level transitions while measuring temporal distance through accumulated state-dependent information. This reparameterization retains effective-horizon contraction while moderating long-range attenuation. The same information density also adapts the clipping range across states, improving sample efficiency and accelerating policy learning.}
\label{method}
\end{center}
\vskip -0.2in
\end{figure*}

\paragraph{Temporal Misalignment under Uniform Token Time.}
Conventional RL formulations treat autoregressive token generation as an MDP over a uniform time measure, denoted as $\mu_{\text{unif}}$. Under this paradigm, each generated token advances the temporal coordinate by one unit, $\Delta\mu_{\text{unif}}(t)=1$, making temporal distance proportional to token count. This construction implicitly assigns the same temporal increment to every token transition, irrespective of how the local continuation structure changes along the trajectory. At each state, however, the autoregressive policy induces a distribution over possible continuations, and the structure of this distribution can vary substantially across the trajectory \citep{wang2025beyond,fu2025deep}. Some transitions occur where the continuation is relatively constrained, whereas others occur where a broader set of plausible continuations remains. The same unit token step can therefore correspond to different increments in the information traversed by the trajectory. Consequently, equal token distances need not imply equal information distances. We refer to this mismatch between token distance and information distance as \textbf{\textit{Temporal Misalignment}}.

Temporal Misalignment becomes consequential when temporal decay is defined over the same token-based coordinate.
Because discounting and trace decay accumulate with temporal distance, under uniform token time, their cumulative effect is determined by raw token distance rather than the information traversed along the trajectory.
In long-horizon reasoning, this creates a tension between long-range credit propagation and effective-horizon contraction: non-trivial token-wise decay can excessively attenuate terminal supervision over long sequences, whereas removing temporal decay forfeits effective-horizon contraction. This suggests that the issue is not necessarily temporal decay itself, but the coordinate over which that decay accumulates.
Rather than assigning every token transition the same temporal increment, we therefore seek a temporal coordinate whose local progression varies with the information associated with each state.

\paragraph{The Information-Time MDP Construction.}
To formalize this state-dependent temporal coordinate, we augment the standard MDP formulation to the Information-Time MDP, denoted by the tuple $\mathcal{M}_{\text{info}} = ( \mathcal{S}, \mathcal{A}, \mathcal{P}, r, \gamma, \rho )$, where $\rho:\mathcal{S}\rightarrow[0,1]$ quantifies the instantaneous information density of each state, reflecting the local intensity of information change along the trajectory. For example, in RL for LLM reasoning, the state corresponds to the self-generated context, while the old policy defines a predictive distribution over the next token at each visited state. 
Since the old policy remains fixed during the current update, this distribution provides a stable characterization of the information profile of the state.
This motivates using quantities derived from the old policy distribution, such as next-token uncertainty, as proxies for the information density $\rho(s_t)$. The details are provided in \cref{section:implementation}.

To formalize how state-wise information density reparameterizes temporal progression over discrete token steps, we define the information-time measure $\mu_{\text{info}}$ with respect to the uniform time measure $\mu_{\text{unif}}$:
 \begin{equation}
 \frac{d\mu_{\text{info}}}{d\mu_{\text{unif}}}(t) = \rho(s_t).
 \end{equation}
Since each token transition contributes a unit increment under $\mu_{\text{unif}}$, this gives $\Delta\mu_{\text{info}}(t)=\rho(s_t)$.
 
Equipped with this measure transformation, we proceed to reformulate the fundamental RL components, specifically the objective and value functions. To facilitate the derivation, let us first revisit the standard objective function $\eta(\pi)$ defined under the uniform time measure $\mu_{\text{unif}}$:
\begin{equation}
\eta(\pi) = \mathbb{E}_{\tau\sim\pi}\left[\sum_{t=0}^\infty \Gamma(0,t) r(s_t) \right],
\end{equation}
where the cumulative discount factor is defined as $\Gamma(t_1, t_2) \coloneqq\gamma^{\sum_{t=t_1}^{t_2-1} \Delta\mu_{\text{unif}}(t)}$. Since $\Delta\mu_{\text{unif}}(t)=1$, this recovers the standard form $\Gamma(t_1,t_2) = \gamma^{t_2-t_1}$.

Replacing the uniform time measure with the information-time measure, we define the corresponding objective as
\begin{equation}
\eta^{\text{info}}(\pi) = \mathbb{E}_{\tau\sim\pi}\left[\sum_{t=0}^\infty \Gamma^{\text{info}}(0,t) r(s_t)\right],
\end{equation}
Here, the information-time discount factor is defined as $\Gamma^{\text{info}}(t_1, t_2) \coloneqq\gamma^{\sum_{t=t_1}^{t_2-1} \Delta\mu_{\text{info}}(t)}$, specifically, this is equivalent to $\Gamma^{\text{info}}(t_1, t_2) =\gamma^{\sum_{t=t_1}^{t_2-1} \rho(s_{t})}$. For $\gamma=1$, the information-time objective reduces to the original
undiscounted RLVR objective, whereas for $\gamma<1$, information time determines
how temporal decay is accumulated along the trajectory.

Analogously, we define the value function $V_{\pi}^{\text{info}}$, action-value function $Q_{\pi}^{\text{info}}$, and advantage function $A_{\pi}^{\text{info}}$ with respect to the information-time measure $\mu_{\text{info}}$:
\begin{equation}
\begin{gathered}
\hspace{-0em} V_{\pi}^{\text{info}}(s_t) 
\coloneqq \mathbb{E}_{\pi}\!\left[\sum_{u=t}^{\infty}\Gamma^{\text{info}}(t,u){r}(s_{u}) \Big| s_t\right],\nonumber\\
\hspace{-0em} Q_{\pi}^{\text{info}}(s_t,a_t)
\coloneqq \mathbb{E}_{\pi}\!\left[\sum_{u=t}^{\infty}\Gamma^{\text{info}}(t,u){r}(s_{u}) \Big| s_t,a_t\right],\nonumber\\
\hspace{-0em} A_{\pi}^{\text{info}}(s,a) \coloneqq Q_{\pi}^{\text{info}}(s,a)-V_{\pi}^{\text{info}}(s).\nonumber
\end{gathered}
\end{equation}
When $\rho(s) \equiv 1$ for all $s \in \mathcal{S}$, the information-time formulation reduces to the standard uniform-time formulation.

\subsection{Policy Improvement on Information-Time MDP}
\label{sec:info_policy_improvement}
To analyze policy updates in the Information-Time MDP, we first extend the classic performance difference lemma \citep{kakade2002approximately} to the information-time setting.

 \begin{restatable}{lemma}{lemmaPD}\textup{(Information-Time Performance Difference).}\label{lemma:im_performance_diff}
Given two policies $\pi$ and $\tilde{\pi}$ within the Information-Time MDP framework, the following identity holds:
\begin{equation}
\hspace{-0.8em}
\begin{aligned}
\eta^{\text{info}}&(\tilde{\pi}) - \eta^{\text{info}}(\pi) = \mathbb{E}_{\tau \sim \tilde{\pi}}\!\!\left[\sum_{t=0}^{\infty} \Gamma^{\text{info}}(0,t) A^{\text{info}}_{\pi}(s_t, a_t)\right].
\end{aligned}
\label{eq:PD}
\end{equation}
  \end{restatable}
The detailed proof is provided in Appendix \ref{app_lemma:im_performance_diff}.

To express the information-time performance difference in state space,
we define the unnormalized information-discounted state visitation
measure
\begin{equation}
\nu_{\pi}^{\mathrm{info}}(s)
\coloneqq
\mathbb{E}_{\tau\sim\pi}
\left[
\sum_{t=0}^{\infty}
\Gamma^{\mathrm{info}}(0,t)
\mathbf{1}\{s_t=s\}
\right].
\label{eq:info_visitation}
\end{equation}
Unlike the standard discounted visitation measure, the information
discount remains inside the expectation because
$\Gamma^{\mathrm{info}}(0,t)$ depends on the information accumulated
along the sampled trajectory.

With this definition, Eq.~\eqref{eq:PD} can be equivalently written as
\begin{align}
\eta^{\mathrm{info}}(\tilde{\pi})
-
\eta^{\mathrm{info}}(\pi)
=
\sum_s
\nu_{\tilde{\pi}}^{\mathrm{info}}(s)
\sum_a
\tilde{\pi}(a|s)
A_{\pi}^{\mathrm{info}}(s,a).
\label{eq:state_adv}
\end{align}
The derivation is provided in Appendix~\ref{app:info_state_representation}.

This formulation implies that if the condition $\sum_a \tilde{\pi}(a|s) A_{\pi}^{\text{info}}(s,a) \ge 0$ holds for all states $s$, 
the update from $\pi$ to $\tilde{\pi}$ guarantees non-decreasing information-time return.
However, direct optimization of this expression is difficult because
the state visitation measure depends on the candidate policy
$\tilde{\pi}$. We therefore replace
$\nu_{\tilde{\pi}}^{\mathrm{info}}$ with the visitation measure induced
by the current policy $\pi$, yielding the local surrogate
\begin{align}
L_{\pi}^{\mathrm{info}}(\tilde{\pi})
=
\eta^{\mathrm{info}}(\pi)
+
\sum_s
\nu_{\pi}^{\mathrm{info}}(s)
\sum_a
\tilde{\pi}(a|s)
A_{\pi}^{\mathrm{info}}(s,a).
\label{eq:surrogate}
\end{align}

This replacement introduces an approximation error between the
surrogate objective and the true return of $\tilde{\pi}$.
To relate the surrogate to the information-time return, we need to control this error as the
candidate policy $\tilde{\pi}$ departs from the current policy $\pi$.
Unlike the standard discounted setting, however, the information-time
discount $\Gamma^{\text{info}}(0,t)$ depends on the accumulated
information density along the trajectory rather than on the token index
alone. As a result, $\Gamma^{\mathrm{info}}(0,t)$ no longer forms a standard
geometric sequence, and its cumulative sum cannot in general be reduced
to $(1-\gamma)^{-1}$ as in the uniform-time setting. The standard
geometric-series argument used to control the approximation error
therefore does not directly apply.

The underlying difference is that each transition now advances the
temporal coordinate by $\rho(s)$, rather than by a uniform unit step.
This state-dependent temporal increment motivates expressing the local
policy-update scale in terms of information density.
Specifically, we require the total-variation divergence between the
candidate policy $\tilde{\pi}$ and the current policy $\pi$ at state
$s$ to be constrained by its information density:
\begin{equation}
D_{\mathrm{TV}}
\left(
\tilde{\pi}(\cdot|s),
\pi(\cdot|s)
\right)
\le
\alpha \rho(s),
\qquad \forall s.
\label{eq:tv_divergence}
\end{equation}

Under this condition, states with lower information density admit smaller divergence
bounds, whereas states with higher information density admit larger ones.
Here, $\alpha$ determines the overall divergence scale.
We justify this condition for LLM policy optimization in
Section~\ref{section:implementation}.
We next show that this
state-wise divergence constraint provides a corresponding bound on the
surrogate approximation error.

\begin{restatable}{theorem}{theoremoneLB}
\label{theoremLB}
Let $\pi$ be the current policy and $\tilde{\pi}$ a candidate policy.
Suppose that the divergence condition in
Eq.~(\ref{eq:tv_divergence}) holds for some $\alpha \ge 0$.
Then, the following lower bound holds:
\begin{equation}
\begin{gathered}
\eta^{\text{info}}(\tilde{\pi}) \ge L_{\pi}^{\text{info}}(\tilde{\pi}) - \frac{4 C \alpha^2}{(1-\gamma)^2}, \\
 \text{where } C = \max_{s,a}|A_{\pi}^{\text{info}}(s,a)|.
\end{gathered}
\end{equation}
\end{restatable}

The proof is provided in Appendix~\ref{thm:lower_bound}.
It uses a coupling argument to relate the state-wise total-variation
bound to action disagreement between the two policies. Consequently, if the surrogate gain over the current policy
exceeds the penalty term in Theorem~\ref{theoremLB}, the candidate
policy is guaranteed to achieve a non-decreasing information-time
return. 
Section~\ref{section:implementation} extends this one-step result to the practical setting in which the information clock is induced by the old policy and refreshed across policy iterations.

\subsection{Practical Construction of the Information Clock}
\label{section:density_instantiation}
The analysis in Section~\ref{sec:info_policy_improvement} is stated in
the general infinite-horizon form of the Information-Time MDP. Episodic
language generation is naturally covered by this formulation by treating
the state reached after termination as absorbing, with zero continuation
reward and value. We now specialize this framework to the practical RLVR
setting considered in this work. Each generated response forms a finite
trajectory
$
\tau=(s_0,a_0,\ldots,s_T)
$, 
which terminates at the EOS action, and supervision is provided through
a terminal outcome reward $R_T$. 

\paragraph{Instantiating the Information Density.} The Information-Time MDP specifies the role of the state-wise information
density $\rho(s)$, but leaves its functional form unspecified. Because
$\rho(s_t)$ determines the information-time increment assigned to the
transition from $s_t$ to $s_{t+1}$, its practical instantiation should capture the local uncertainty that remains in resolving that transition under the current policy.
In autoregressive language generation, conditioned on the current prefix
$s_t$, the next transition is determined by the next-token choice
$\pi_{{\mathrm{old}}}(\cdot| s_t)$. The conditional
next-token distribution therefore provides a state-local characterization
of the uncertainty over the next transition. Shannon entropy provides a
direct scalar summary of this uncertainty, because it equals the expected
surprisal of the next token under the predictive distribution. Unlike the
surprisal of a realized token, entropy is determined from the full
predictive distribution before the next token is sampled and thus defines
a state-wise quantity within the current policy update, consistent with
the role of $\rho$. A concentrated distribution yields low entropy and
indicates a locally constrained continuation, whereas a dispersed
distribution yields high entropy and reflects a broader set of plausible
continuations. Moreover, entropy can be computed directly from the policy
distribution at the current state, varies continuously with the predictive distribution, and requires neither future trajectory information nor
additional supervision. We therefore use the normalized next-token entropy of the frozen old policy as a tractable proxy for $\rho(s_t)$.

\begin{figure*}[t]
    \centering


    \begin{subfigure}[b]{0.48\textwidth}
        \centering
        \includegraphics[width=\textwidth]{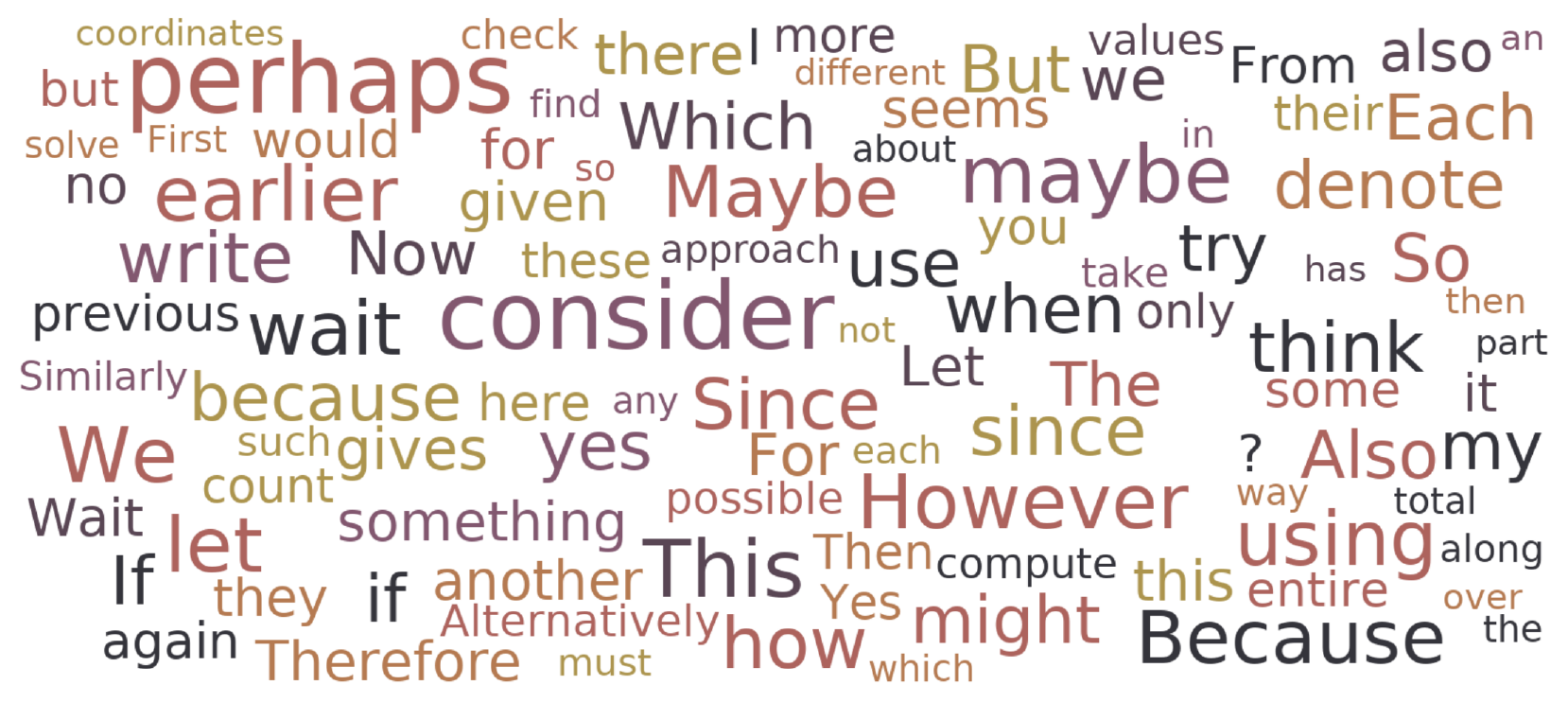}
        \vspace{-0.1in}
        \caption{Highest average predictive entropy.}
        \label{fig:high_entropy_wordcloud}
    \end{subfigure}
    \hfill
    \begin{subfigure}[b]{0.48\textwidth}
        \centering
        \includegraphics[width=\textwidth]{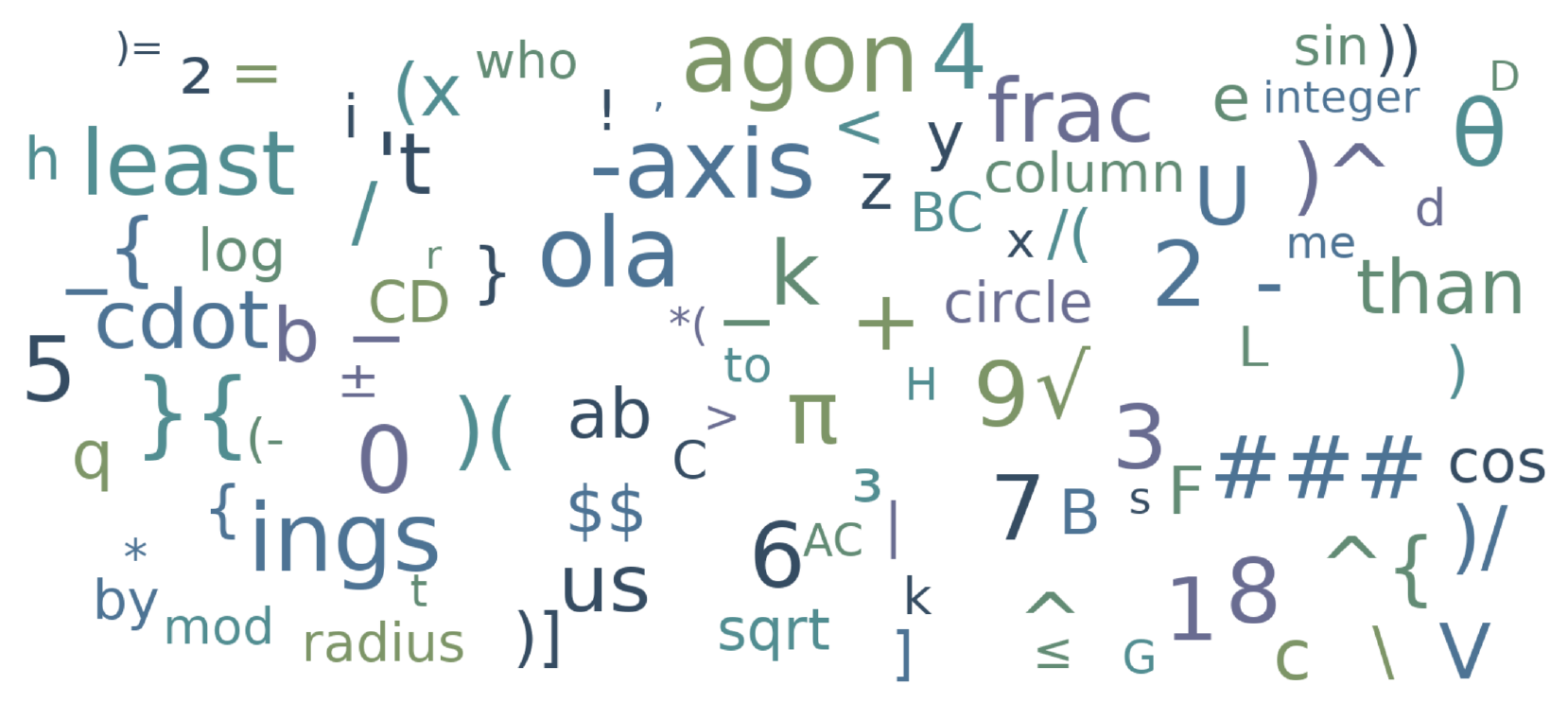}
        \vspace{-0.1in}
        \caption{Lowest average predictive entropy.} \label{fig:low_entropy_wordcloud}
    \end{subfigure}

    \caption{
\textbf{Predictive entropy patterns in mathematical reasoning.}
We visualize up to 100 tokens with the highest and lowest average predictive
entropy in Qwen3-4B, 8B, and 14B responses to AIME24 and AIME25.
To reduce noise in the average entropy estimates, we only include tokens
occurring more than 100 times in every model.} \label{fig:predictive_entropy_wordclouds}
\end{figure*}

Figure~\ref{fig:predictive_entropy_wordclouds} provides a qualitative view of
token patterns associated with different levels of predictive entropy in
mathematical reasoning. Tokens with higher average entropy tend to occur in
discourse and reasoning transitions, whereas lower entropy is more common in
mathematical symbols, numerals, and predictable subword fragments. This
contrast supports predictive entropy as a practical proxy for local information
density, since higher entropy reflects a broader set of plausible next-token
continuations, while lower entropy indicates more constrained transitions.
Figures~\ref{app_fig:res_entropy_1}--\ref{app_fig:res_entropy_6} further illustrate this local variation
along individual reasoning trajectories.

Formally, for each state $s_t$ visited under the frozen old policy,
we define its predictive entropy as
\begin{equation}
\mathcal H_{\mathrm{old}}(s_t) \coloneqq \mathcal{H}(\pi_{\text{old}}(\cdot|s_t))
=
-\sum_{a\in\mathcal V}
\pi_{{\mathrm{old}}}(a|s_t)
\log
\pi_{{\mathrm{old}}}(a|s_t),
\end{equation}
and instantiate the information density as
\begin{equation}
\rho(s_t)
=
\frac{\mathcal H_{\mathrm{old}}(s_t)}
{\mathcal H_{\max}},
\end{equation}
where $\mathcal H_{\max}>0$ is a normalization scale held fixed
throughout the current policy update, so that $\rho(s_t)\in[0,1]$
over the states under consideration. This construction makes $\rho$
well defined and fixed within a single policy update, while allowing it to adapt across policy iterations as the policy's predictive uncertainty changes. Connecting this practical
construction to the preceding analysis requires addressing two
additional points. Theorem~\ref{theoremLB} assumes that policy
divergence scales with $\rho(s)$, so this relation must first be
verified for the entropy construction. Moreover, the theorem compares
policies under a common information density, whereas the practical
algorithm recomputes $\rho$ after each policy update. We therefore also
need to quantify the change in information-time return introduced by
this recomputation.

We first verify the required divergence relation. For a fixed state $s$, the action-dependent
term of the surrogate in Eq.~\eqref{eq:surrogate} admits the equivalent
importance-ratio representation
\begin{equation}
\sum_{a}
\tilde{\pi}(a|s)
A_{\pi_{{\mathrm{old}}}}^{\mathrm{info}}(s,a)
=
\mathbb{E}_{a\sim\pi_{{\mathrm{old}}}(\cdot|s)}
\left[
\frac{\tilde{\pi}(a|s)}
{\pi_{{\mathrm{old}}}(a|s)}
A_{\pi_{{\mathrm{old}}}}^{\mathrm{info}}(s,a)
\right].
\label{eq:conditional_info_surrogate}
\end{equation}
Because the divergence condition in Eq.~\eqref{eq:tv_divergence} is imposed separately at each state, we characterize the policy movement induced by this conditional objective at a fixed $s$.

\begin{restatable}{proposition}{propositionEntropyMovement}
\label{proposition:entropy_policy_movement}
Suppose that $\pi_{\mathrm{old}}$ is parameterized by a softmax
distribution and that $\pi_{\mathrm{new}}$ is obtained by one local
gradient step on the conditional surrogate in
Eq.~\eqref{eq:conditional_info_surrogate}, with the local step size
bounded above by $\bar{\beta}$. Then, for every state $s$,
\begin{align}
D_{\mathrm{TV}}
\left(
\pi_{\mathrm{new}}(\cdot|s),
\pi_{\mathrm{old}}(\cdot|s)
\right)
\leq
\frac{\bar{\beta}C}{2}
\left(
1-\exp\{-\mathcal H_{\mathrm{old}}(s)\}
\right)
\leq
\frac{\bar{\beta}C\mathcal H_{\max}}{2}
\rho(s),
\label{eq:entropy_local_tv}
\end{align}
where $C = \max_{s,a}|A_{\pi_{\text{old}}}^{\text{info}}(s,a)|$. Hence, with $
\alpha=\bar{\beta}C\mathcal H_{\max}/{2}$,
the local update satisfies the state-dependent divergence scaling in
Eq.~\eqref{eq:tv_divergence}.
\end{restatable}
The proof is provided in
Appendix~\ref{app:proof_policy_optimization}.

Proposition~\ref{proposition:entropy_policy_movement} resolves the first
point by establishing the required divergence scaling for the entropy
construction. We now turn to the effect of recomputing the information
density after the policy update. Let $\pi_k$ denote the policy before
an update and let $\rho_k$ be the information density computed from it.
During the update from $\pi_k$ to $\pi_{k+1}$, $\rho_k$ remains fixed,
so Theorem~\ref{theoremLB} applies directly to the comparison between
these two policies under the same information density. Once the update
is complete, the density is recomputed from $\pi_{k+1}$, yielding
$\rho_{k+1}$. We therefore quantify the change in information-time
return caused solely by replacing $\rho_k$ with $\rho_{k+1}$ while
keeping $\pi_{k+1}$ fixed. For this
comparison, we write $\eta_{\rho}^{\mathrm{info}}(\pi)$ for the
information-time return of $\pi$ evaluated using density $\rho$.
The quantities $\eta_{\rho_k}^{\mathrm{info}}(\pi_{k+1})$ and
$\eta_{\rho_{k+1}}^{\mathrm{info}}(\pi_{k+1})$ therefore isolate the
effect of recomputing $\rho$ while keeping the policy fixed.

\begin{proposition}
\label{prop:recomputed_density_bound}
Suppose that the terminal reward is bounded by
$|R_T|\leq R_{\max}<\infty$.
Under the update from $\pi_k$ to $\pi_{k+1}$ in
Proposition~\ref{proposition:entropy_policy_movement}, the change in
information-time return caused by recomputing the information density
satisfies
\begin{equation}
\left|
\eta_{\rho_{k+1}}^{\mathrm{info}}(\pi_{k+1})
-
\eta_{\rho_k}^{\mathrm{info}}(\pi_{k+1})
\right|
\leq
\frac{8R_{\max}}{e}\alpha.
\label{eq:recomputed_density_perturbation_bound}
\end{equation}
\end{proposition}
Since $\alpha=\bar{\beta}C\mathcal H_{\max}/2$, the effect of
recomputing $\rho$ is directly controlled by the local step-size bound $\bar{\beta}$. Together with Theorem~\ref{theoremLB}, this separates the policy update performed with $\rho_k$ fixed from the subsequent variation introduced by recomputing $\rho$. The proof is provided in Appendix~\ref{app:recomputed_density_bound}.

The preceding result controls the perturbation caused by a single
information-clock refresh. Across multiple policy iterations, however,
the information density is repeatedly recomputed. To compare policy
iterates on a common basis, we evaluate them under the same reference
clock. The following proposition establishes a lower bound for this
common-clock comparison.

\begin{proposition}
\label{prop:cross_iteration_clock_ordering}
Let $\pi_x$ and $\pi_y$, with $x<y$, be two policy iterates from the
same sequence of updates, and let $\rho_r$ be any fixed
information-density snapshot used as a common reference clock.
Suppose that $|R_T|\le R_{\max}<\infty$. Define the cumulative certified gain between the two iterates as
\begin{equation}
\mathcal G_{x,y}
\coloneqq
\sum_{t=x}^{y-1}
\left[
L_{\pi_t}^{\mathrm{info}}(\pi_{t+1})
-
\eta_{\rho_t}^{\mathrm{info}}(\pi_t)
-
\frac{4C_t\alpha_t^2}{(1-\gamma)^2}
-
\frac{8R_{\max}}{e}\alpha_t
\right],
\label{eq:cumulative_certified_gain}
\end{equation}
where $C_t$ and $\alpha_t$ denote the corresponding quantities for
the update $\pi_t\rightarrow\pi_{t+1}$, and all information-time quantities in the $t$-th summand are evaluated under $\rho_t$. For a trajectory $\tau=(s_0,a_0,\ldots,s_T)$, define $
\bar\rho_j(\tau)
\coloneqq
\frac{1}{T}
\sum_{u=0}^{T-1}\rho_j(s_u),
$
and
$
\varepsilon_{x,y}^{(r)}
\coloneqq
\frac{1}{2}
\sum_{j\in\{x,y\}}
\mathbb E_{\tau\sim\pi_j}
\left[
\left|
\log
\frac{\bar\rho_r(\tau)}
{\bar\rho_j(\tau)}
\right|
\right]
$.
Then
\begin{equation}
\eta_{\rho_r}^{\mathrm{info}}(\pi_y)
-
\eta_{\rho_r}^{\mathrm{info}}(\pi_x)
\ge
\mathcal G_{x,y}
-
\frac{2R_{\max}}{e}
\varepsilon_{x,y}^{(r)}.
\label{eq:cross_iteration_clock_ordering}
\end{equation}
\end{proposition}
The proof is provided in
Appendix~\ref{app_prop:cross_iteration_clock_ordering}.
Here, $\varepsilon_{x,y}^{(r)}$ measures the average relative
discrepancy between the trajectory-level information densities of the
two endpoint policies and the common reference clock. Proposition~
\ref{prop:cross_iteration_clock_ordering} therefore shows that moderate
variation in the information clock induces only a limited perturbation
to comparisons between policy iterates, making their
relative comparison less sensitive to clock recomputation across
training. In the experiments, we further track the evolution of
normalized entropy throughout training to empirically assess how the
information clock evolves in practice.

\subsection{Information-Time Policy Optimization}
\label{section:implementation}

With the information clock specified, we incorporate
$\rho_t\coloneqq\rho(s_t)$ into advantage estimation and policy
optimization. Replacing uniform token-time decay in standard GAE with
information-time decay gives
\begin{equation}
\begin{aligned}
\delta_t^{\mathrm{info}}
&=
r_t+\gamma^{\rho_t}V_\phi(s_{t+1})-V_\phi(s_t),\\
\hat A_t^{\mathrm{info}}
&=
\delta_t^{\mathrm{info}}
+
(\gamma\lambda)^{\rho_t}
\hat A_{t+1}^{\mathrm{info}}.
\end{aligned}
\label{eq:info_gae}
\end{equation}
By defining the information-time trace decay as
$\Lambda^{\mathrm{info}}(t_1,t_2)
\coloneqq
\lambda^{\sum_{j=t_1}^{t_2-1}\rho_j}$,
the recursion can be equivalently written as
\begin{equation}
    \hat A_t^{\mathrm{info}}
=
\sum_{l=0}^{T-t-1}
\Gamma^{\mathrm{info}}(t,t+l)
\Lambda^{\mathrm{info}}(t,t+l)
\delta_{t+l}^{\mathrm{info}}.
\end{equation}
The information density also regulates the proximal policy update.
Theorem~\ref{theoremLB} and
Proposition~\ref{proposition:entropy_policy_movement}
motivate a state-dependent update scale that increases with $\rho_t$.
Following the PPO paradigm, we realize this update geometry through an adaptive clipping range. A direct choice would scale the bounds linearly
as $[1-\epsilon_{\mathrm{low}}^{\text{info}}\rho_t,
1+\epsilon_{\mathrm{high}}^{\text{info}}\rho_t]$. While this linear construction captures the desired dependence on information density, it can yield relatively permissive updates at high-information states. We therefore adopt a more conservative logarithmic scaling. Specifically, we define the upper and lower clipping bounds as $\xi_{\mathrm{high},t}^{\mathrm{info}}\coloneqq 1+\log\left(1+\epsilon_{\mathrm{high}}^{\mathrm{info}}\rho_t\right)$ and $\xi_{\mathrm{low},t}^{\mathrm{info}}\coloneqq\frac{1}{
1+\log\left(1+\epsilon_{\mathrm{low}}^{\mathrm{info}}\rho_t\right)
}$, respectively. Thus, the resulting clipping range expands with information density, while its logarithmic scaling moderates policy updates in high-information regions. Finally, we obtain the optimization objective of \methodabb{}:
\begin{equation}
\begin{aligned}
\mathcal{L}^{\mathrm{info\text{-}CLIP}}(\theta)
=
\mathbb{E}_{\pi_{\theta_\text{old}}}
\Bigg[
\min\Bigg(\omega_t(\theta)\hat A_t^{\mathrm{info}}, \operatorname{clip}\!\left(
\omega_t(\theta),
\xi_{\mathrm{low},t}^{\mathrm{info}},
\xi_{\mathrm{high},t}^{\mathrm{info}}
\right)
\hat A_t^{\mathrm{info}}
\Bigg)
\Bigg].
\end{aligned}
\label{eq:actor_loss}
\end{equation}

\section{Experiments}
We organize our experiments to examine three aspects of \methodabb{}. We first characterize the information-time quantities induced during training, including information density, cumulative information-time discounting, and state-dependent clipping. We then compare information-time and token-time PPO under non-trivial discount and trace-decay settings. Finally, we evaluate the complete method against representative RLVR baselines across model scales and mathematical reasoning benchmarks.

\textbf{Models and Training Settings.} We conduct our experiments on the Qwen3 series of models \citep{yang2025qwen3} with the DAPO-Math-17K dataset \citep{yu2025dapo} and compare \methodabb{} against three representative baselines: PPO \citep{schulman2017proximal}, DAPO \citep{yu2025dapo}, and DAPO-FT \citep{wang2025beyond}. For both \methodabb{} and PPO, the policy model is trained with a learning rate of $1e{-}6$ and a 10-step warmup, while the critic model employs a learning rate of $2e-6$ without warmup. 
For the reproduction of DAPO and DAPO-FT, we adhere to the official default configurations. The maximum response length is set to 10k tokens for the 4B and 8B models and 12k tokens for the 14B model, consistently during training and evaluation. Further details are provided in the Appendix~\ref{app_:implementation_details}.

\textbf{Evaluation.}
We evaluate the models on five challenging competition-style
mathematical reasoning benchmarks: AMC23~\citep{li2024numinamath},
AIME24~\citep{aime24}, AIME25~\citep{aime25},
AIME26~\citep{aime26}, and BeyondAIME~\citep{beyondaime2025}.
These benchmarks provide a challenging evaluation of mathematical
reasoning across competition-style problem settings.
All evaluations are performed in a zero-shot setting. For each prompt, we independently sample 16 responses with decoding temperature $T=1.0$ and
report the average accuracy, denoted as Mean@16.

\subsection{Empirical Characterization of the Information Clock}
\label{sec:empirical_information_time}
\begin{figure*}[ht]
    \centering


    \begin{subfigure}[b]{0.32\textwidth}
        \centering
        \includegraphics[width=\textwidth]{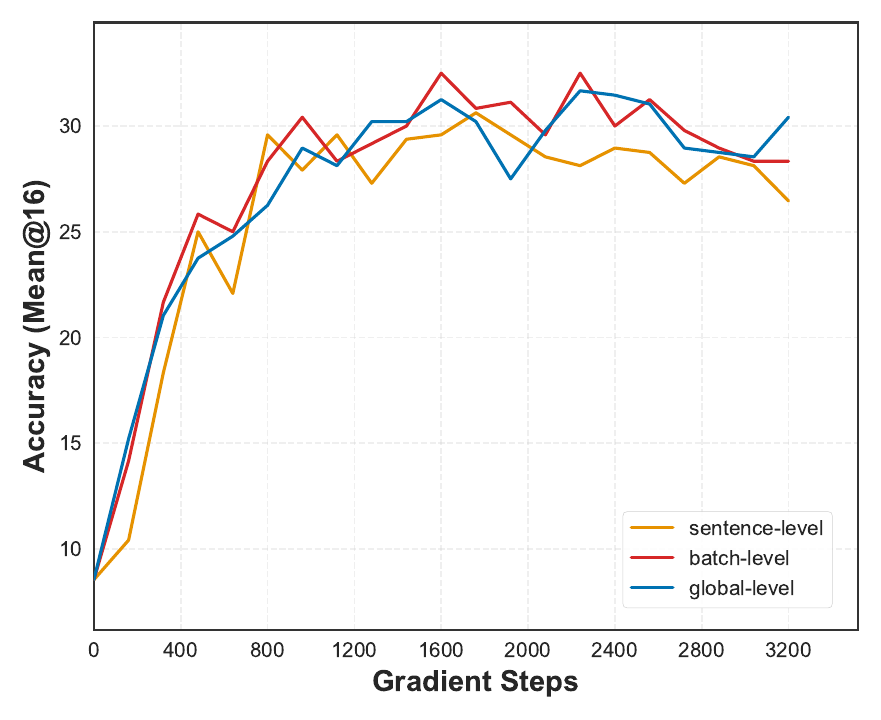}
        \vspace{-0.05in}
        \caption{AIME24 Accuracy.}
        \label{fig:entropy_norm_aime24}
    \end{subfigure}
    \hfill
    \begin{subfigure}[b]{0.32\textwidth}
        \centering
        \includegraphics[width=\textwidth]{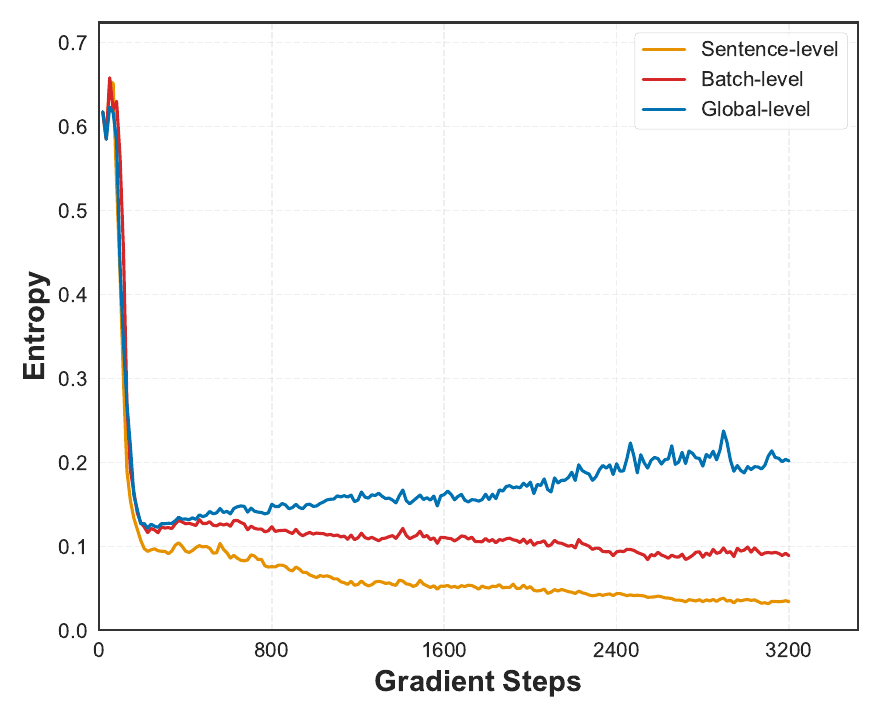}
        \vspace{-0.05in}
        \caption{Raw Predictive Entropy.}
        \label{fig:entropy_norm_raw_entropy}
    \end{subfigure}
    \hfill
    \begin{subfigure}[b]{0.32\textwidth}
        \centering
        \includegraphics[width=\textwidth]{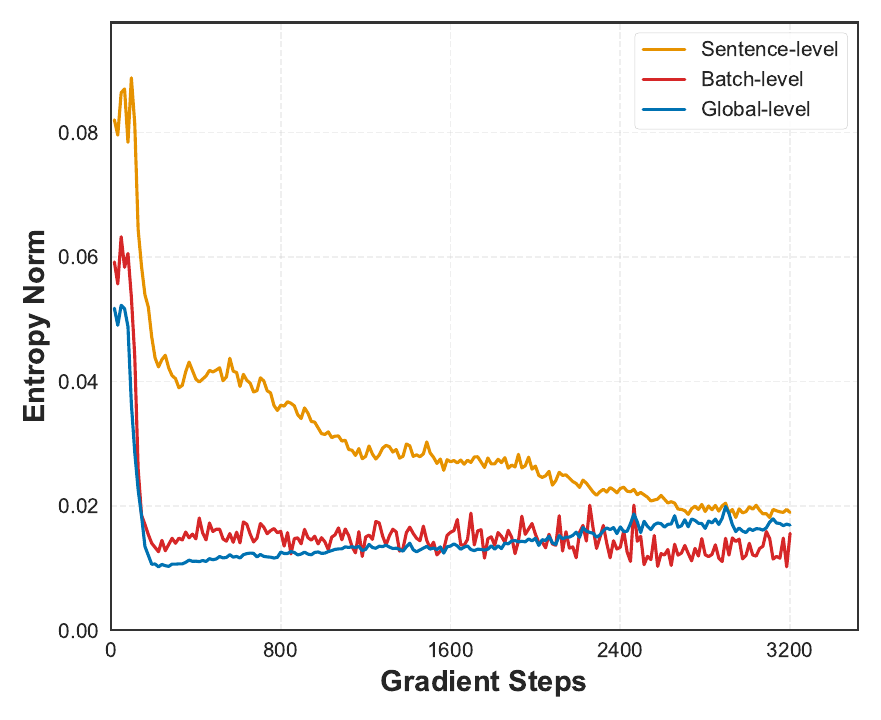}
        \vspace{-0.05in}
        \caption{Normalized Entropy.}
        \label{fig:entropy_norm_normalized_entropy}
    \end{subfigure}

    \par\vspace{0.08in}


    \begin{subfigure}[b]{0.32\textwidth}
        \centering
        \includegraphics[width=\textwidth]{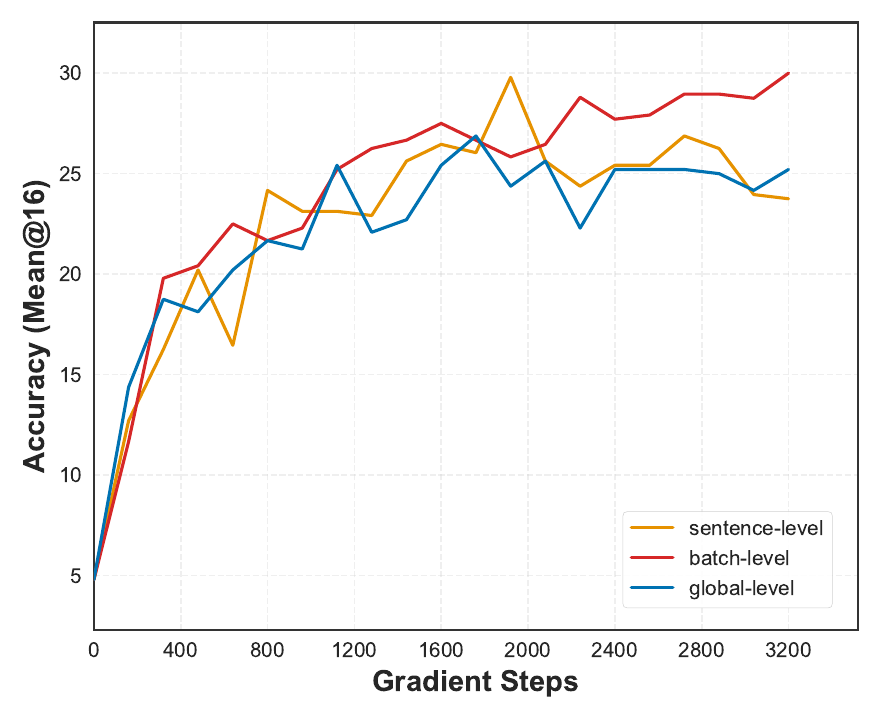}
        \vspace{-0.05in}
        \caption{AIME25 Accuracy.}
        \label{fig:entropy_norm_aime25}
    \end{subfigure}
    \hfill
    \begin{subfigure}[b]{0.32\textwidth}
        \centering
        \includegraphics[width=\textwidth]{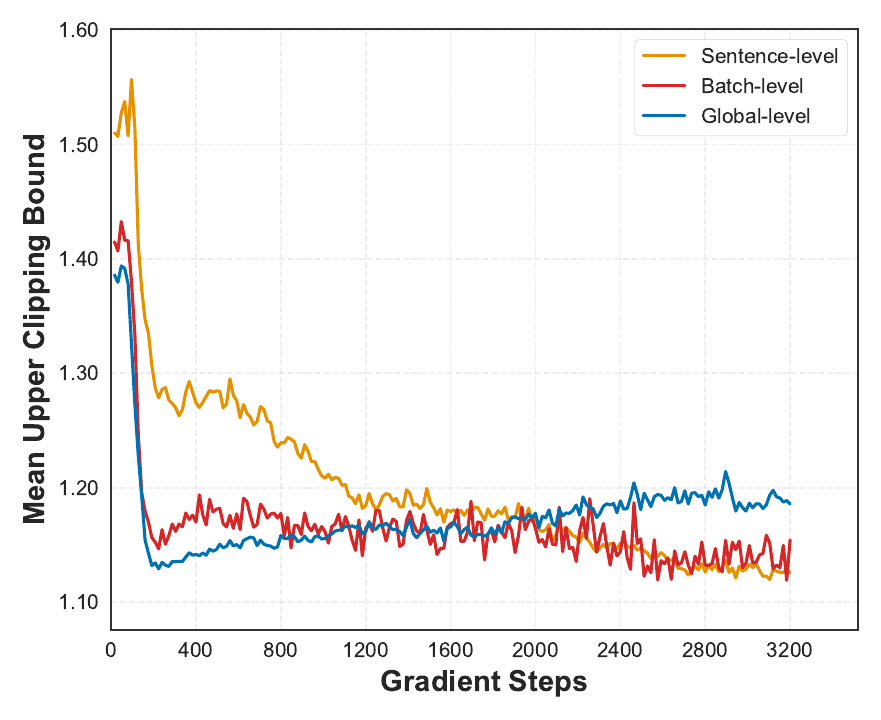}
        \vspace{-0.05in}
        \caption{Mean Upper Clipping Bound.}
        \label{fig:entropy_norm_clip_upper}
    \end{subfigure}
    \hfill
    \begin{subfigure}[b]{0.32\textwidth}
        \centering
        \includegraphics[width=\textwidth]{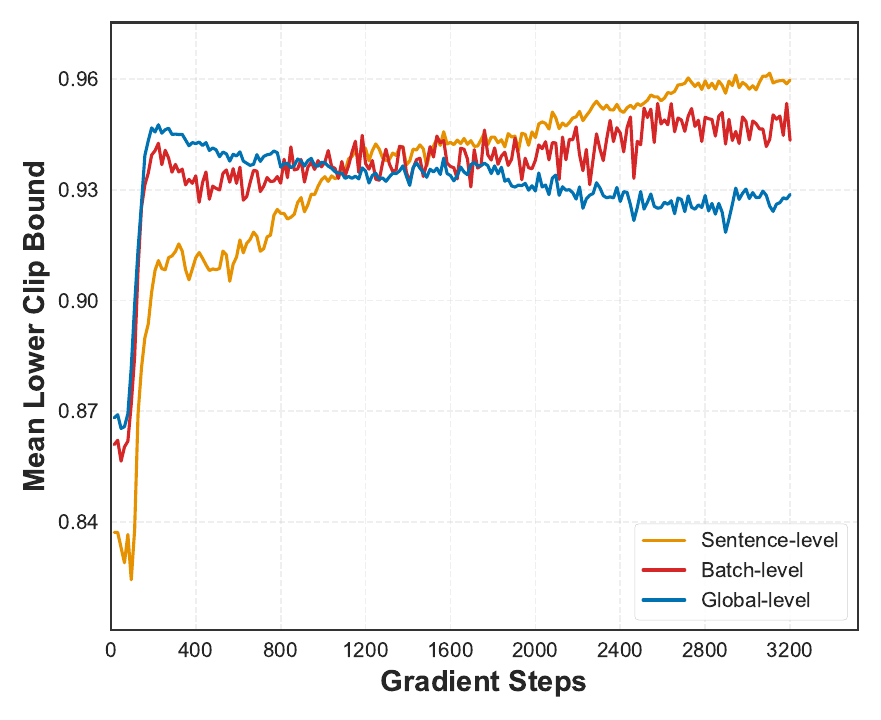}
        \vspace{-0.05in}
        \caption{Mean Lower Clipping Bound.}
        \label{fig:entropy_norm_clip_lower}
    \end{subfigure}

    \caption{
Effects of Entropy Normalization.
We empirically compare alternative normalization schemes by examining their effects on normalized entropy, adaptive clipping, and performance.
}
    \label{fig:entropy_normalization_effects}

\end{figure*}

\begin{figure*}[ht]
    \centering


    \begin{subfigure}[b]{0.32\textwidth}
        \centering
        \includegraphics[width=\textwidth]{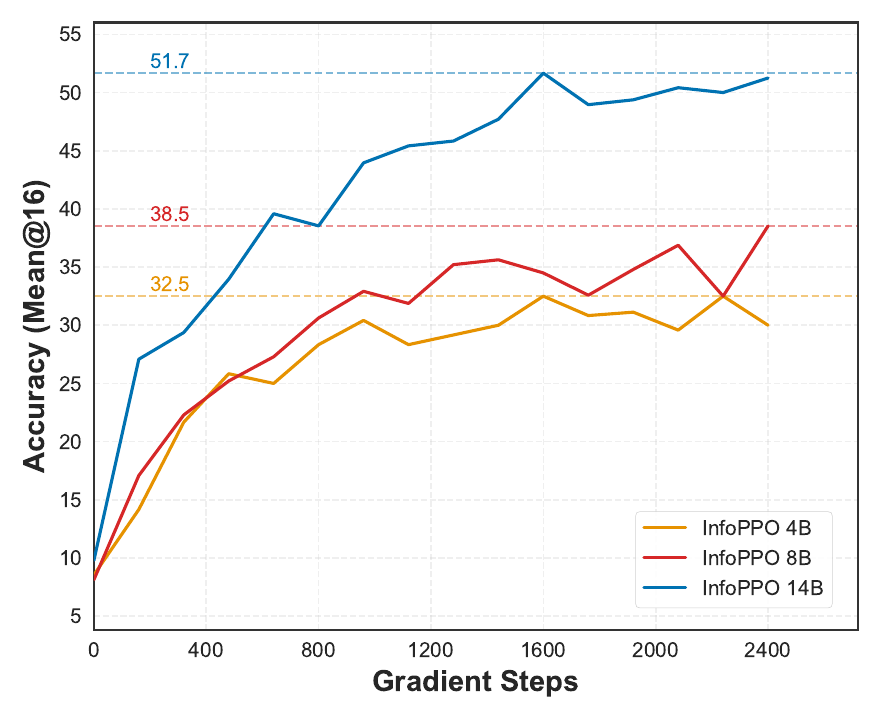}
        \vspace{-0.05in}
        \caption{AIME24 Accuracy.}
        \label{fig:info_time_aime24}
    \end{subfigure}
    \hfill
    \begin{subfigure}[b]{0.32\textwidth}
        \centering
        \includegraphics[width=\textwidth]{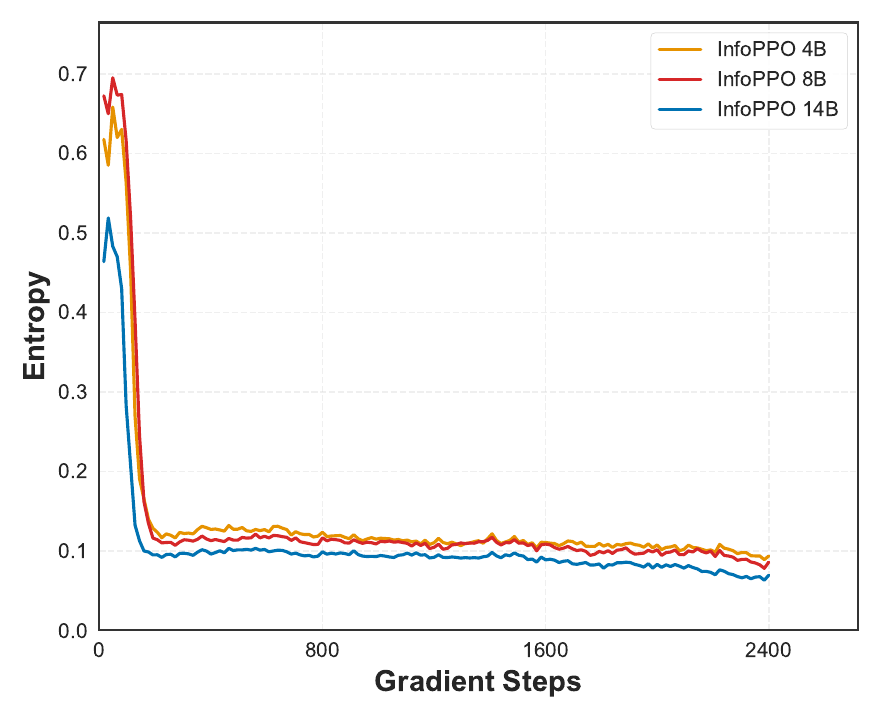}
        \vspace{-0.05in}
        \caption{Raw Predictive Entropy.}
        \label{fig:info_time_raw_entropy}
    \end{subfigure}
    \hfill
    \begin{subfigure}[b]{0.32\textwidth}
        \centering
        \includegraphics[width=\textwidth]{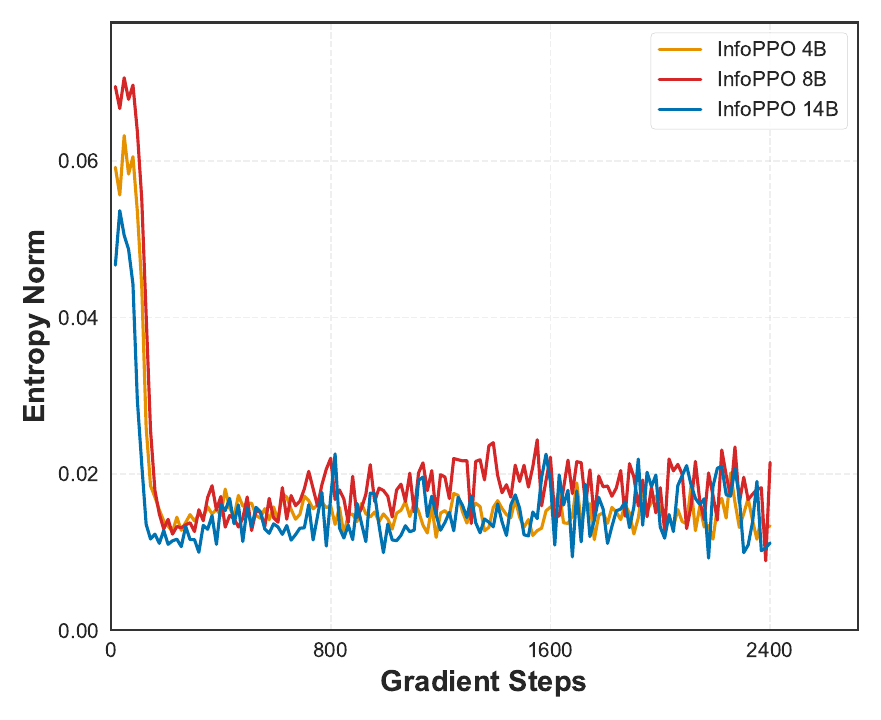}
        \vspace{-0.05in}
        \caption{Normalized Entropy.}
        \label{fig:info_time_norm_entropy}
    \end{subfigure}

    \par\vspace{0.08in}


    \begin{subfigure}[b]{0.32\textwidth}
        \centering
        \includegraphics[width=\textwidth]{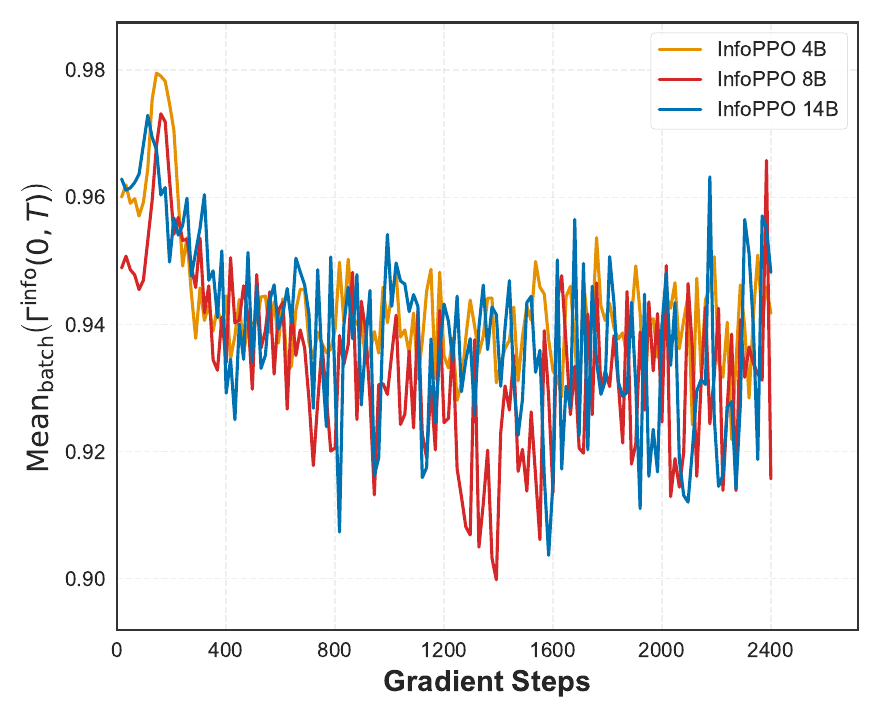}
        \vspace{-0.05in}
        \caption{Mean $\Gamma^{\mathrm{info}}(0,T)$.}
        \label{fig:info_time_discount_mean}
    \end{subfigure}
    \hfill
    \begin{subfigure}[b]{0.32\textwidth}
        \centering
        \includegraphics[width=\textwidth]{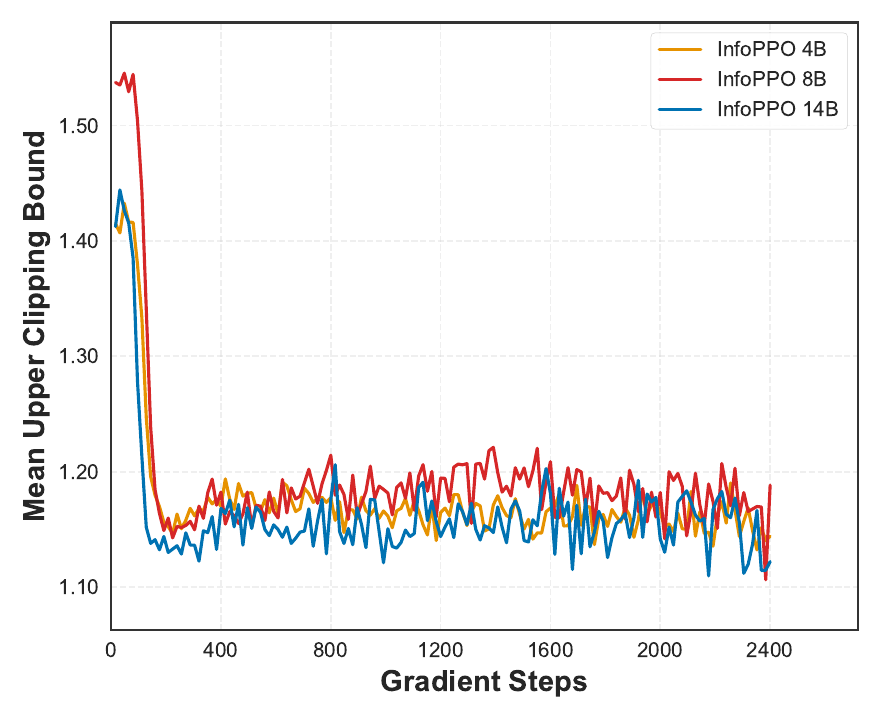}
        \vspace{-0.05in}
        \caption{Mean Upper Clipping Bound.}
        \label{fig:info_time_clip_upper}
    \end{subfigure}
    \hfill
    \begin{subfigure}[b]{0.32\textwidth}
        \centering
        \includegraphics[width=\textwidth]{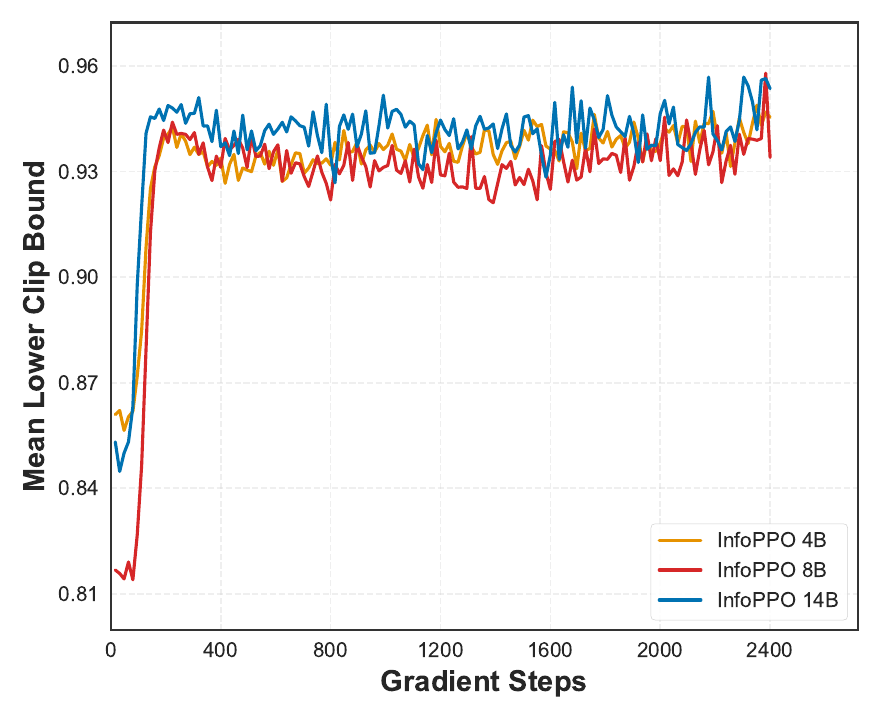}
        \vspace{-0.05in}
        \caption{Mean Lower Clipping Bound.}
        \label{fig:info_time_clip_lower}
    \end{subfigure}

    \caption{
        Information-Time Dynamics Across Model Scales.
        We track the training dynamics of \methodabb{} on Qwen3-4B, 8B, and 14B base models.
    (a) reports AIME24 accuracy; (b)--(c) show the raw predictive entropy
        and its normalized form used as the information density $\rho_t$;
        (d) reports the mean terminal information-time discount
        $\Gamma^{\mathrm{info}}(0,T)$; and (e)--(f) show the corresponding
        upper and lower adaptive clipping bounds.
    }
    \label{fig:info_time_dynamics}

\end{figure*}

Section~\ref{section:density_instantiation} instantiates the information
density as normalized predictive entropy,
$\rho(s_t)=\mathcal H_{\mathrm{old}}(s_t)/\mathcal H_{\max}$.
We examine two complementary aspects of this construction.
First, using globally fixed normalization as the reference, we further
evaluate batch- and sentence-level alternatives empirically, focusing on
the trade-off between local scale adaptation and information-clock stability.
Second, using the selected scheme, we characterize how information density,
temporal discounting, and adaptive clipping evolve during training across
model scales.

\paragraph{Entropy Normalization.}
Figure~\ref{fig:entropy_normalization_effects} compares global-, batch-,
and sentence-level entropy normalization. Relative to a globally fixed reference scale, more local normalization can make better use of the effective dynamic range of normalized entropy by adapting the scale to the entropy statistics of the current samples. This increased local adaptivity, however,
makes the normalization scale itself data-dependent and can introduce
additional variation into the resulting information density. 
The three schemes therefore exhibit different trade-offs. Global normalization preserves a common reference scale but can compress locally relevant entropy variation into a relatively narrow range. Sentence-level normalization provides the strongest local rescaling, but exhibits larger fluctuations and removes the absolute entropy-scale information across responses. Batch-level
normalization lies between these two extremes, preserving meaningful local
variation while retaining a more consistent reference across samples.
Empirically, it also maintains comparatively well-behaved normalized-entropy and clipping dynamics together with competitive downstream performance.
We therefore adopt batch-level normalization in the remaining experiments.

\paragraph{Information-Time Dynamics Across Model Scales.}
Figure~\ref{fig:info_time_dynamics} characterizes the information-time
dynamics across Qwen3 model scales. After an initial transient, the normalized entropy defining
$\rho_t$ remains relatively stable, suggesting that the overall scale of
the information clock does not exhibit large systematic drift as the policy evolves. This empirical stability is consistent with the moderate clock variation
regime characterized by the cross-iteration analysis in
Section~\ref{section:density_instantiation}.
The mean terminal information-time discount
$\Gamma^{\mathrm{info}}(0,T)$ remains at a substantial level throughout
training, indicating that information-time discounting avoids severe
attenuation of terminal supervision along typical trajectories.
Meanwhile, the adaptive
clipping bounds remain within stable ranges, and AIME24 accuracy improves
over training across model scales. These results show that the information clock
induces well-behaved temporal weighting and state-dependent update scales across model scales.

\subsection{Information-Time vs. Token-Time PPO}
\label{sec:info_vs_token_time}

We next compare information-time PPO with standard token-time PPO under
non-trivial discount and trace-decay settings, highlighting how the
information-time parameterization changes optimization behavior in
long-horizon reasoning.

\begin{figure*}[ht]
    \centering
    \includegraphics[width=\textwidth]{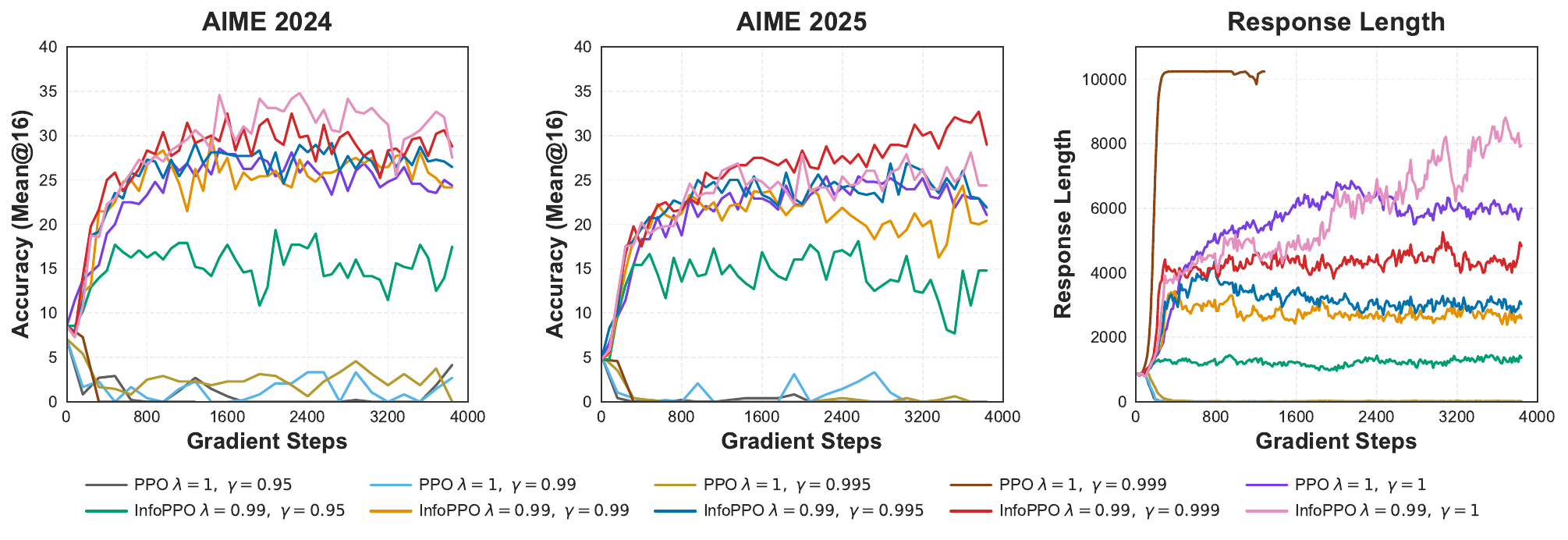}
    \caption{ Comparison of information-time and token-time discounting under different
    $\gamma$ settings. All experiments are conducted using the Qwen3-4B Base Model. (a) and (b) report the accuracy results on AIME24 and AIME25, respectively. (c) illustrates the average response length. Token-time PPO is highly sensitive to non-trivial discounting, whereas information-time PPO maintains stronger performance and more stable response lengths across $\gamma$ settings. }
    \label{fig:ablation_study_gamma}
    \vskip 0.1in
\end{figure*}

\paragraph{Discounting over Information Time.}
Figure~\ref{fig:ablation_study_gamma} compares the two temporal
parameterizations across different $\gamma$ settings. Token-time PPO is
highly sensitive to non-trivial discounting. For $\gamma<1$, accuracy
deteriorates sharply on both AIME24 and AIME25, while response length
exhibits pronounced and often degenerate changes. Notably, at $\gamma=0.999$, response length rapidly saturates near the
maximum generation length. This sharp variation highlights the instability
of token-time discounting across $\gamma$ settings. In contrast,
information-time PPO remains effective across the tested $\gamma$ range,
with substantially stronger accuracy and more gradual changes in response
length. More importantly, information-time discounting preserves effective-horizon
control, thereby keeping response length within a well-behaved,
non-degenerate range across $\gamma$ settings. This behavior is consistent with the different accumulation of
temporal decay. Token-time discounting compounds directly with raw sequence
length as $\gamma^T$, whereas information-time discounting accumulates
according to $\gamma^{\sum_t\rho_t}$. This reparameterization retains effective-horizon control while avoiding the excessive attenuation caused by the raw token distance.

\begin{figure*}[t]
    \centering
    \begin{minipage}{\textwidth}
        \centering

    \begin{subfigure}[b]{0.35\textwidth} 
        \centering
        \includegraphics[width=\textwidth]{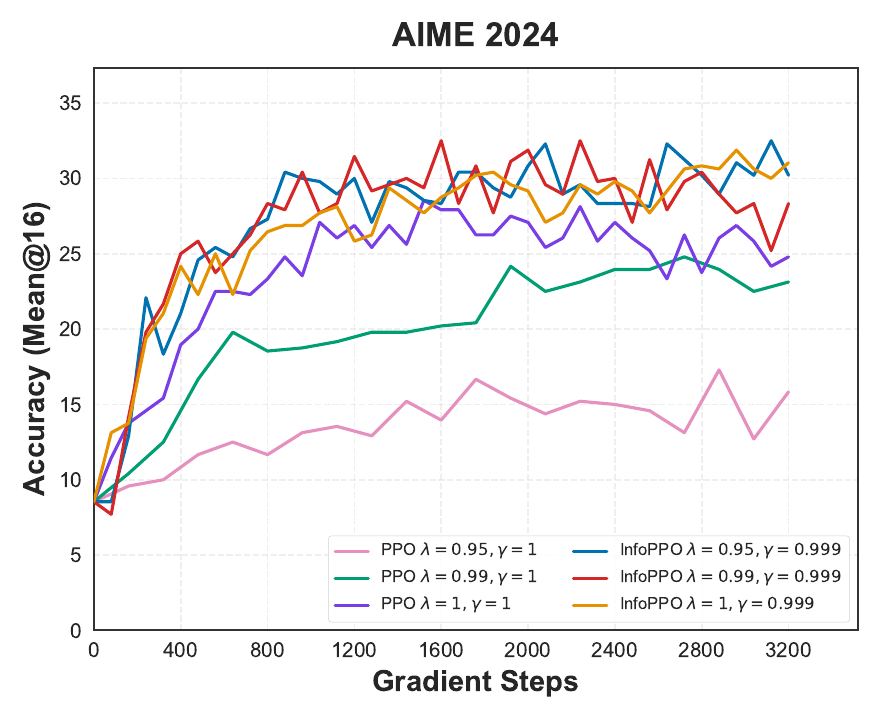}
        \caption{AIME24 Accuracy} 
        \label{fig:24_lambda_acc}
    \end{subfigure}
    \qquad 
    \begin{subfigure}[b]{0.35\textwidth}
        \centering
        \includegraphics[width=\textwidth]{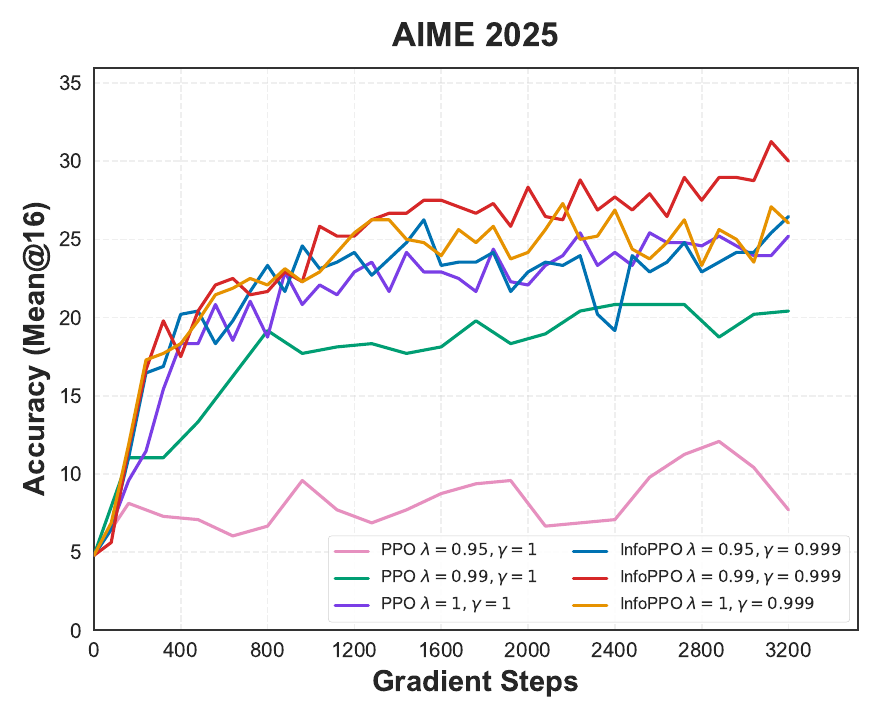}
        \caption{AIME25 Accuracy} 
        \label{fig:25_lambda_acc}
    \end{subfigure}
    \caption{    Comparison of information-time and token-time trace decay under different
    $\lambda$ settings.  (a) and (b) report the accuracy results on AIME24 and AIME25,
    respectively.}
    \label{fig:ablation_study_lambda}
    \end{minipage}
\end{figure*}

\paragraph{Trace Decay over Information Time.}
Figure~\ref{fig:ablation_study_lambda} compares the two parameterizations
across different $\lambda$ settings. Under token-time PPO, reducing
$\lambda$ below one leads to substantial performance degradation, whereas
information-time PPO remains competitive across non-trivial $\lambda$
values. Notably, $\lambda=0.99$ achieves particularly strong performance
and even outperforms $\lambda=1$ on AIME25, suggesting that non-trivial
trace decay can provide a more favorable bias--variance trade-off rather
than merely being tolerated. Under token time, 
$\lambda^{t_2-t_1}$ compounds over every generated token, causing the trace
to attenuate rapidly over long trajectories. In \methodabb{}, trace decay
accumulates over information time rather than raw token distance, allowing
credit to propagate over longer horizons without excessive attenuation.
This makes non-trivial $\lambda$ practically useful for controlling the
bias--variance trade-off in GAE.

Taken together, these results show that information-time discounting
supports non-trivial temporal decay while preserving both effective-horizon
control and stable long-range credit propagation. Based on the performance
and training behavior observed on AIME24 and AIME25, we use
$\gamma=0.999$ and $\lambda=0.99$ as the default configuration for
\methodabb{} in the remaining experiments. For the PPO baseline, we retain
its standard configuration with $\gamma=1$ and $\lambda=1$.

\subsection{Overall Performance Comparison}
Having characterized the information-time dynamics and its effects on
temporal credit propagation, we now evaluate the complete \methodabb{}
against representative RLVR baselines across model scales and mathematical
reasoning benchmarks. We examine whether the benefits observed in the
preceding analyses translate into consistent gains in reasoning performance
while maintaining well-behaved response lengths.

\begin{table*}[t]
\centering
\caption{Performance comparison of \methodabb{} with PPO, DAPO, and DAPO-FT. For each prompt, we independently sample 16 responses and report the average accuracy, denoted as Mean@16. \textbf{Bold} indicates the best result.}
\label{tab:main_results}
\begingroup
\setlength{\tabcolsep}{7pt}
\renewcommand{\arraystretch}{0.9}
\resizebox{\textwidth}{!}{%
\begin{tabular}{lcccccc>{\columncolor[HTML]{D7E8E8}}c>{\columncolor[HTML]{D7E8E8}}c}
\toprule
\multirow{2}{*}{\textbf{Benchmark}} 
& \multicolumn{2}{c}{\textbf{DAPO}} 
& \multicolumn{2}{c}{\textbf{DAPO-FT}} 
& \multicolumn{2}{c}{\textbf{PPO}} 
& \multicolumn{2}{c}{\textbf{\methodabb{}}} \\
\cmidrule(lr){2-3} \cmidrule(lr){4-5} \cmidrule(lr){6-7} \cmidrule(lr){8-9}
& \textbf{Score} & \textbf{Len.} 
& \textbf{Score} & \textbf{Len.} 
& \textbf{Score} & \textbf{Len.} 
& \textbf{Score} & \textbf{Len.} \\
\midrule

\multicolumn{9}{c}{\cellcolor[HTML]{EAF0FC}\textit{Qwen3-4B Base Model}} \\
\midrule
AMC23 
& 73.4 &  3938.7
& 75.0 &  4013.0
& 72.8 & 6583.0
& \textbf{76.6} & 3792.5 \\

AIME24 
& 30.4 & 5650.1
& 31.1 &  6286.2
& 28.1  &  8552.0
& \textbf{32.5} & 5855.8 \\

AIME25 
& 25.2 &  5098.3
& 25.4 &  5415.1
& 25.4 &  8322.0
& \textbf{28.8} &  5545.9 \\

AIME26 
& 21.7 & 5457.5 
& 20.6 & 6019.3
& 20.2 & 8632.6
& \textbf{23.8} & 5863.3 \\

BeyondAIME
& 10.3 & 4820.6
& 12.3 & 5580.7
& 12.3 & 8428.7
& \textbf{12.8} & 4714.9  \\

\midrule
\textbf{Average} 
& 32.2 & 4993.0
& 32.9 & 5462.9
& 31.8 & 8103.7
& \textbf{34.9} & 5154.5 \\

\midrule
\multicolumn{9}{c}{\cellcolor[HTML]{EAF0FC}\textit{Qwen3-8B Base Model}} \\
\midrule

AMC23 
& 80.5 & 3206.6
& 81.7 & 2855.4
& 76.9  & 5230.8
& \textbf{84.4} & 3233.2  \\

AIME24 
& 32.3 & 4789.9
& 31.3 &  4919.0
& 31.3 & 7263.6
& \textbf{38.5} & 5657.9 \\

AIME25 
& 27.5 &  4099.5
& 23.1 &  3956.1
& 25.4 & 6499.4
& \textbf{30.8} & 4795.9 \\

AIME26 
& 28.1 &  4833.6
& 26.0 &  4758.3
& 26.5 &  7483.6
& \textbf{33.5} &  5495.4 \\

BeyondAIME
& 13.0 & 3924.7
& 12.4 & 3957.5
& 12.6 & 7160.8
& \textbf{18.8} & 4878.8  \\

\midrule
\textbf{Average} 
& 36.3 & 4170.9
& 34.9 & 4089.3
& 34.5 & 6727.6
& \textbf{41.2} & 4812.2 \\
\midrule
\multicolumn{9}{c}{\cellcolor[HTML]{EAF0FC}\textit{Qwen3-14B Base Model}} \\
\midrule
AMC23 
& 87.2 & 3680.5
& 86.1 & 4382.5
& 86.4 & 5066.1
& \textbf{89.1}  & 3497.4 \\

AIME24 
& 47.7 & 6519.0
& 47.5 & 7622.8
& 45.4 & 8254.1
& \textbf{51.7} & 6334.6  \\

AIME25 
& 34.8 & 6326.4
& 37.1 & 7231.8
& 30.8 & 8024.7
& \textbf{37.5}  & 6751.6 \\

AIME26 
& 37.1 &  6459.9
& 41.3 &  7508.1
& 37.3 &  8045.8
& \textbf{45.8} &  6801.0 \\

BeyondAIME
& 19.8 & 5746.1
& 20.9 & 7245.0
& 19.4 & 7954.2
& \textbf{22.7} & 6260.1  \\

\midrule
\textbf{Average} 
& 45.3 & 5746.4
& 46.6  & 6798.0
& 43.9 & 7469.0
& \textbf{49.4} & 5928.9 \\

\bottomrule
\end{tabular}
}
\endgroup
\vskip -0.05in
\end{table*}

\begin{figure*}[ht]
    \centering

    \begin{subfigure}[b]{0.32\textwidth}
        \centering
        \includegraphics[width=\textwidth]{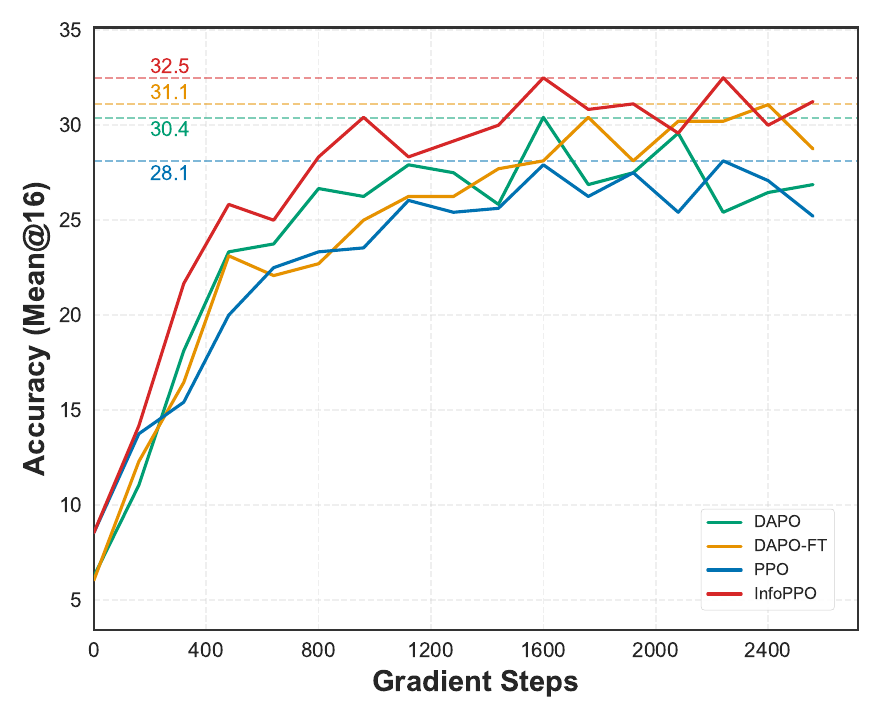}
        \vspace{-0.05in}
        \caption{AIME24 Accuracy trained from Qwen3-4B Base.}
        \label{fig:AIME24_4B}
    \end{subfigure}
    \hfill
    \begin{subfigure}[b]{0.32\textwidth}
        \centering
        \includegraphics[width=\textwidth]{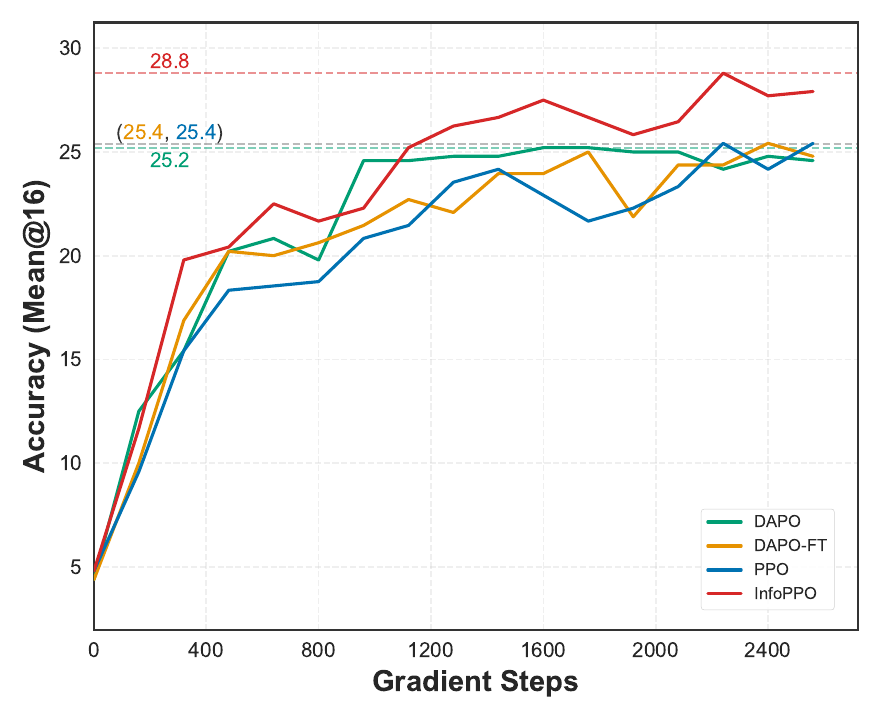}
        \vspace{-0.05in}
        \caption{AIME25 Accuracy trained from Qwen3-4B Base.}
        \label{fig:AIME25_4B}
    \end{subfigure}
    \hfill
    \begin{subfigure}[b]{0.32\textwidth}
        \centering
        \includegraphics[width=\textwidth]{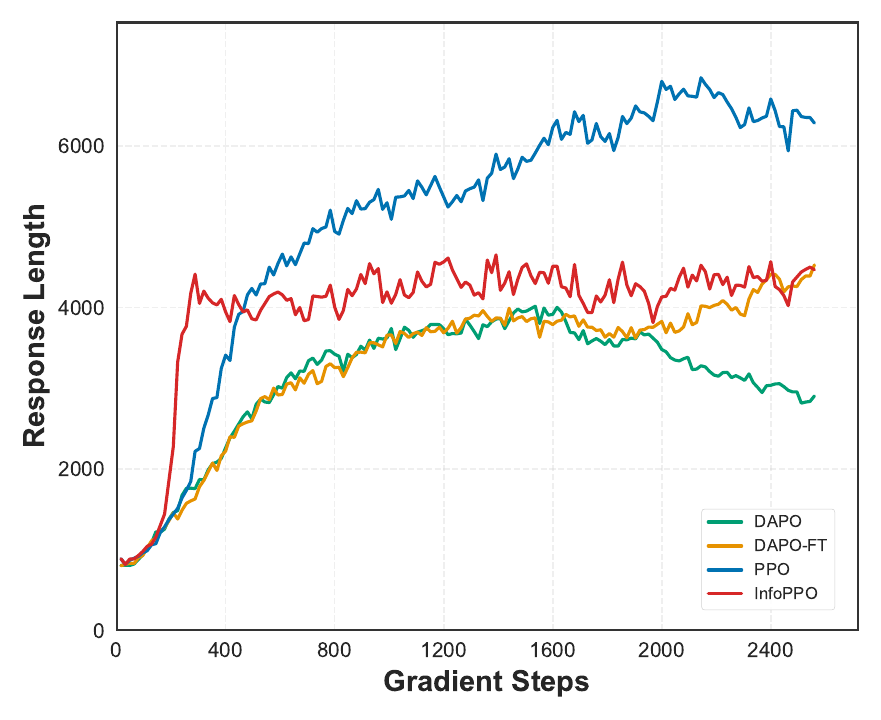}
        \vspace{-0.05in}
        \caption{Response lengths trained from Qwen3-4B Base.}
        \label{fig:response_4B}
    \end{subfigure}

    \par\vspace{0.08in}

    \begin{subfigure}[b]{0.32\textwidth}
        \centering
        \includegraphics[width=\textwidth]{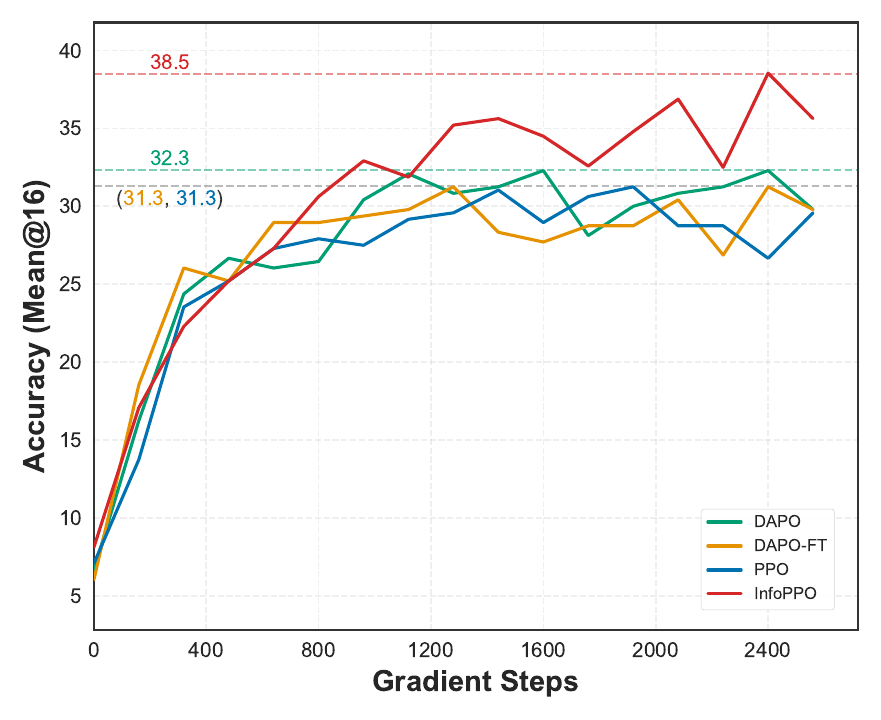}
        \vspace{-0.05in}
        \caption{AIME24 Accuracy trained from Qwen3-8B Base.}
        \label{fig:AIME24_8B}
    \end{subfigure}
    \hfill
    \begin{subfigure}[b]{0.32\textwidth}
        \centering
        \includegraphics[width=\textwidth]{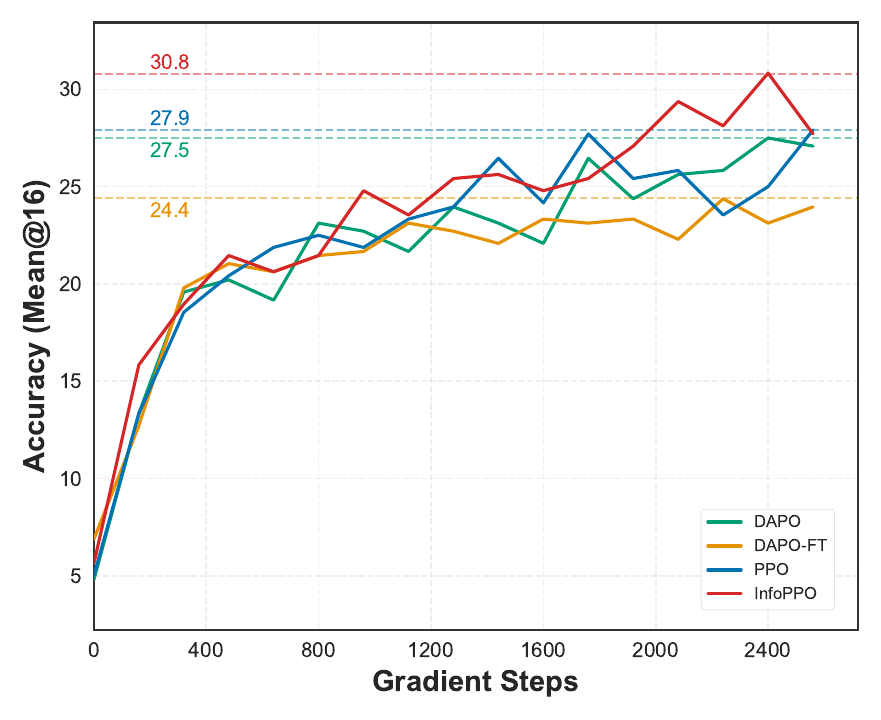}
        \vspace{-0.05in}
        \caption{AIME25 Accuracy trained from Qwen3-8B Base.}
        \label{fig:AIME25_8B}
    \end{subfigure}
    \hfill
    \begin{subfigure}[b]{0.32\textwidth}
        \centering
        \includegraphics[width=\textwidth]{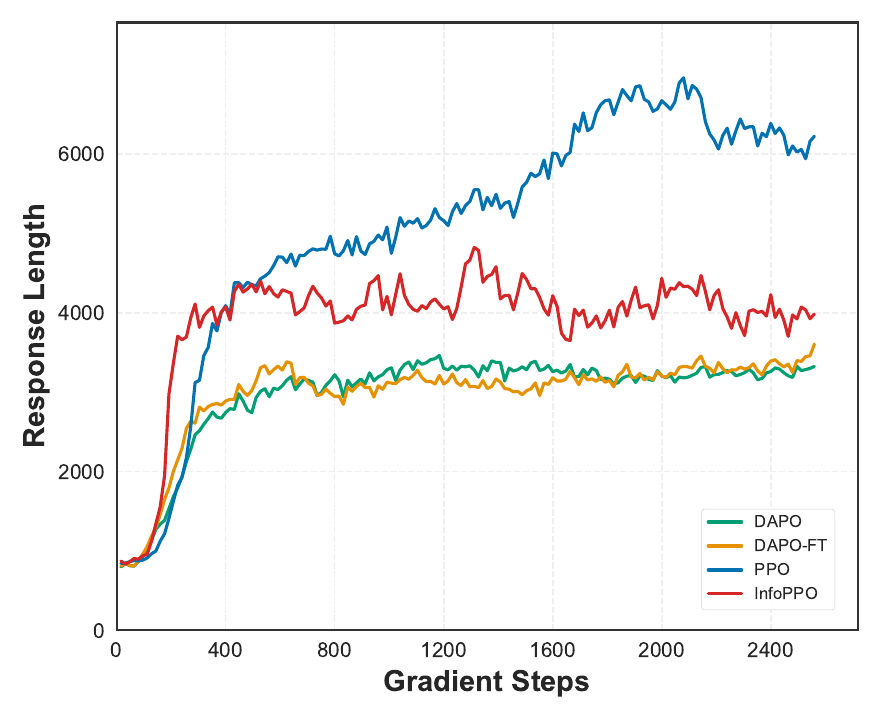}
        \vspace{-0.05in}
        \caption{Response lengths trained from Qwen3-8B Base.}
        \label{fig:response_8B}
    \end{subfigure}

    \par\vspace{0.08in}

    \begin{subfigure}[b]{0.32\textwidth}
        \centering
        \includegraphics[width=\textwidth]{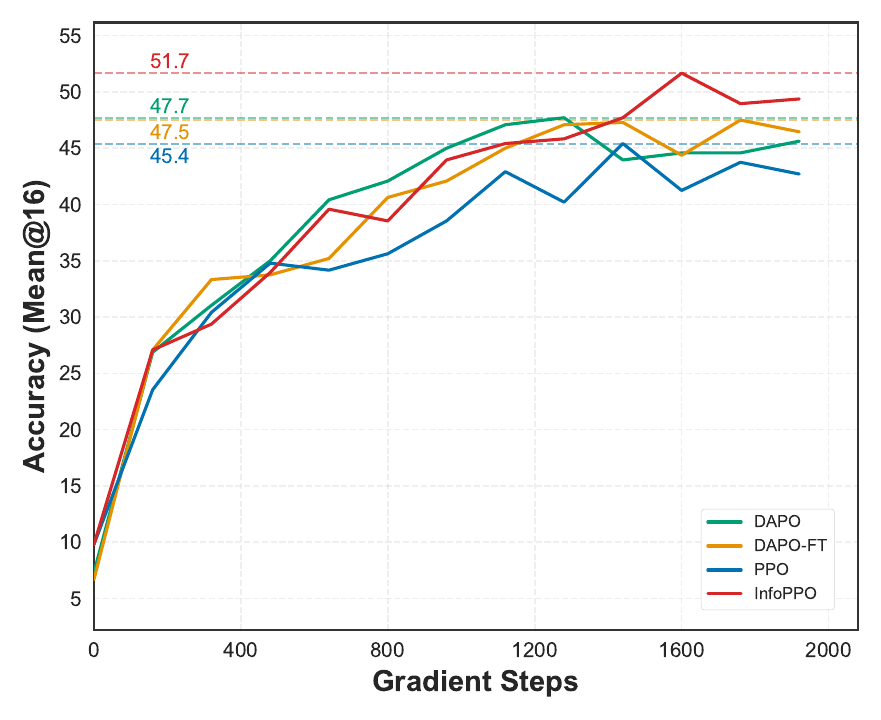}
        \vspace{-0.05in}
        \caption{AIME24 Accuracy trained from Qwen3-14B Base.}
        \label{fig:AIME24_14B}
    \end{subfigure}
    \hfill
    \begin{subfigure}[b]{0.32\textwidth}
        \centering
        \includegraphics[width=\textwidth]{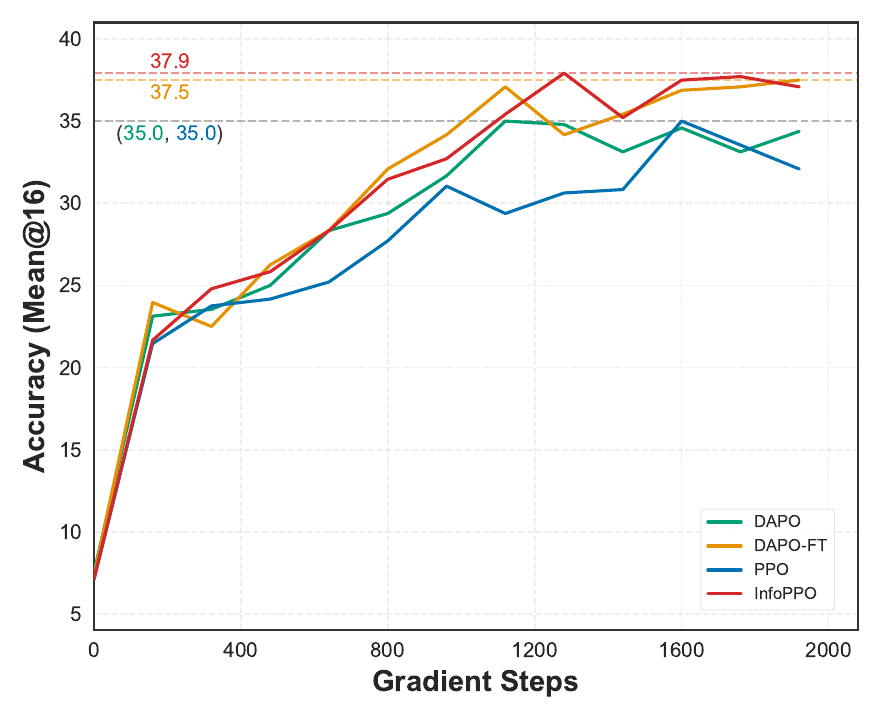}
        \vspace{-0.05in}
        \caption{AIME25 Accuracy trained from Qwen3-14B Base.}
        \label{fig:AIME25_14B}
    \end{subfigure}
    \hfill
    \begin{subfigure}[b]{0.32\textwidth}
        \centering
        \includegraphics[width=\textwidth]{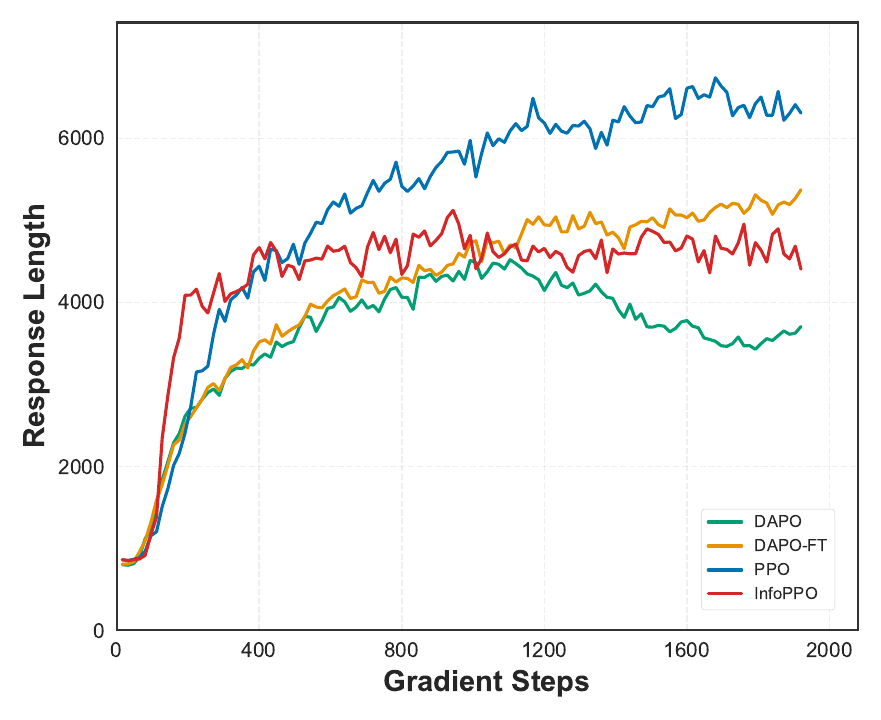}
        \vspace{-0.05in}
        \caption{Response lengths trained from Qwen3-14B base.}
        \label{fig:response_14B}
    \end{subfigure}


    \caption{
    Training Dynamics of Accuracy and Response Length Across Model Scales.
    \methodabb{} exhibits stronger accuracy gains over training while maintaining
    stable response lengths, indicating more efficient policy improvement without
    relying on continued length growth.
    }
    \label{fig:nine_subfigures}

\end{figure*}

Table \ref{tab:main_results} summarizes the performance of \methodabb{} against competitive baselines. \methodabb{} delivers consistently
strong performance across Qwen3-4B, 8B, and 14B, with an overall advantage
across multiple challenging mathematical reasoning benchmarks. These gains
are not accompanied by systematic response-length growth. The generation
lengths of \methodabb{} remain broadly comparable to those of DAPO and
DAPO-FT, while relative to PPO, \methodabb{} achieves higher average accuracy
with shorter responses. Overall, \methodabb{} exhibits a more favorable
accuracy--length trade-off rather than relying on continued expansion of the
reasoning trajectory.

Figure~\ref{fig:nine_subfigures} further shows that this pattern persists
throughout training. Across model scales, \methodabb{} steadily improves
accuracy while keeping response lengths within a relatively stable range.
In contrast, the baselines exhibit larger variations in response length over the course of training. This behavior is consistent with the
information-time formulation. Information-time discounting introduces a
non-trivial effective-horizon constraint while avoiding excessive accumulation of token-wise decay over long sequences, and information-dependent
policy updates adapt the optimization scale to local information density.
Together, these mechanisms enable effective policy improvement within a controlled
information horizon without relying on continued trajectory expansion.

Taken together, the final performance and training dynamics show that
\methodabb{} consistently translates the information-time formulation into
improved reasoning performance across model scales and challenging
mathematical reasoning tasks, while maintaining a favorable overall
efficiency.

\section{Conclusion}
We introduce \methodabb{}, which reparameterizes temporal progression in LLM
reinforcement learning by accumulated information rather than raw token count.
This information-time formulation provides a common basis for temporal credit
propagation and policy-update regulation, enabling non-trivial effective-horizon
control while adapting policy updates to local information density. We establish
policy-improvement guarantees under this state-dependent temporal geometry and
connect the practical entropy-based construction to the resulting policy
movement. Experiments across Qwen3 model scales and challenging mathematical
reasoning benchmarks show consistent performance gains and stable response-length behavior, demonstrating the effectiveness of information time as a temporal parameterization for long-horizon LLM reinforcement learning. Looking forward, an important direction is to explore alternative
instantiations of information density beyond predictive entropy, potentially
capturing complementary aspects of information progression along reasoning
trajectories. It would also be interesting to explicitly regulate predictive
entropy during policy optimization, which could help maintain a more stable
information clock across policy iterations while offering a potential means of
balancing exploration and exploitation during training.

\clearpage

\bibliography{example_paper}

@article{schulman2017proximal,
  title={Proximal policy optimization algorithms},
  author={Schulman, John and Wolski, Filip and Dhariwal, Prafulla and Radford, Alec and Klimov, Oleg},
  journal={arXiv preprint arXiv:1707.06347
        
        
        
        
        
        },
  year={2017}
}

@inproceedings{schulman2015trust,
  title={Trust region policy optimization},
  author={Schulman, John and Levine, Sergey and Abbeel, Pieter and Jordan, Michael and Moritz, Philipp},
  booktitle={International conference on machine learning},
  pages={1889--1897},
  year={2015},
  organization={PMLR}
}

@article{lambert2024t,
  title={T$\backslash$" ULU 3: Pushing Frontiers in Open Language Model Post-Training},
  author={Lambert, Nathan and Morrison, Jacob and Pyatkin, Valentina and Huang, Shengyi and Ivison, Hamish and Brahman, Faeze and Miranda, Lester James V and Liu, Alisa and Dziri, Nouha and Lyu, Shane and others},
  journal={arXiv preprint arXiv:2411.15124
        
        
        
        
        
        
        
        
        
        
        
        
        
        },
  year={2024}
}

@article{guo2025deepseek,
  title={Deepseek-r1: Incentivizing reasoning capability in llms via reinforcement learning},
  author={Guo, Daya and Yang, Dejian and Zhang, Haowei and Song, Junxiao and Zhang, Ruoyu and Xu, Runxin and Zhu, Qihao and Ma, Shirong and Wang, Peiyi and Bi, Xiao and others},
  journal={arXiv preprint arXiv:2501.12948
        
        
        
        
        
        
        
        
        
        },
  year={2025}
}

@article{yu2025dapo,
  title={Dapo: An open-source llm reinforcement learning system at scale},
  author={Yu, Qiying and Zhang, Zheng and Zhu, Ruofei and Yuan, Yufeng and Zuo, Xiaochen and Yue, Yu and Fan, Tiantian and Liu, Gaohong and Liu, Lingjun and Liu, Xin and others},
  journal={arXiv preprint arXiv:2503.14476
        
        
        
        
        
        
        
        
        
        
        
        
        
        
        
        
        
        
        
        
        
        
        
        },
  year={2025}
}

@article{shao2024deepseekmath,
  title={Deepseekmath: Pushing the limits of mathematical reasoning in open language models},
  author={Shao, Zhihong and Wang, Peiyi and Zhu, Qihao and Xu, Runxin and Song, Junxiao and Bi, Xiao and Zhang, Haowei and Zhang, Mingchuan and Li, YK and Wu, Y and others},
  journal={arXiv preprint arXiv:2402.03300
        
        
        
        },
  year={2024}
}

@article{li2024numinamath,
  title={Numinamath: The largest public dataset in ai4maths with 860k pairs of competition math problems and solutions},
  author={Li, Jia and Beeching, Edward and Tunstall, Lewis and Lipkin, Ben and Soletskyi, Roman and Huang, Shengyi and Rasul, Kashif and Yu, Longhui and Jiang, Albert Q and Shen, Ziju and others},
  journal={Hugging Face repository},
  volume={13},
  pages={9},
  year={2024}
}

@article{yang2025qwen3,
  title={Qwen3 technical report},
  author={Yang, An and Li, Anfeng and Yang, Baosong and Zhang, Beichen and Hui, Binyuan and Zheng, Bo and Yu, Bowen and Gao, Chang and Huang, Chengen and Lv, Chenxu and others},
  journal={arXiv preprint arXiv:2505.09388
        
        
        
        
        
        
        
        
        
        
        
        
        
        },
  year={2025}
}

@article{jaech2024openai,
  title={Openai o1 system card},
  author={Jaech, Aaron and Kalai, Adam and Lerer, Adam and Richardson, Adam and El-Kishky, Ahmed and Low, Aiden and Helyar, Alec and Madry, Aleksander and Beutel, Alex and Carney, Alex and others},
  journal={arXiv preprint arXiv:2412.16720
        
        
        
        
        
        
        
        
        
        
        
        
        
        
        
        
        
        },
  year={2024}
}

@article{qu2025dynamic,
  title={Dynamic Large Concept Models: Latent Reasoning in an Adaptive Semantic Space},
  author={Qu, Xingwei and Wang, Shaowen and Huang, Zihao and Hua, Kai and Yin, Fan and Zhu, Rui-Jie and Zhou, Jundong and Min, Qiyang and Wang, Zihao and Li, Yizhi and others},
  journal={arXiv preprint arXiv:2512.24617
        
        
        
        
        
        
        
        
        
        
        
        
        
        
        
        
        
        
        
        
        
        
        
        },
  year={2025}
}

@article{wang2025beyond,
  title={Beyond the 80/20 rule: High-entropy minority tokens drive effective reinforcement learning for llm reasoning},
  author={Wang, Shenzhi and Yu, Le and Gao, Chang and Zheng, Chujie and Liu, Shixuan and Lu, Rui and Dang, Kai and Chen, Xionghui and Yang, Jianxin and Zhang, Zhenru and others},
  journal={arXiv preprint arXiv:2506.01939
        
        
        
        
        
        
        
        
        
        
        
        
        
        
        
        
        
        
        
        
        
        
        
        
        
        
        
        },
  year={2025}
}

@article{brockman2016openai,
  title={Openai gym},
  author={Brockman, Greg and Cheung, Vicki and Pettersson, Ludwig and Schneider, Jonas and Schulman, John and Tang, Jie and Zaremba, Wojciech},
  journal={arXiv preprint arXiv:1606.01540
        
        
        
        
        
        
        
        },
  year={2016}
}

@article{zheng2025group,
  title={Group sequence policy optimization},
  author={Zheng, Chujie and Liu, Shixuan and Li, Mingze and Chen, Xiong-Hui and Yu, Bowen and Gao, Chang and Dang, Kai and Liu, Yuqiong and Men, Rui and Yang, An and others},
  journal={arXiv preprint arXiv:2507.18071
        
        
        
        
        
        
        
        
        
        },
  year={2025}
}

@article{sutton1999between,
  title={Between MDPs and semi-MDPs: A framework for temporal abstraction in reinforcement learning},
  author={Sutton, Richard S and Precup, Doina and Singh, Satinder},
  journal={Artificial intelligence},
  volume={112},
  number={1-2},
  pages={181--211},
  year={1999},
  publisher={Elsevier}
}

@book{precup2000temporal,
  title={Temporal abstraction in reinforcement learning},
  author={Precup, Doina},
  year={2000},
  publisher={University of Massachusetts Amherst}
}

@inproceedings{bacon2017option,
  title={The option-critic architecture},
  author={Bacon, Pierre-Luc and Harb, Jean and Precup, Doina},
  booktitle={Proceedings of the AAAI conference on artificial intelligence},
  volume={31},
  number={1},
  year={2017}
}

@inproceedings{kakade2002approximately,
  title={Approximately optimal approximate reinforcement learning},
  author={Kakade, Sham and Langford, John},
  booktitle={Proceedings of the nineteenth international conference on machine learning},
  pages={267--274},
  year={2002}
}

@article{schulman2015high,
  title={High-dimensional continuous control using generalized advantage estimation},
  author={Schulman, John and Moritz, Philipp and Levine, Sergey and Jordan, Michael and Abbeel, Pieter},
  journal={arXiv preprint arXiv:1506.02438
        
        
        
        
        
        
        
        
        
        
        
        },
  year={2015}
}

@article{fu2025deep,
  title={Deep think with confidence},
  author={Fu, Yichao and Wang, Xuewei and Tian, Yuandong and Zhao, Jiawei},
  journal={arXiv preprint arXiv:2508.15260
        
        
        
        
        
        
        
        
        
        
        
        
        
        
        
        
        
        },
  year={2025}
}

@article{feinberg1994markov,
  title={Markov decision models with weighted discounted criteria},
  author={Feinberg, Eugene A and Shwartz, Adam},
  journal={Mathematics of Operations Research},
  volume={19},
  number={1},
  pages={152--168},
  year={1994},
  publisher={INFORMS}
}

@inproceedings{white2017unifying,
  title={Unifying task specification in reinforcement learning},
  author={White, Martha},
  booktitle={International Conference on Machine Learning},
  pages={3742--3750},
  year={2017},
  organization={PMLR}
}

@misc{aime26,
      title={American Invitational Mathematics Examination (AIME) 2026}, 
      author={Zhang, Yifan and Math-AI, Team},
      year={2026},
}

@misc{aime25,
      title={American Invitational Mathematics Examination (AIME) 2025}, 
      author={Zhang, Yifan and Math-AI, Team},
      year={2025},
}

@misc{aime24,
      title={American Invitational Mathematics Examination (AIME) 2024}, 
      author={Zhang, Yifan and Math-AI, Team},
      year={2024},
}

@misc{beyondaime2025,
  author    = {{ByteDance-Seed}},
  title     = {{BeyondAIME}: Advancing Math Reasoning Evaluation Beyond High School Olympiads},
  year      = {2025},
  publisher = {Hugging Face},
  url       = {https://huggingface.co/datasets/ByteDance-Seed/BeyondAIME}
}
\bibliographystyle{icml2026}


\appendix

\newpage

\section{Mathematical Derivations}
\subsection{Proofs for  Policy Improvement on the Information-Time MDP}\label{app:proof_monotonic_policy_improvement}
 \begin{lemma}
 [Information-Time Performance Difference]
\label{app_lemma:im_performance_diff}
Given two policies $\pi$ and $\tilde{\pi}$ within the Information-Time MDP framework, the following identity holds:
\begin{equation}
\begin{aligned}
\eta^{\text{info}}&(\tilde{\pi}) - \eta^{\text{info}}(\pi) = \mathbb{E}_{\tau \sim \tilde{\pi}}\left[\sum_{t=0}^{\infty} \Gamma^{\text{info}}(0,t) A^{\text{info}}_{\pi}(s_t, a_t)\right],
\end{aligned}
\end{equation}
where $\Gamma^{\text{info}}(0,t)$ represents the cumulative information-time discount factor.
\end{lemma}

\begin{proof}
First, we verify that the cumulative information discount factor $\Gamma^{\text{info}}(t_1,t_2)$ satisfies the following multiplicative property:
\begin{align}
\Gamma^{\text{info}}(t_1, t_2) = \gamma^{\sum_{t=t_1}^{t_2-1} \Delta\mu_{\text{info}}(t)}=\gamma^{\sum_{t_1}^{t_{\text{mid}}-1} \Delta\mu_{\text{info}}}\times \gamma^{\sum_{t_{\text{mid}}}^{t_2-1} \Delta\mu_{\text{info}}} = \Gamma^{\text{info}}(t_1, t_{\text{mid}})\Gamma^{\text{info}}(t_{\text{mid}},t_2).
\end{align}

Recalling the definition $Q_{\pi}^{\text{info}}(s_t,a_t)
\coloneqq \mathbb{E}_{\pi}\!\left[\sum_{l=0}^{\infty}\Gamma^{\text{info}}(t,t+l)r(s_{t+l}) \Big| s_t,a_t\right]$, we proceed to derive the Bellman equation on the information-time:
\begin{align}
Q_{\pi}^{\text{info}}(s_t,a_t)
&= \mathbb{E}_{\pi}\!\left[\sum_{l=0}^{\infty}\Gamma^{\text{info}}(t,t+l){r}(s_{t+l}) \Big| s_t,a_t\right]\\
&=\mathbb{E}_{\pi}\!\left[{r}(s_{t}) + \Gamma^{\text{info}}(t,t+1)\sum_{l=t+1}^{\infty}\Gamma^{\text{info}}(t+1,l){r}(s_{l}) \Big| s_t,a_t\right]\\
&={r}(s_{t}) + \Gamma^{\text{info}}(t,t+1)\mathbb{E}_{\pi}\!\left[ \sum_{l=t+1}^{\infty}\Gamma^{\text{info}}(t+1,l){r}(s_{l}) \Big| s_t,a_t\right]\\
&={r}(s_{t}) + \Gamma^{\text{info}}(t,t+1)V_{\pi}^{\text{info}}(s_{t+1}).
\end{align}

Consequently, for the advantage function $A_{\pi}^{\text{info}}(s,a)$, we obtain
\begin{align}
A_{\pi}^{\text{info}}(s,a) = Q_{\pi}^{\text{info}}(s,a) - V_{\pi}^{\text{info}}(s) ={r}(s_{t}) + \Gamma^{\text{info}}(t,t+1)V_{\pi}^{\text{info}}(s_{t+1}) - V_{\pi}^{\text{info}}(s_{t}).
\end{align}
Therefore, 
\begin{align}
&\mathbb{E}_{\tau \sim \tilde{\pi}}\left[\sum_{t=0}^{\infty} \Gamma^{\text{info}}(0,t) A^{\text{info}}_{\pi}(s_t, a_t)\right]\\
&= \mathbb{E}_{\tau \sim \tilde{\pi}}\left[\sum_{0}^{\infty} \Gamma^{\text{info}}(0,t) \left({r}(s_{t}) + \Gamma^{\text{info}}(t,t+1)V_{\pi}^{\text{info}}(s_{t+1}) - V_{\pi}^{\text{info}}(s_{t})\right)\right]\\
&=\mathbb{E}_{\tau \sim \tilde{\pi}}\left[\sum_{0}^{\infty}\left(\Gamma^{\text{info}}(0,t+1)V_{\pi}^{\text{info}}(s_{t+1}) - \Gamma^{\text{info}}(0,t)V_{\pi}^{\text{info}}(s_{t})\right)+\sum_{t=0}^{\infty} \Gamma^{\text{info}}(0,t) {r}(s_{t})\right]\\
&=\mathbb{E}_{\tau \sim \tilde{\pi}}\left[-\Gamma^{\text{info}}(0,0)V_{\pi}^{\text{info}}(s_{0})+\sum_{t=0}^{\infty} \Gamma^{\text{info}}(0,t) {r}(s_{t})\right]\\
&=\mathbb{E}_{\tau \sim \tilde{\pi}}\left[-V_{\pi}^{\text{info}}(s_{0})+\sum_{t=0}^{\infty} \Gamma^{\text{info}}(0,t) {r}(s_{t})\right]\\
&=  -\eta^{\text{info}}(\pi) + \eta^{\text{info}}(\tilde{\pi}).
\end{align}
For the infinite-horizon telescoping argument above, the remaining boundary term is required to vanish:
$$
\lim_{N\to\infty}
\mathbb{E}_{\tau\sim\tilde{\pi}}
\left[
\Gamma^{\mathrm{info}}(0,N+1)
V_\pi^{\mathrm{info}}(s_{N+1})
\right]
=0.
$$
This condition is naturally satisfied in episodic environments with terminal states, since the continuation value is zero after termination. In particular, it holds for the language-generation setting considered in this work, where each response eventually terminates and therefore has zero continuation value thereafter. 
\end{proof}

To connect the trajectory form of the Information-Time Performance
Difference with the state-space formulation used in
Section~\ref{sec:info_policy_improvement}, we now derive its equivalent
representation in terms of the information-discounted state visitation
measure.

\begin{corollary}
\textup{(State-Space Representation of the Information-Time Performance Difference).}
\label{app:info_state_representation}
Define the unnormalized information-discounted state visitation measure
under policy $\pi$ as
\begin{equation}
\nu_{\pi}^{\mathrm{info}}(s)
\coloneqq
\mathbb{E}_{\tau\sim\pi}
\left[
\sum_{t=0}^{\infty}
\Gamma^{\mathrm{info}}(0,t)
\mathbf{1}\{s_t=s\}
\right].\nonumber
\end{equation}
Then
\begin{equation}
\eta^{\mathrm{info}}(\tilde{\pi})
-
\eta^{\mathrm{info}}(\pi)
=
\sum_s
\nu_{\tilde{\pi}}^{\mathrm{info}}(s)
\sum_a
\tilde{\pi}(a|s)
A_{\pi}^{\mathrm{info}}(s,a).
\end{equation}
\end{corollary}

\begin{proof}
Starting from the Information-Time Performance Difference in Lemma~\ref{app_lemma:im_performance_diff},
\begin{align}
\eta^{\mathrm{info}}(\tilde{\pi})
-
\eta^{\mathrm{info}}(\pi)
&=
\mathbb{E}_{\tau\sim\tilde{\pi}}
\left[
\sum_{t=0}^{\infty}
\Gamma^{\mathrm{info}}(0,t)
A_{\pi}^{\mathrm{info}}(s_t,a_t)
\right].
\label{eq:state_rep_start}
\end{align}
For each fixed time $t$, we introduce an indicator over the current state:
\begin{equation}
A_{\pi}^{\mathrm{info}}(s_t,a_t)
=
\sum_s
\mathbf{1}\{s_t=s\}
A_{\pi}^{\mathrm{info}}(s,a_t).
\end{equation}
Substituting this identity into Eq.~\eqref{eq:state_rep_start} gives
\begin{align}
\eta^{\mathrm{info}}(\tilde{\pi})
-
\eta^{\mathrm{info}}(\pi)
&=
\sum_{t=0}^{\infty}
\mathbb{E}_{\tau\sim\tilde{\pi}}
\left[
\Gamma^{\mathrm{info}}(0,t)
\sum_s
\mathbf{1}\{s_t=s\}
A_{\pi}^{\mathrm{info}}(s,a_t)
\right]
\nonumber\\
&=
\sum_{t=0}^{\infty}
\sum_s
\mathbb{E}_{\tau\sim\tilde{\pi}}
\left[
\Gamma^{\mathrm{info}}(0,t)
\mathbf{1}\{s_t=s\}
A_{\pi}^{\mathrm{info}}(s,a_t)
\right].
\label{eq:state_rep_decompose}
\end{align}

Let $h_t$ denote the trajectory history up to the current state $s_t$,
before $a_t$ is sampled. Conditional on $h_t$, both
$\Gamma^{\mathrm{info}}(0,t)$ and $\mathbf{1}\{s_t=s\}$ are determined,
whereas the current action is sampled according to
$a_t\sim\tilde{\pi}(\cdot|s_t)$. Hence, applying iterated expectation
first over the current action and then over the history gives
\begin{align}
&
\mathbb{E}_{\tau\sim\tilde{\pi}}
\left[
\Gamma^{\mathrm{info}}(0,t)
\mathbf{1}\{s_t=s\}
A_{\pi}^{\mathrm{info}}(s,a_t)
\right]
\nonumber\\
&=
\mathbb{E}_{h_t\sim\tilde{\pi}}
\left[
\Gamma^{\mathrm{info}}(0,t)
\mathbf{1}\{s_t=s\}
\,
\mathbb{E}_{a_t\sim\tilde{\pi}(\cdot|s_t)}
\left[
A_{\pi}^{\mathrm{info}}(s,a_t)
\right]
\right].
\label{eq:state_rep_iterated}
\end{align}

On the event $\{s_t=s\}$,
\begin{equation}
\mathbb{E}_{a_t\sim\tilde{\pi}(\cdot|s_t)}
\left[
A_{\pi}^{\mathrm{info}}(s,a_t)
\right]
=
\sum_a
\tilde{\pi}(a|s)
A_{\pi}^{\mathrm{info}}(s,a).
\end{equation}
This quantity depends only on the fixed state $s$ and is therefore
constant with respect to the outer expectation over $h_t$. Hence,
\begin{align}
&
\mathbb{E}_{\tau\sim\tilde{\pi}}
\left[
\Gamma^{\mathrm{info}}(0,t)
\mathbf{1}\{s_t=s\}
A_{\pi}^{\mathrm{info}}(s,a_t)
\right]
\nonumber\\
&=
\mathbb{E}_{h_t\sim\tilde{\pi}}
\left[
\Gamma^{\mathrm{info}}(0,t)
\mathbf{1}\{s_t=s\}
\right]
\sum_a
\tilde{\pi}(a|s)
A_{\pi}^{\mathrm{info}}(s,a).
\label{eq:state_rep_action}
\end{align}

Substituting Eq.~\eqref{eq:state_rep_action} into
Eq.~\eqref{eq:state_rep_decompose} and collecting the contributions
associated with each state yields
\begin{align}
\eta^{\mathrm{info}}(\tilde{\pi})
-
\eta^{\mathrm{info}}(\pi)
&=
\sum_s
\left[
\sum_{t=0}^{\infty}
\mathbb{E}_{\tau\sim\tilde{\pi}}
\left[
\Gamma^{\mathrm{info}}(0,t)
\mathbf{1}\{s_t=s\}
\right]
\right]
\sum_a
\tilde{\pi}(a|s)
A_{\pi}^{\mathrm{info}}(s,a)
\nonumber\\
&=
\sum_s
\mathbb{E}_{\tau\sim\tilde{\pi}}
\left[
\sum_{t=0}^{\infty}
\Gamma^{\mathrm{info}}(0,t)
\mathbf{1}\{s_t=s\}
\right]
\sum_a
\tilde{\pi}(a|s)
A_{\pi}^{\mathrm{info}}(s,a)
\nonumber\\
&=
\sum_s
\nu_{\tilde{\pi}}^{\mathrm{info}}(s)
\sum_a
\tilde{\pi}(a|s)
A_{\pi}^{\mathrm{info}}(s,a),
\end{align}
where the final equality follows from the definition of
$\nu_{\tilde{\pi}}^{\mathrm{info}}(s)$. 
\end{proof}

Corollary~\ref{app:info_state_representation} provides the state-space
representation used to define the surrogate in Eq.~\eqref{eq:surrogate}.
We now turn to bounding the approximation error introduced by replacing
$\nu_{\tilde{\pi}}^{\mathrm{info}}$ with
$\nu_{\pi}^{\mathrm{info}}$. For this purpose, it is convenient to
return to the equivalent trajectory representation and follow the
coupling argument of TRPO \citep{schulman2015trust}.

To do so, we first define $\meanadv(s)$ as the expected information-time advantage of the
candidate policy $\tilde{\pi}$ relative to the current policy $\pi$ at
state $s$:
\begin{align}
\meanadv(s)
=
\mathbb{E}_{a\sim\tilde{\pi}(\cdot|s)}
\left[
A_\pi^{\mathrm{info}}(s,a)
\right].
\end{align}

The Information-Time Performance Difference
(\Cref{app_lemma:im_performance_diff}) can then be written as:
\begin{align}
\eta^{\text{info}}(\tilde\pi) = \eta^{\text{info}}(\pi) + \mathbb{E}_{\tau \sim \tilde\pi}\left[{\sum_{t=0}^{\infty} \Gamma^{\text{info}}(0,t) \meanadv(s_t)}\right]. \label{etadef}
\end{align}
The surrogate objective can be written as
\begin{align}
L_{\pi}^{\text{info}}(\tilde\pi) = \eta^{\text{info}}(\pi) + \mathbb{E}_{\tau \sim \pi}\left[{\sum_{t=0}^{\infty} \Gamma^{\text{info}}(0,t) \meanadv(s_t)}\right]. \label{surrogatedef} 
\end{align}
The two expressions differ only in the trajectory distribution.
the true return uses trajectories generated by the candidate policy
$\tilde{\pi}$, whereas the surrogate uses trajectories generated by
the current policy $\pi$. To compare these distributions under the
state-wise total-variation constraint in
Eq.~(\ref{eq:tv_divergence}), we place the two policies on a common
probability space through a state-wise coupling.

\begin{definition}
$(\pi,\tilde{\pi})$ is an \textit{$\alpha\rho$-coupled policy pair}
if, for every state $s$, there exists a joint distribution
$(a,\tilde{a})|s$ satisfying
$P(a\neq\tilde{a}| s)\le\alpha\rho(s)$,
where $\pi(\cdot|s)$ and $\tilde{\pi}(\cdot|s)$ are the marginal
distributions of $a$ and $\tilde{a}$, respectively.
\end{definition}
By the standard coupling characterization of total variation,
the divergence condition in Eq.~(\ref{eq:tv_divergence}) guarantees
the existence of such a coupling. Thus, the coupling condition used
below is the coupling representation of the state-wise total-variation
condition assumed in Theorem~\ref{theoremLB}.

\begin{lemma}\label{llama_adv_upper_bound}
Given that $\pi, \tilde{\pi}$ are $\alpha\rho$-coupled policies, for all $s$,
\begin{align}
|\bar{A}^{\text{info}}(s)| \le 2 \alpha \rho(s) \max_{s,a}|A_{\pi}^{\text{info}}(s,a)|
\end{align}
\end{lemma}
\begin{proof}
\begin{align}
\bar{A}^{\text{info}}(s) &= \mathbb{E}_{\tilde{a} \sim \tilde{\pi}}\left[{{A}_\pi^{\text{info}}(s,\tilde{a})} \right]=\mathbb{E}_{(a,\tilde{a}) \sim (\pi,\tilde{\pi})}\left[{{A}_\pi^{\text{info}}(s,\tilde{a}) - {A}_\pi^{\text{info}}(s, a)}\right] \\
&= P(a \neq \tilde{a} | s) \mathbb{E}_{(a,\tilde{a}) \sim (\pi,\tilde{\pi}) | a \neq \tilde{a}}\left[{A}_\pi^{\text{info}}(s,\tilde{a}) - {A}_\pi^{\text{info}}(s, a)\right]\\
|\bar{A}^{\text{info}}(s)| &\le \alpha \rho(s) \cdot 2 \max_{s,a}|A_{\pi}^{\text{info}}(s,a)| = 2\alpha\rho(s) \max_{s,a}|A_{\pi}^{\text{info}}(s,a)|
\end{align}
\end{proof}

Next we establish a bound on the cumulative discounted information weights, which will be used to control the surrogate approximation error.

\begin{lemma}
\label{lemma:geometric_sum}
For any sequence of information densities $\{\rho(s_t)\}_{t=0}^\infty$, the following bound holds:
\begin{equation}
    \sum_{t=0}^\infty \Gamma^{\text{info}}(0,t) \rho(s_t) \le \frac{1}{1 - \gamma}.
\end{equation}
\end{lemma}

\begin{proof}
For brevity, let $\rho_t \coloneqq \rho(s_t)$.
Consider the function $f(x)=1-\gamma^x$. Since $f$ is concave, the chord between endpoints 0 and 1 gives
\begin{align}
1-\gamma^{\rho_t}
=
f(\rho_t) \ge
(1-\rho_t)f(0)+\rho_t f(1) =
\rho_t(1-\gamma).
\end{align}
Rearranging terms yields 
\begin{equation}
\rho_t
\le
\frac{1-\gamma^{\rho_t}}{1-\gamma}.
\end{equation}

Multiplying by $\Gamma^{\text{info}}(0,t)$ and summing over $t$:
\begin{align}
    \sum_{t=0}^\infty \Gamma^{\text{info}}(0,t) \rho_t &\le \frac{1}{1 - \gamma} \sum_{t=0}^\infty \Gamma^{\text{info}}(0,t) (1 - \gamma^{\rho_t}) \\
    &= \frac{1}{1 - \gamma} \sum_{t=0}^\infty (\Gamma^{\text{info}}(0,t) - \Gamma^{\text{info}}(0,t+1)) \quad (\text{since } \Gamma^{\text{info}}(0,t+1) = \Gamma^{\text{info}}(0,t) \gamma^{\rho_t})
\end{align}
The sum on the right-hand side is a telescoping sum:
\begin{align}
    \sum_{t=0}^\infty (\Gamma^{\text{info}}(0,t) - \Gamma^{\text{info}}(0,t+1)) = \Gamma^{\text{info}}(0,0) - \lim_{t \to \infty} \Gamma^{\text{info}}(0,t) \le 1.
\end{align}
Thus, $\sum_{t=0}^\infty \Gamma^{\text{info}}(0,t) \rho(s_t) \le \frac{1}{1 - \gamma}$.
\end{proof}

\begin{theorem}  
\label{thm:lower_bound}
Let $L_\pi^{\text{info}}(\tilde{\pi})$ be the surrogate objective and $\eta^{\text{info}}(\tilde{\pi})$ be the true expected return. Assume the policies $(\pi,\tilde{\pi})$ are $\alpha\rho$-coupled. Then:
\begin{equation}
    |\eta^{\text{info}}(\tilde{\pi}) - L_\pi^{\text{info}}(\tilde{\pi})| \le \frac{4 C \alpha^2}{(1 - \gamma)^2},
\end{equation}
where $C = \max\limits_{s,a}|A_{\pi}^{\text{info}}(s,a)|$.
\end{theorem}

\begin{proof}
Let $\Delta = |\eta^{\text{info}}(\tilde{\pi}) - L_\pi^{\text{info}}(\tilde{\pi})|$. The difference between the true and surrogate objectives can be expressed as the expected sum of advantages over the trajectories generated by $\tilde{\pi}$ vs $\pi$:
\begin{align}
    \Delta &= \left| \mathbb{E}_{\tau \sim \tilde\pi}\left[{\sum_{t=0}^{\infty} \Gamma^{\text{info}}(0,t) \bar{A}^{\text{info}}(s_t)}\right] - \mathbb{E}_{\tau \sim \pi}\left[{\sum_{t=0}^{\infty} \Gamma^{\text{info}}(0,t) \bar{A}^{\text{info}}(s_t)}\right] \right|.
\end{align}
We analyze this difference under the joint coupling distribution of the two policies. Let $T_{\text{div}}$ be the time step of the first action divergence between the coupled policies. For $t < T_{\text{div}}$, the trajectories are identical, so the difference is zero. We can decompose the expectation using the indicator function $\mathbf{1}{\{T_{\text{div}}=k\}}$:
\begin{align}
    \Delta &= \left| \mathbb{E}_{(\tau, \tilde{\tau})}\left[ \sum_{k=0}^{\infty} \mathbf{1}{\{T_{\text{div}}=k\}} \left(\sum_{t=k}^{\infty} \Gamma^{\text{info}}(0,t) \bar{A}^{\text{info}}(s_t) \bigg|_{\tilde\tau \sim \tilde{\pi}} - \sum_{t=k}^{\infty} \Gamma^{\text{info}}(0,t) \bar{A}^{\text{info}}(s_t) \bigg|_{\tau \sim \pi} \right) \right] \right|.
\end{align}
Using the triangle inequality and the fact that $\Gamma^{\text{info}}(0,t) = \Gamma^{\text{info}}(0,k)\Gamma^{\text{info}}(k,t)$, we factor out the common discount term up to step $k$:

\begin{equation}
\begin{aligned}
\Delta
&\le
\sum_{k=0}^{\infty}
\mathbb{E}_{(\tau,\tilde{\tau})}
\Bigg[
\mathbf{1}\{T_{\mathrm{div}}=k\}
\Gamma^{\mathrm{info}}(0,k)
\Bigg(
\\
&\qquad\qquad
\left|
\sum_{t=k}^{\infty}
\Gamma^{\mathrm{info}}(k,t)
\bar{A}^{\mathrm{info}}(s_t)
\right|_{\tilde{\tau}\sim\tilde{\pi}}
+
\left|
\sum_{t=k}^{\infty}
\Gamma^{\mathrm{info}}(k,t)
\bar{A}^{\mathrm{info}}(s_t)
\right|_{\tau\sim\pi}
\Bigg)
\Bigg].
\end{aligned}
\label{eq:coupling_tail_bound}
\end{equation}
Leveraging the result from Lemma~\ref{llama_adv_upper_bound}, we obtain:
\begin{align}
    \left| \sum_{t=k}^{\infty} \Gamma^{\text{info}}(k,t) \bar{A}^{\text{info}}(s_t) \right|
    \le \sum_{t=k}^{\infty} \Gamma^{\text{info}}(k,t) \left( 2\alpha  \rho(s_t) C \right) = 2\alpha C \sum_{t=k}^{\infty} \Gamma^{\text{info}}(k,t) \rho(s_t)
\end{align}
Applying Lemma~\ref{lemma:geometric_sum}, we bound the summation term inside the expectation:
\begin{align}
    \Delta 
    &\le \sum_{k=0}^{\infty} \mathbb{E}_{(\tau, \tilde{\tau})}\left[ \mathbf{1}{\{T_{\text{div}}=k\}} \Gamma^{\text{info}}(0,k) \left( \frac{2\alpha C}{1-\gamma} + \frac{2\alpha C}{1-\gamma} \right) \right] \\
    &= \frac{4\alpha C}{1-\gamma} \sum_{k=0}^{\infty} \mathbb{E}_{(\tau, \tilde{\tau})}\left[ \mathbf{1}{\{T_{\text{div}}=k\}} \Gamma^{\text{info}}(0,k) \right].
\end{align}
Next, we bound the probability of divergence. Let $H_k$ denote the coupled history up to state $s_k$,
before the actions at step $k$ are sampled. Since $
\Gamma^{\text{info}}(0,k)
=
\gamma^{\sum_{t=0}^{k-1}\rho(s_t)}
$ is determined by this history, the law of iterated expectations gives
\begin{align}
&
\mathbb{E}_{(\tau,\tilde{\tau})}
\left[
\mathbf{1}{\{T_{\text{div}}=k\}}
\Gamma^{\text{info}}(0,k)
\right]
\\
&=
\mathbb{E}_{(\tau,\tilde{\tau})}
\left[
\Gamma^{\text{info}}(0,k)
\,
\mathbb{E}
\left[
\mathbf{1}{\{T_{\text{div}}=k\}}
|
H_k
\right]
\right].
\end{align}

The event $\{T_{\text{div}}=k\}$ occurs if the two trajectories have
not diverged before step $k$ and their actions differ at step $k$.
Hence,
\begin{equation}
\{T_{\text{div}}=k\}
=
\{T_{\text{div}}\ge k\}
\cap
\{a_k\neq\tilde{a}_k\}.
\end{equation}
Because whether $T_{\text{div}}\ge k$ holds is already determined by
the history $H_k$,
\begin{align}
\mathbb{E}
\left[
\mathbf{1}{\{T_{\text{div}}=k\}}
|  H_k
\right]
&=\mathbb{E}
\left[
\mathbf{1}{\{T_{\text{div}}\ge k\}}\mathbf{1}{\{a_k\neq\tilde{a}_k\}}
|  H_k
\right] \\
&=
\mathbf{1}{\{T_{\text{div}}\ge k\}}
P(a_k\neq\tilde{a}_k|H_k)
\\
&\le
\mathbf{1}{\{T_{\text{div}}\ge k\}}
\alpha\rho(s_k),
\end{align}
where the inequality follows from the $\alpha\rho$-coupling assumption,
since on $\{T_{\text{div}}\ge k\}$ the two trajectories share the same
state $s_k$.

Substituting this inequality into the previous expression gives
\begin{align}
&
\mathbb{E}_{(\tau,\tilde{\tau})}
\left[
\mathbf{1}{\{T_{\text{div}}=k\}}
\Gamma^{\text{info}}(0,k)
\right]
\\
&\le
\alpha
\mathbb{E}_{(\tau,\tilde{\tau})}
\left[
\mathbf{1}{\{T_{\text{div}}\ge k\}}
\Gamma^{\text{info}}(0,k)
\rho(s_k)
\right]
\\
&\le
\alpha
\mathbb{E}_{(\tau,\tilde{\tau})}
\left[
\Gamma^{\text{info}}(0,k)
\rho(s_k)
\right]
\\
&=
\alpha
\mathbb{E}_{\tau\sim\pi}
\left[
\Gamma^{\text{info}}(0,k)
\rho(s_k)
\right],
\end{align}
where the second inequality uses
$\mathbf{1}{\{T_{\text{div}}\ge k\}}\le1$.
In the last equality, $s_k$ and $\Gamma^{\text{info}}(0,k)$ are taken
along the $\pi$-trajectory; the equality follows because the
$\pi$-marginal of the trajectory coupling is distributed according to
$\pi$.

Substituting this bound back into the summation:
\begin{align}
\Delta
&\le
\frac{4\alpha^2 C}{1-\gamma}
\mathbb{E}_{\tau\sim\pi}
\left[
\sum_{k=0}^{\infty}
\Gamma^{\text{info}}(0,k)
\rho(s_k)
\right].
\label{eq:final_sum}
\end{align}

Finally, applying Lemma~\ref{lemma:geometric_sum} again to the
summation term in Eq.~\eqref{eq:final_sum}:
\begin{equation}
\Delta
\le
\frac{4\alpha^2 C}{1-\gamma}
\cdot
\frac{1}{1-\gamma}
=
\frac{4 C \alpha^2}{(1-\gamma)^2}.
\end{equation}
\end{proof}

\subsection{Theoretical Analysis of Information-Time Policy Optimization}\label{app:proof_policy_optimization}

To connect the entropy-based information density to the policy update analysis, we examine the policy change induced by the information-time
surrogate at a fixed state $s$. With the information-time advantage evaluated under the old policy, its action-dependent component can be
written as
\begin{equation}
\sum_{a}
\tilde{\pi}(a|s)
A_{\pi_{\mathrm{old}}}^{\mathrm{info}}(s,a)=
\mathbb{E}_{a\sim\pi_{\mathrm{old}}(\cdot|s)}
\left[
\frac{\tilde{\pi}(a|s)}
{\pi_{\mathrm{old}}(a|s)}
A_{\pi_{\mathrm{old}}}^{\mathrm{info}}(s,a)
\right].
\label{eq:conditional_surrogate_appendix}
\end{equation}
The following proposition characterizes the
policy divergence induced by a local gradient step on this objective.

\begin{proposition}
Suppose that $\pi_{\mathrm{old}}$ is parameterized by a softmax
distribution and that $\pi_{\mathrm{new}}$ is obtained by one local
softmax-logit gradient step induced by the conditional surrogate, with
step size $0\leq\beta\leq\bar{\beta}$. Then, for every state $s$,
\begin{align}
D_{\mathrm{TV}}
\left(
\pi_{\mathrm{new}}(\cdot|s),
\pi_{\mathrm{old}}(\cdot|s)
\right)
\leq
\frac{\bar{\beta}C}{2}
\left(
1-\exp\{-\mathcal H_{\mathrm{old}}(s)\}
\right)
\leq
\frac{\bar{\beta}C\mathcal H_{\max}}{2}
\rho(s),
\label{eq:entropy_local_tv}
\end{align}
where
$C=\max_{s,a}|A_{\pi_{\mathrm{old}}}^{\mathrm{info}}(s,a)|$.
\end{proposition}

\begin{proof}
Fix an arbitrary state $s$. By the definition of $C$,
\begin{equation}
\left|
A_{\pi_{\mathrm{old}}}^{\mathrm{info}}(s,a)
\right|
\leq C
\qquad
\text{for every }a\in\mathcal A.
\label{eq:uniform_advantage_bound}
\end{equation}
Moreover, the information-time advantage has zero expectation under
the old policy,
\begin{equation}
\sum_{a\in\mathcal A}
\pi_{\mathrm{old}}(a|s)
A_{\pi_{\mathrm{old}}}^{\mathrm{info}}(s,a)
=
0.
\label{eq:advantage_zero_mean_local}
\end{equation}
Therefore, for any action $a$,
\begin{align}
A_{\pi_{\mathrm{old}}}^{\mathrm{info}}(s,a)
&=
A_{\pi_{\mathrm{old}}}^{\mathrm{info}}(s,a)
-
\sum_{b\in\mathcal A}
\pi_{\mathrm{old}}(b|s)
A_{\pi_{\mathrm{old}}}^{\mathrm{info}}(s,b)
\nonumber\\
&=
\sum_{b\neq a}
\pi_{\mathrm{old}}(b|s)
\left[
A_{\pi_{\mathrm{old}}}^{\mathrm{info}}(s,a)
-
A_{\pi_{\mathrm{old}}}^{\mathrm{info}}(s,b)
\right].
\label{eq:advantage_pairwise_general}
\end{align}
Using Eq.~\eqref{eq:uniform_advantage_bound},
\begin{align}
\left|
A_{\pi_{\mathrm{old}}}^{\mathrm{info}}(s,a)
\right|
&\leq
\sum_{b\neq a}
\pi_{\mathrm{old}}(b|s)
\left|
A_{\pi_{\mathrm{old}}}^{\mathrm{info}}(s,a)
-
A_{\pi_{\mathrm{old}}}^{\mathrm{info}}(s,b)
\right|
\nonumber\\
&\leq
2C
\sum_{b\neq a}
\pi_{\mathrm{old}}(b|s)
\nonumber\\
&=
2C
\left(
1-\pi_{\mathrm{old}}(a|s)
\right).
\label{eq:advantage_probability_bound}
\end{align}
Multiplying by $\pi_{\mathrm{old}}(a|s)$ and summing over actions gives
\begin{align}
&
\sum_{a\in\mathcal A}
\pi_{\mathrm{old}}(a|s)
\left|
A_{\pi_{\mathrm{old}}}^{\mathrm{info}}(s,a)
\right|
\nonumber\\
&\qquad\leq
2C
\sum_{a\in\mathcal A}
\pi_{\mathrm{old}}(a|s)
\left(
1-\pi_{\mathrm{old}}(a|s)
\right)
\nonumber\\
&\qquad=
2C
\left(
1-
\sum_{a\in\mathcal A}
\pi_{\mathrm{old}}(a|s)^2
\right).
\label{eq:weighted_advantage_collision}
\end{align}

By concavity of the logarithm,
\begin{align}
-\mathcal H_{\mathrm{old}}(s)
=
\sum_{a\in\mathcal A}
\pi_{\mathrm{old}}(a|s)
\log\pi_{\mathrm{old}}(a|s)
\leq
\log
\left(
\sum_{a\in\mathcal A}
\pi_{\mathrm{old}}(a|s)^2
\right).
\end{align}
Exponentiating both sides yields
\begin{equation}
\exp\{-\mathcal H_{\mathrm{old}}(s)\}
\leq
\sum_{a\in\mathcal A}
\pi_{\mathrm{old}}(a|s)^2.
\label{eq:entropy_collision_relation}
\end{equation}
Combining
Eq.~\eqref{eq:weighted_advantage_collision}
and
Eq.~\eqref{eq:entropy_collision_relation},
\begin{align}
&
\sum_{a\in\mathcal A}
\pi_{\mathrm{old}}(a|s)
\left|
A_{\pi_{\mathrm{old}}}^{\mathrm{info}}(s,a)
\right|
\leq
2C
\left(
1-\exp\{-\mathcal H_{\mathrm{old}}(s)\}
\right).
\label{eq:entropy_weighted_advantage_bound}
\end{align}

Let $z_{\mathrm{old}}(s,\cdot)$ denote the logits of
$\pi_{\mathrm{old}}(\cdot|s)$, and let
$\pi_z(\cdot|s)$ denote the softmax policy induced by logits
$z(s,\cdot)$. Holding the information-time advantage under the old
policy fixed, the conditional surrogate in
Eq.~\eqref{eq:conditional_surrogate_appendix} can be written as
\begin{equation}
\sum_{b\in\mathcal A}
\pi_z(b|s)
A_{\pi_{\mathrm{old}}}^{\mathrm{info}}(s,b).
\label{eq:conditional_surrogate_logit_form}
\end{equation}
For the softmax mapping,
\begin{equation}
\frac{\partial\pi_z(b|s)}
{\partial z(s,a)}
=
\pi_z(b|s)
\left(
\mathbf{1}\{a=b\}
-
\pi_z(a|s)
\right).
\label{eq:softmax_derivative_local}
\end{equation}
Evaluating the gradient of
Eq.~\eqref{eq:conditional_surrogate_logit_form}
at $z=z_{\mathrm{old}}$ and using
Eq.~\eqref{eq:advantage_zero_mean_local} gives
\begin{align}
&
\left.
\frac{\partial}{\partial z(s,a)}
\sum_{b\in\mathcal A}
\pi_z(b|s)
A_{\pi_{\mathrm{old}}}^{\mathrm{info}}(s,b)
\right|_{z=z_{\mathrm{old}}}
\nonumber\\
&=
\pi_{\mathrm{old}}(a|s)
A_{\pi_{\mathrm{old}}}^{\mathrm{info}}(s,a)
-
\pi_{\mathrm{old}}(a|s)
\sum_{b\in\mathcal A}
\pi_{\mathrm{old}}(b|s)
A_{\pi_{\mathrm{old}}}^{\mathrm{info}}(s,b)
\nonumber\\
&=
\pi_{\mathrm{old}}(a|s)
A_{\pi_{\mathrm{old}}}^{\mathrm{info}}(s,a).
\label{eq:conditional_surrogate_logit_gradient}
\end{align}
Hence, the local softmax-logit gradient step in the proposition
satisfies
\begin{equation}
\Delta z(s,a)
\coloneqq
z_{\mathrm{new}}(s,a)
-
z_{\mathrm{old}}(s,a)
=
\beta
\pi_{\mathrm{old}}(a|s)
A_{\pi_{\mathrm{old}}}^{\mathrm{info}}(s,a).
\label{eq:local_logit_update}
\end{equation}
Taking the $\ell_1$ norm and using
Eq.~\eqref{eq:entropy_weighted_advantage_bound},
\begin{align}
\|\Delta z(s,\cdot)\|_1
&=
\beta
\sum_{a\in\mathcal A}
\pi_{\mathrm{old}}(a|s)
\left|
A_{\pi_{\mathrm{old}}}^{\mathrm{info}}(s,a)
\right|
\nonumber\\
&\leq
2\bar{\beta}C
\left(
1-\exp\{-\mathcal H_{\mathrm{old}}(s)\}
\right).
\label{eq:entropy_logit_movement_general}
\end{align}

Consider the line segment
\begin{equation}
z_\xi(s,\cdot)
=
z_{\mathrm{old}}(s,\cdot)
+
\xi\Delta z(s,\cdot),
\qquad
\xi\in[0,1].
\end{equation}
Let $J(\xi)$ denote the Jacobian of the softmax probability vector
with respect to the logits, evaluated at $z_\xi(s,a)$. Its entries are
\begin{equation}
J_{ij}(\xi)
=
\left.
\frac{
\partial \pi_z(i|s)
}{
\partial z(s,j)
}
\right|_{z=z_\xi}
=
\pi_{z_\xi}(i|s)
\left(
\mathbf{1}\{i=j\}
-
\pi_{z_\xi}(j|s)
\right).
\label{eq:softmax_jacobian_entries}
\end{equation}
For every column $j$,
\begin{align}
\sum_i
|J_{ij}(\xi)|
=
2\pi_{z_\xi}(j|s)
\left(
1-\pi_{z_\xi}(j|s)
\right)
\leq
\frac{1}{2}.
\end{align}
Therefore,
\begin{equation}
\|J(\xi)\|_{1\rightarrow1}
\leq
\frac{1}{2}.
\label{eq:softmax_jacobian_l1_bound}
\end{equation}

By the fundamental theorem of calculus,
\begin{align}
&
\left\|
\pi_{\mathrm{new}}(\cdot|s)
-
\pi_{\mathrm{old}}(\cdot|s)
\right\|_1
\nonumber\\
&=
\left\|
\int_0^1
J(\xi)\Delta z(s,\cdot)
\,d\xi
\right\|_1
\nonumber\\
&\leq
\int_0^1
\|J(\xi)\|_{1\rightarrow1}
\|\Delta z(s,\cdot)\|_1
\,d\xi
\nonumber\\
&\leq
\frac{1}{2}
\|\Delta z(s,\cdot)\|_1.
\label{eq:softmax_l1_policy_bound}
\end{align}
Thus,
\begin{align}
&D_{\mathrm{TV}}
\left(
\pi_{\mathrm{new}}(\cdot|s),
\pi_{\mathrm{old}}(\cdot|s)
\right)
\nonumber\\
&\qquad=
\frac{1}{2}
\left\|
\pi_{\mathrm{new}}(\cdot|s)
-
\pi_{\mathrm{old}}(\cdot|s)
\right\|_1
\nonumber\\
&\qquad\leq
\frac{1}{4}
\|\Delta z(s,\cdot)\|_1
\nonumber\\
&\qquad\leq
\frac{\bar{\beta}C}{2}
\left(
1-\exp\{-\mathcal H_{\mathrm{old}}(s)\}
\right).
\label{eq:entropy_tv_first_bound}
\end{align}
Finally, since $1-e^{-x}\leq x$ for $x\geq0$ and
$\mathcal H_{\mathrm{old}}(s)=\mathcal H_{\max}\rho(s)$,
\begin{equation}
D_{\mathrm{TV}}
\left(
\pi_{\mathrm{new}}(\cdot|s),
\pi_{\mathrm{old}}(\cdot|s)
\right)
\leq
\frac{\bar{\beta}C\mathcal H_{\max}}{2}
\rho(s).
\end{equation}
Since $s$ was arbitrary, the result holds for every state.
\end{proof}

By setting
$\alpha=
\bar{\beta}C\mathcal H_{\max}/{2}$,
the proposition yields
\begin{equation}
D_{\mathrm{TV}}
\left(
\pi_{\mathrm{new}}(\cdot|s),
\pi_{\mathrm{old}}(\cdot|s)
\right)
\leq
\alpha\rho(s),
\qquad
\forall s,
\label{eq:entropy_density_tv_condition}
\end{equation}
which is the state-dependent divergence condition in
Eq.~\eqref{eq:tv_divergence}.

To connect this bound to the coupling formulation used in the proof of
Theorem~\ref{thm:lower_bound}, we first define the common probability mass:
\begin{equation}
m_a
\coloneqq
\min
\left\{
\pi_{\mathrm{old}}(a|s),
\pi_{\mathrm{new}}(a|s)
\right\},
\qquad
c
\coloneqq
\sum_{a\in\mathcal A}m_a.
\label{eq:common_probability_mass}
\end{equation}
Using
\begin{equation}
\min\{x,y\}
=
\frac{x+y-|x-y|}{2},
\end{equation}
we obtain
\begin{align}
1-c
&=
\frac{1}{2}
\sum_{a\in\mathcal A}
\left|
\pi_{\mathrm{new}}(a|s)
-
\pi_{\mathrm{old}}(a|s)
\right|
\nonumber\\
&=
D_{\mathrm{TV}}
\left(
\pi_{\mathrm{new}}(\cdot|s),
\pi_{\mathrm{old}}(\cdot|s)
\right).
\label{eq:common_mass_tv}
\end{align}

The quantity $m_a$ represents the probability mass shared by the two
policies at action $a$. We now construct a coupling of $\pi_{\text{old}}(\cdot|s)$ and $\pi_{\text{new}}(\cdot|s)$ by assigning this common mass to identical action pairs. Specifically, for each action $a$, we assign
\begin{equation}
P
\left(
a_{\mathrm{old}}=a,
a_{\mathrm{new}}=a
\mid s
\right)
=
m_a.
\label{eq:common_mass_assignment}
\end{equation}
Since $
\sum_{a\in\mathcal A}m_a=c$,
the total probability assigned to identical action pairs is $c$.

After removing this common mass, the remaining probability masses of
the two policies are
\begin{equation}
\pi_{\mathrm{old}}(a|s)-m_a
\qquad\text{and}\qquad
\pi_{\mathrm{new}}(a|s)-m_a,\nonumber
\end{equation}
respectively. Moreover,
\begin{align}
\sum_{a\in\mathcal A}
\left(
\pi_{\mathrm{old}}(a|s)-m_a
\right)
&=
1-c,
\nonumber\\
\sum_{a\in\mathcal A}
\left(
\pi_{\mathrm{new}}(a|s)-m_a
\right)
&=
1-c.
\end{align}
Thus, when $c<1$, the normalized remaining distributions are
\begin{equation}
\frac{
\pi_{\mathrm{old}}(a|s)-m_a
}{
1-c
},
\qquad
\frac{
\pi_{\mathrm{new}}(a|s)-m_a
}{
1-c
}.
\label{eq:residual_distributions}
\end{equation}
The remaining probability mass $1-c$ is assigned according to these
two distributions.

By the definition of $m_a$, for every action $a$, at least one of
\begin{equation}
\pi_{\mathrm{old}}(a|s)-m_a
\qquad\text{and}\qquad
\pi_{\mathrm{new}}(a|s)-m_a\nonumber
\end{equation}
is zero. Hence, the two remaining distributions have disjoint
supports. The sampled actions therefore agree on the common mass and
differ almost surely on the remaining mass. Consequently,
\begin{align}
P
\left(
a_{\mathrm{old}}\neq a_{\mathrm{new}}
\mid s
\right)
=
1-c
=
D_{\mathrm{TV}}
\left(
\pi_{\mathrm{old}}(\cdot|s),
\pi_{\mathrm{new}}(\cdot|s)
\right)\le \alpha \rho(s).
\label{eq:maximal_coupling_tv}
\end{align}
Hence,
$(\pi_{\mathrm{old}},\pi_{\mathrm{new}})$
forms an $\alpha\rho$-coupled policy pair in the sense used in the proof of Theorem~\ref{thm:lower_bound}.

\begin{proposition}
\label{app:recomputed_density_bound}
Suppose that the terminal reward is bounded by
$|R_T|\leq R_{\max}<\infty$.
Under the update from $\pi_k$ to $\pi_{k+1}$ in
Proposition~\ref{proposition:entropy_policy_movement}, with $\alpha=\bar{\beta}C\mathcal H_{\max}/2$, the change in
information-time return caused by recomputing the information density
satisfies
\begin{equation}
\left|
\eta_{\rho_{k+1}}^{\mathrm{info}}(\pi_{k+1})
-
\eta_{\rho_k}^{\mathrm{info}}(\pi_{k+1})
\right|
\leq
\frac{8R_{\max}}{e}
\alpha.
\label{app_eq:recomputed_density_perturbation_bound}
\end{equation}
\end{proposition}

\begin{proof}
We first bound the relative change in the state-wise information density, then propagate this bound to the accumulated information time along a trajectory, and finally control the resulting change in the terminal information-time discount.

\paragraph{Step 1: Control of the local logit movement.}

Fix an arbitrary state $s$ and define
\[
\Delta z_k(s,\cdot)
\coloneqq
z_{k+1}(s,\cdot)-z_k(s,\cdot).
\]
From Proposition~\ref{proposition:entropy_policy_movement},
\begin{align}
\|\Delta z_k(s,\cdot)\|_1
&\le
2\bar{\beta}C
\left(
1-e^{-\mathcal H_k(s)}
\right)
\nonumber\\
&\le
2\bar{\beta}C\mathcal H_k(s)
\nonumber\\
&=
2\bar{\beta}C\mathcal H_{\max}\rho_k(s)
\nonumber\\
&=
4\alpha\rho_k(s),
\end{align}
where the second inequality follows from
$1-e^{-x}\le x$ for $x\ge0$.
Since $\|v\|_\infty\le\|v\|_1$, we have
\begin{equation}
\|\Delta z_k(s,\cdot)\|_\infty
\le
4\alpha\rho_k(s).
\label{eq:recomputed_logit_movement}
\end{equation}

\paragraph{Step 2: Multiplicative stability of the information density.}

For a generic softmax distribution $\pi_z(\cdot|s)$, We write the corresponding predictive entropy as
\[
\mathcal H_z(s)
=
-\sum_a
\pi_z(a|s)\log\pi_z(a|s).
\]
Its derivative with respect to the logit $z_a$ is
\begin{equation}
\frac{\partial\mathcal H_z(s)}{\partial z(s,a)}
=
-\pi_z(a|s)
\left(
\log\pi_z(a|s)+\mathcal H_z(s)
\right).
\end{equation}
Hence
\begin{align}
\|\nabla_z\mathcal H_z(s)\|_1
&=
\sum_a
\pi_z(a|s)
\left|
\log\pi_z(a|s)+\mathcal H_z(s)
\right|
\nonumber\\
&\le
\sum_a
\pi_z(a|s)
\left(
-\log\pi_z(a|s)+\mathcal H_z(s)
\right)
\nonumber\\
&=
2\mathcal H_z(s).
\label{eq:entropy_gradient_bound_recompute}
\end{align}

Consider the interpolation
\[
z_\xi(s,\cdot)
=
z_k(s,\cdot)+\xi\Delta z_k(s,\cdot),
\qquad
\xi\in[0,1],
\]
By the chain rule, Hölder's inequality, and
Eq.~\eqref{eq:entropy_gradient_bound_recompute},
\begin{align}
\left|
\frac{d}{d\xi}\mathcal H_{z_\xi}(s)
\right|
&=
\left|
\nabla_z\mathcal H_{z_\xi}(s)^\top
\Delta z_k(s,\cdot)
\right|
\nonumber\\
&\le
\|\nabla_z\mathcal H_{z_\xi}(s)\|_1
\|\Delta z_k(s,\cdot)\|_\infty
\nonumber\\
&\le
2\mathcal H_{z_\xi}(s)
\|\Delta z_k(s,\cdot)\|_\infty.
\end{align}
For positive entropy,
\begin{equation}
\left|
\frac{d}{d\xi}\log\mathcal H_{z_{\xi}}(s)
\right|
\le
2\|\Delta z_k(s,\cdot)\|_\infty.
\end{equation}
Integrating over $\xi\in[0,1]$ yields
\begin{align}
\left|
\log\mathcal H_{z_1}(s)
-
\log\mathcal H_{z_0}(s)
\right|
&=
\left|
\int_0^1
\frac{d}{d\xi}
\log\mathcal H_{z_\xi}(s)
\,d\xi
\right|
\nonumber\\
&\le
\int_0^1
\left|
\frac{d}{d\xi}
\log\mathcal H_{z_\xi}(s)
\right|
\,d\xi
\nonumber\\
&\le
2\|\Delta z_k(s,\cdot)\|_\infty.
\label{eq:integrated_log_entropy_bound}
\end{align}

By construction,
\[
z_0(s,\cdot)=z_k(s,\cdot),
\qquad
z_1(s,\cdot)=z_{k+1}(s,\cdot),
\]
and hence
\[
\mathcal H_{z_0}(s)=\mathcal H_k(s),
\qquad
\mathcal H_{z_1}(s)=\mathcal H_{k+1}(s).
\]
Therefore,
\begin{equation}
\left|
\log
\frac{\mathcal H_{k+1}(s)}
{\mathcal H_k(s)}
\right|
\le
2\|\Delta z_k(s,\cdot)\|_\infty.
\end{equation}
Finally, applying
Eq.~\eqref{eq:recomputed_logit_movement} gives
\begin{equation}
\left|
\log
\frac{\mathcal H_{k+1}(s)}
{\mathcal H_k(s)}
\right|
\le
8\alpha\rho_k(s).
\label{eq:entropy_log_ratio_recompute}
\end{equation}

Equivalently,
\begin{equation}
e^{-8\alpha\rho_k(s)}
\mathcal H_k(s)
\le
\mathcal H_{k+1}(s)
\le
e^{8\alpha\rho_k(s)}
\mathcal H_k(s).
\end{equation}
Dividing both sides by the normalization scale
$\mathcal H_{\max}$ gives
\begin{equation}
e^{-8\alpha\rho_k(s)}
\rho_k(s)
\le
\rho_{k+1}(s)
\le
e^{8\alpha\rho_k(s)}
\rho_k(s).
\label{eq:density_multiplicative_stability_recompute}
\end{equation}

\paragraph{Step 3: Stability of the accumulated information time.}

Fix a finite trajectory
\[
\tau=(s_0,a_0,\ldots,s_T)
\]
and define
\begin{equation}
I_k(\tau)
\coloneqq
\sum_{t=0}^{T-1}\rho_k(s_t),
\qquad
I_{k+1}(\tau)
\coloneqq
\sum_{t=0}^{T-1}\rho_{k+1}(s_t).
\end{equation}
Also define
\begin{equation}
\rho_{k,\max}(\tau)
\coloneqq
\max_{0\le t<T}\rho_k(s_t).
\end{equation}

By Eq.~\eqref{eq:density_multiplicative_stability_recompute},
for every $t$,
\begin{equation}
e^{-8\alpha\rho_k(s_t)}
\rho_k(s_t)
\le
\rho_{k+1}(s_t)
\le
e^{8\alpha\rho_k(s_t)}
\rho_k(s_t).
\end{equation}
Since
$\rho_k(s_t)\le\rho_{k,\max}(\tau)$,
\begin{equation}
e^{-8\alpha\rho_{k,\max}(\tau)}
\rho_k(s_t)
\le
\rho_{k+1}(s_t)
\le
e^{8\alpha\rho_{k,\max}(\tau)}
\rho_k(s_t).
\end{equation}
Summing over $t=0,\ldots,T-1$ gives
\begin{equation}
e^{-8\alpha\rho_{k,\max}(\tau)}
I_k(\tau)
\le
I_{k+1}(\tau)
\le
e^{8\alpha\rho_{k,\max}(\tau)}
I_k(\tau).
\label{eq:information_time_multiplicative_stability}
\end{equation}

If $I_k(\tau)=0$, then nonnegativity of the information density implies
$\rho_k(s_t)=0$ for every $t$. Eq.~\eqref{eq:density_multiplicative_stability_recompute} then implies
$\rho_{k+1}(s_t)=0$ for every $t$, and therefore
$I_{k+1}(\tau)=0$. In this case the two trajectory discount factors
are identical.

Now suppose that $I_k(\tau)>0$. Eq.~\eqref{eq:information_time_multiplicative_stability} implies that
$I_{k+1}(\tau)>0$, and taking logarithms yields
\begin{equation}
\left|
\log I_{k+1}(\tau)
-
\log I_k(\tau)
\right|
\le
8\alpha\rho_{k,\max}(\tau).
\label{eq:information_time_log_stability}
\end{equation}

\paragraph{Step 4: Log-Lipschitz continuity of the information-time discount.}

We use the following elementary inequality. For any
$\gamma\in(0,1)$ and any $x,y>0$,
\begin{equation}
|\gamma^x-\gamma^y|
\le
\frac{1}{e}
|\log x-\log y|.
\label{eq:discount_log_lipschitz}
\end{equation}
To prove this inequality, consider the function
\begin{equation}
f(x)
\coloneqq
\gamma^{e^x}
=
\exp\!\left\{
(\log\gamma)e^x
\right\},
\qquad
x\in\mathbb R.
\end{equation}
Its derivative is
\begin{align}
f'(x)
&=
(\log\gamma)e^x
\exp\!\left\{
(\log\gamma)e^x
\right\}.
\end{align}
Since $\log\gamma<0$, letting
$q=(-\log\gamma)e^x>0$, we can get
\begin{equation}
|f'(x)|
=
q e^{-q}
\le
\frac{1}{e},
\end{equation}
where the last inequality follows because $q e^{-q}$ attains its
maximum at $q=1$. By applying the mean value theorem, we finally get
Eq.~\eqref{eq:discount_log_lipschitz}.

For the trajectory $\tau$,
\begin{align}
\Gamma_{\rho_k}^{\mathrm{info}}(0,T)
&=
\prod_{t=0}^{T-1}\gamma^{\rho_k(s_t)}
=
\gamma^{I_k(\tau)},\\
\Gamma_{\rho_{k+1}}^{\mathrm{info}}(0,T)
&=
\prod_{t=0}^{T-1}\gamma^{\rho_{k+1}(s_t)}
=
\gamma^{I_{k+1}(\tau)}.
\end{align}
Therefore, when $I_k(\tau)>0$,
Eqs.~\eqref{eq:discount_log_lipschitz} and
\eqref{eq:information_time_log_stability} give
\begin{align}
\left|
\Gamma_{\rho_{k+1}}^{\mathrm{info}}(0,T)
-
\Gamma_{\rho_k}^{\mathrm{info}}(0,T)
\right|
\le
\frac1e
\left|
\log I_{k+1}(\tau)
-
\log I_k(\tau)
\right|
\le
\frac{8\alpha}{e}
\rho_{k,\max}(\tau).
\label{eq:pathwise_clock_refresh_bound}
\end{align}
As shown above, the same inequality also holds when
$I_k(\tau)=0$, because in that case the left-hand side is zero.
Hence Eq.~\eqref{eq:pathwise_clock_refresh_bound} holds for every
trajectory.

\paragraph{Step 5: From pathwise discount stability to return stability.}

Under the terminal-reward setting,
\begin{equation}
\eta_{\rho}^{\mathrm{info}}(\pi_{k+1})
=
\mathbb E_{\tau\sim\pi_{k+1}}
\left[
\Gamma_{\rho}^{\mathrm{info}}(0,T)R_T
\right].
\end{equation}
Thus
\begin{align}
&
\left|
\eta_{\rho_{k+1}}^{\mathrm{info}}(\pi_{k+1})
-
\eta_{\rho_k}^{\mathrm{info}}(\pi_{k+1})
\right|
\nonumber\\
&=
\left|
\mathbb E_{\tau\sim\pi_{k+1}}
\left[
R_T
\left(
\Gamma_{\rho_{k+1}}^{\mathrm{info}}(0,T)
-
\Gamma_{\rho_k}^{\mathrm{info}}(0,T)
\right)
\right]
\right|
\nonumber\\
&\le
R_{\max}
\mathbb E_{\tau\sim\pi_{k+1}}
\left[
\left|
\Gamma_{\rho_{k+1}}^{\mathrm{info}}(0,T)
-
\Gamma_{\rho_k}^{\mathrm{info}}(0,T)
\right|
\right]
\nonumber\\
&\le
\frac{8R_{\max}\alpha}{e}
\mathbb E_{\tau\sim\pi_{k+1}}
\left[
\rho_{k,\max}(\tau)
\right].
\label{eq:trajectory_dependent_recompute_bound}
\end{align}
Finally, since $\rho_k(s)\in[0,1]$,
\[
\rho_{k,\max}(\tau)\le1
\]
for every trajectory. Therefore,
\begin{equation}
\left|
\eta_{\rho_{k+1}}^{\mathrm{info}}(\pi_{k+1})
-
\eta_{\rho_k}^{\mathrm{info}}(\pi_{k+1})
\right|
\le
\frac{8R_{\max}}{e}\alpha.
\end{equation}
This proves the proposition.
\end{proof}

The preceding proposition characterizes the variation in the information clock
induced by policy updates under a fixed normalization scheme. In
Section~\ref{sec:empirical_information_time}, we further empirically examine
different normalization schemes. Although changing the normalization scale may
introduce additional variation in the resulting information density, more
adaptive normalization can also better capture local information dynamics.
Among the three normalization schemes considered, batch-level normalization
provides a favorable balance between local adaptivity and the stability of the
induced information-density scale.

\begin{proposition}
\label{app_prop:cross_iteration_clock_ordering}
Let $\pi_x$ and $\pi_y$, with $x<y$, be two policy iterates from the
same sequence of updates, and let $\rho_r$ be any fixed
information-density snapshot used as a common reference clock.
Suppose that $|R_T|\le R_{\max}<\infty$. Define the cumulative certified gain between the two iterates as
\begin{equation}
\mathcal G_{x,y}
\coloneqq
\sum_{t=x}^{y-1}
\left[
L_{\pi_t}^{\mathrm{info}}(\pi_{t+1})
-
\eta_{\rho_t}^{\mathrm{info}}(\pi_t)
-
\frac{4C_t\alpha_t^2}{(1-\gamma)^2}
-
\frac{8R_{\max}}{e}\alpha_t
\right],
\label{eq:cumulative_certified_gain}
\end{equation}
where $C_t$ and $\alpha_t$ denote the corresponding quantities for
the update $\pi_t\rightarrow\pi_{t+1}$, and all information-time quantities in the $t$-th summand are evaluated under $\rho_t$. For a trajectory $\tau=(s_0,a_0,\ldots,s_T)$, define $
\bar\rho_j(\tau)
\coloneqq
\frac{1}{T}
\sum_{u=0}^{T-1}\rho_j(s_u),
$
and
$
\varepsilon_{x,y}^{(r)}
\coloneqq
\frac{1}{2}
\sum_{j\in\{x,y\}}
\mathbb E_{\tau\sim\pi_j}
\left[
\left|
\log
\frac{\bar\rho_r(\tau)}
{\bar\rho_j(\tau)}
\right|
\right]
$.
Then
\begin{equation}
\eta_{\rho_r}^{\mathrm{info}}(\pi_y)
-
\eta_{\rho_r}^{\mathrm{info}}(\pi_x)
\ge
\mathcal G_{x,y}
-
\frac{2R_{\max}}{e}
\varepsilon_{x,y}^{(r)}.
\end{equation}
\end{proposition}
\begin{proof}

For each $t\in\{x,\ldots,y-1\}$, the information density $\rho_t$
is fixed during the update from $\pi_t$ to $\pi_{t+1}$.
Hence, by Theorem~\ref{theoremLB},
\begin{align}
\eta_{\rho_t}^{\mathrm{info}}(\pi_{t+1})
-
\eta_{\rho_t}^{\mathrm{info}}(\pi_t)
\ge\;&
L_{\pi_t}^{\mathrm{info}}(\pi_{t+1})
-
\eta_{\rho_t}^{\mathrm{info}}(\pi_t)
-
\frac{4C_t\alpha_t^2}{(1-\gamma)^2}.
\label{eq:cross_iteration_fixed_clock_gain}
\end{align}
Moreover, Proposition~\ref{prop:recomputed_density_bound} gives
\begin{equation}
\eta_{\rho_{t+1}}^{\mathrm{info}}(\pi_{t+1})
-
\eta_{\rho_t}^{\mathrm{info}}(\pi_{t+1})
\ge
-
\frac{8R_{\max}}{e}\alpha_t.
\label{eq:cross_iteration_refresh_gain}
\end{equation}
Adding the two inequalities yields
\begin{align}
\eta_{\rho_{t+1}}^{\mathrm{info}}(\pi_{t+1})
-
\eta_{\rho_t}^{\mathrm{info}}(\pi_t)
\ge
L_{\pi_t}^{\mathrm{info}}(\pi_{t+1})
-
\eta_{\rho_t}^{\mathrm{info}}(\pi_t)
-
\frac{4C_t\alpha_t^2}{(1-\gamma)^2}
-
\frac{8R_{\max}}{e}\alpha_t.
\end{align}
Summing over $t=x,\ldots,y-1$ and telescoping therefore gives
\begin{equation}
\eta_{\rho_y}^{\mathrm{info}}(\pi_y)
-
\eta_{\rho_x}^{\mathrm{info}}(\pi_x)
\ge
\mathcal G_{x,y}.
\label{eq:native_clock_cumulative_gain}
\end{equation}

It remains to evaluate the two endpoint policies under the common
reference clock $\rho_r$. For $j\in\{x,y\}$, let
\[
I_j(\tau)=T\bar\rho_j(\tau),
\qquad
I_r(\tau)=T\bar\rho_r(\tau).
\]
Under the terminal-reward setting and the log-Lipschitz bound in
Eq.~\eqref{eq:discount_log_lipschitz},
\begin{align}
\left|
\eta_{\rho_r}^{\mathrm{info}}(\pi_j)
-
\eta_{\rho_j}^{\mathrm{info}}(\pi_j)
\right|
\le
\frac{R_{\max}}{e}
\mathbb E_{\tau\sim\pi_j}
\left[
\left|
\log\frac{I_r(\tau)}{I_j(\tau)}
\right|
\right]
=
\frac{R_{\max}}{e}
\mathbb E_{\tau\sim\pi_j}
\left[
\left|
\log
\frac{\bar\rho_r(\tau)}
{\bar\rho_j(\tau)}
\right|
\right],
\label{eq:endpoint_reanchoring_bound}
\end{align}
where the trajectory length $T$ cancels in the ratio.

Finally, define
\[
\Delta_j^{(r)}
\coloneqq
\eta_{\rho_r}^{\mathrm{info}}(\pi_j)
-
\eta_{\rho_j}^{\mathrm{info}}(\pi_j).
\]
Then
\begin{align}
\eta_{\rho_r}^{\mathrm{info}}(\pi_y)
-
\eta_{\rho_r}^{\mathrm{info}}(\pi_x)
&=
\eta_{\rho_y}^{\mathrm{info}}(\pi_y)
-
\eta_{\rho_x}^{\mathrm{info}}(\pi_x)
+
\Delta_y^{(r)}
-
\Delta_x^{(r)}
\nonumber\\
&\ge
\mathcal G_{x,y}
-
|\Delta_y^{(r)}|
-
|\Delta_x^{(r)}|
\nonumber\\
&\ge
\mathcal G_{x,y}
-
\frac{2R_{\max}}{e}
\varepsilon_{x,y}^{(r)},
\end{align}
where the last inequality follows from
Eq.~\eqref{eq:endpoint_reanchoring_bound} and the definition of
$\varepsilon_{x,y}^{(r)}$.
This proves the proposition.
\end{proof}
Proposition~\ref{app_prop:cross_iteration_clock_ordering} shows that
recomputing the information clock across training does not by itself
invalidate comparisons between policy iterates. If the cumulative
certified gain exceeds the discrepancy introduced by re-anchoring the
endpoint policies to a common reference clock, the later policy still
achieves a higher information-time return under that clock.
Consequently, smaller clock drift makes the comparison more robust,
although the information clock need not remain fixed throughout
training. We therefore empirically track the evolution of normalized entropy throughout training to characterize how the information clock changes
over the full optimization trajectory.

\newpage
\section{Implementation Details}
\label{app_:implementation_details}

\label{app_:implementation_details}

We train Qwen3-4B Base Model, Qwen3-8B Base Model, and Qwen3-14B Base Model on DAPO-Math-17K.
For both PPO and \methodabb{}, the policy learning rate is
$1\times10^{-6}$ with a 10-step warmup, while the critic learning rate is
$2\times10^{-6}$ without warmup.
The maximum response length is 10k tokens for the 4B and 8B models and
12k tokens for the 14B model.

Unless otherwise specified in the sensitivity experiments, PPO uses
$\gamma=1$ and $\lambda=1$, with clipping parameters
$\epsilon_{\mathrm{low}}=0.2$ and $\epsilon_{\mathrm{high}}=0.28$.
For \methodabb{}, we use $\gamma=0.999$ and $\lambda=0.99$.
The adaptive clipping bounds follow Eq.~(\ref{eq:actor_loss}), with
$\epsilon_{\mathrm{low}}^{\mathrm{info}}=10$ and
$\epsilon_{\mathrm{high}}^{\mathrm{info}}=20$, as supported by the ablation
study below.

\begin{figure*}[ht]
    \centering


    \begin{subfigure}[b]{0.32\textwidth}
        \centering
        \includegraphics[width=\textwidth]{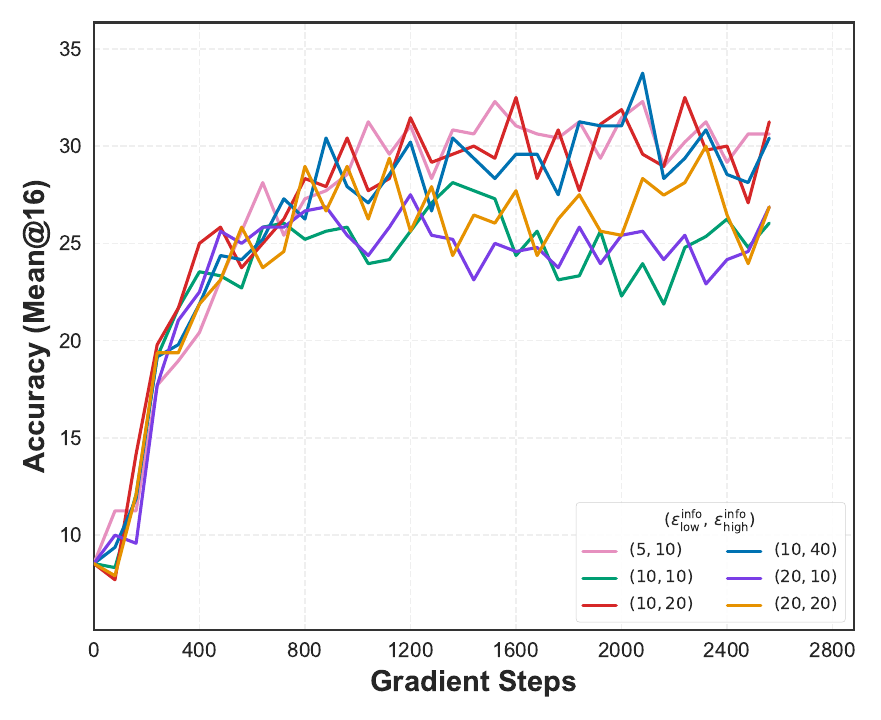}
        \vspace{-0.05in}
        \caption{AIME24 Accuracy.}
        \label{lu_app_fig:info_time_aime24}
    \end{subfigure}
    \hfill
    \begin{subfigure}[b]{0.32\textwidth}
        \centering
        \includegraphics[width=\textwidth]{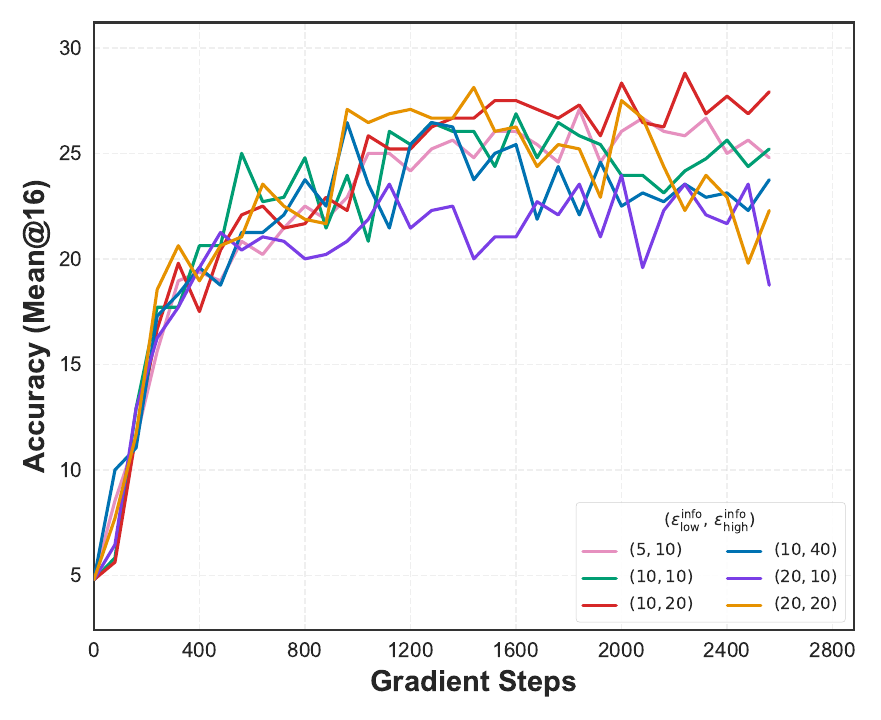}
        \vspace{-0.05in}
        \caption{AIME25 Accuracy.}
        \label{lu_app_fig:info_time_aime25}
    \end{subfigure}
    \hfill
    \begin{subfigure}[b]{0.32\textwidth}
        \centering
        \includegraphics[width=\textwidth]{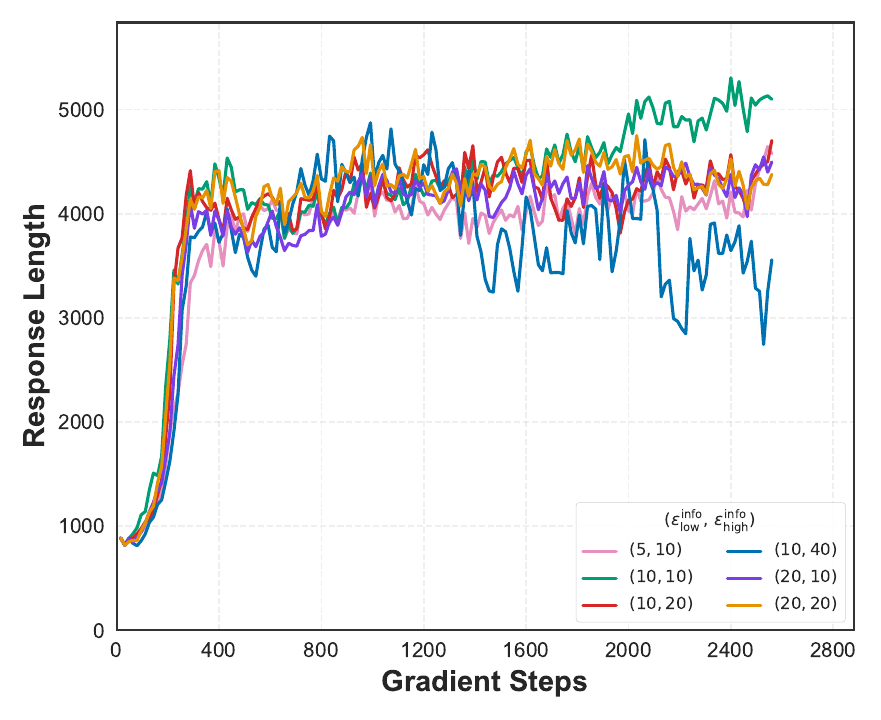}
        \vspace{-0.05in}
        \caption{Response Length.}
        \label{lu_app_fig:response_length}
    \end{subfigure}

    \par\vspace{0.08in}


    \begin{subfigure}[b]{0.32\textwidth}
        \centering
        \includegraphics[width=\textwidth]{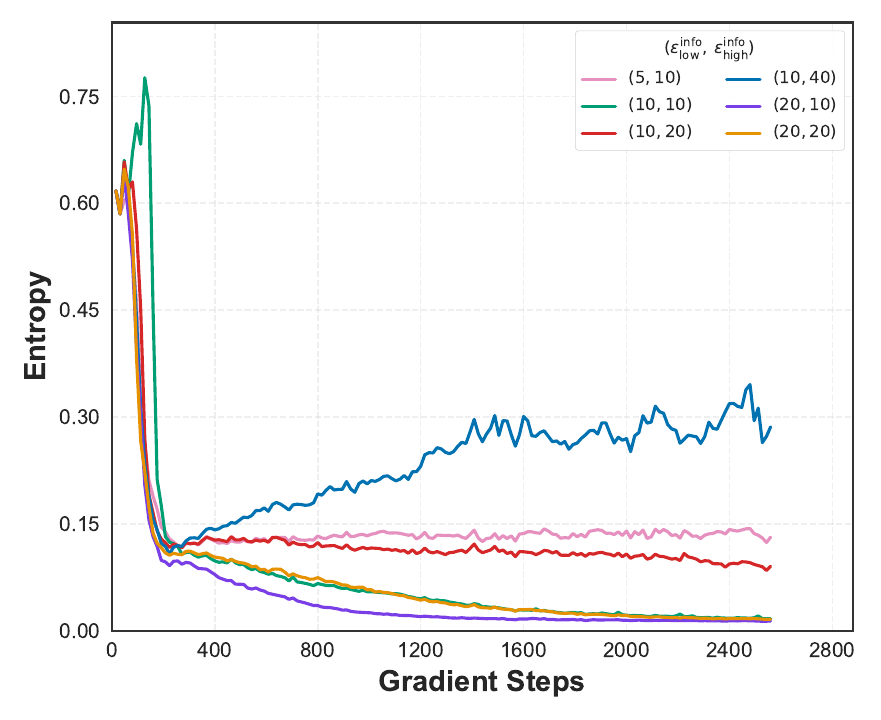}
        \vspace{-0.05in}
        \caption{Raw Predictive Entropy.}
        \label{lu_app_fig:entropy}
    \end{subfigure}
    \hfill
    \begin{subfigure}[b]{0.32\textwidth}
        \centering
        \includegraphics[width=\textwidth]{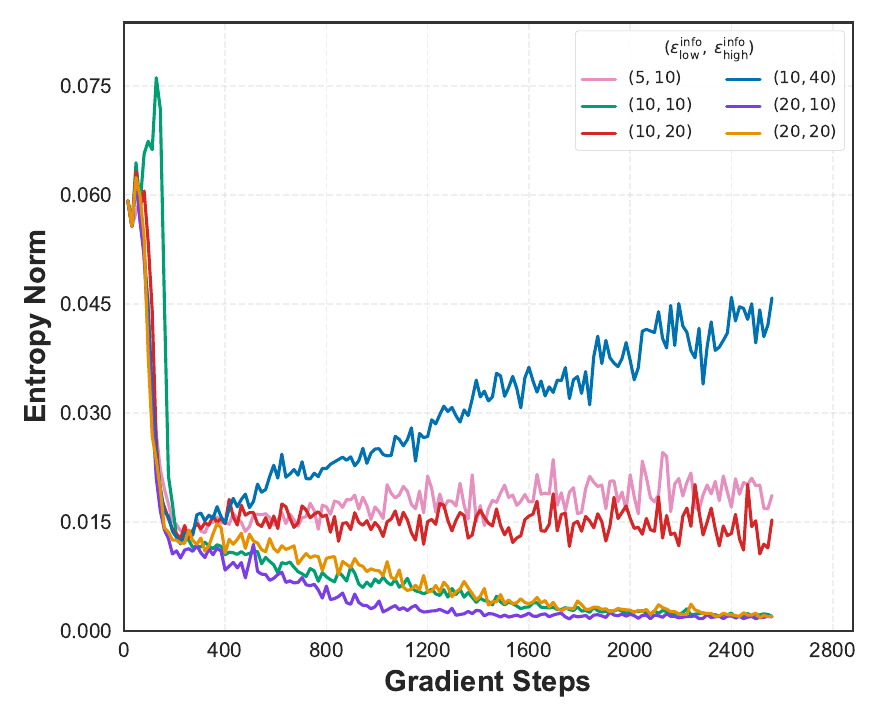}
        \vspace{-0.05in}
        \caption{Normalized Entropy.}
        \label{lu_app_fig:normalized_entropy}
    \end{subfigure}
    \hfill
    \begin{subfigure}[b]{0.32\textwidth}
        \centering
        \includegraphics[width=\textwidth]{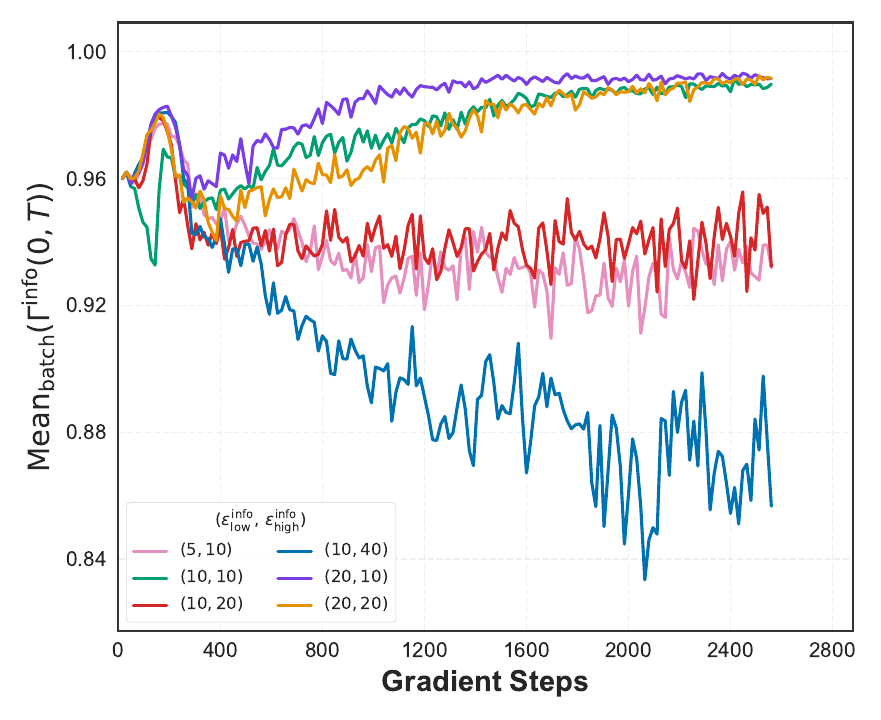}
        \vspace{-0.05in}
        \caption{Mean $\Gamma^{\text{info}}(0,T)$}
        \label{lu_app_fig:mean_gamma}
    \end{subfigure}
\par\vspace{0.08in}


\makebox[\textwidth][c]{%
    \begin{subfigure}[b]{0.32\textwidth}
        \centering
        \includegraphics[width=\textwidth]{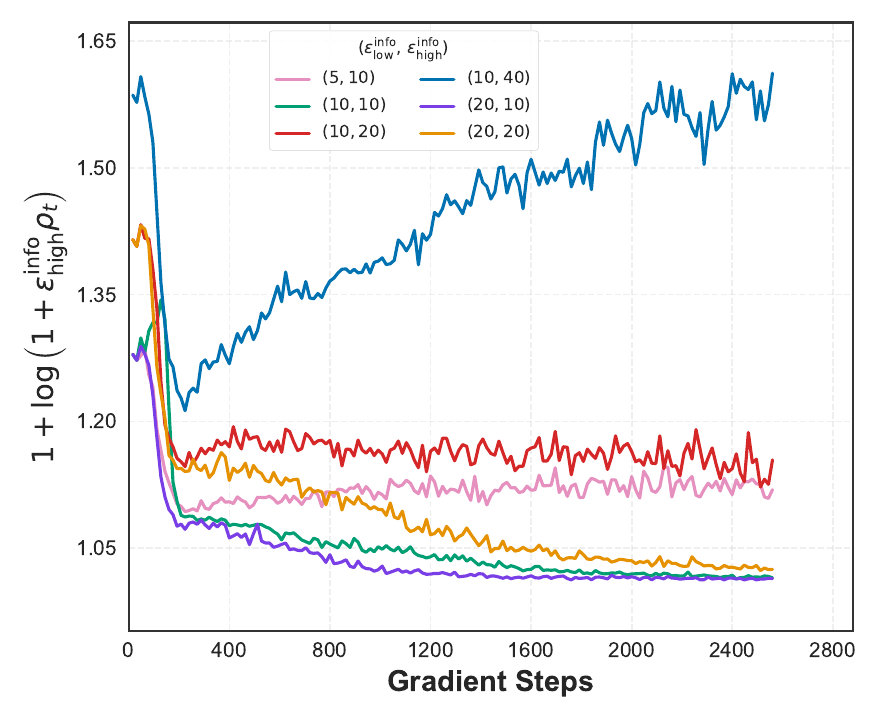}
        \vspace{-0.05in}
        \caption{Mean Upper Clipping Bound.}
        \label{lu_app_fig:mean_upper_clipping_bound}
    \end{subfigure}
    \hspace{0.02\textwidth}
    \begin{subfigure}[b]{0.32\textwidth}
        \centering
        \includegraphics[width=\textwidth]{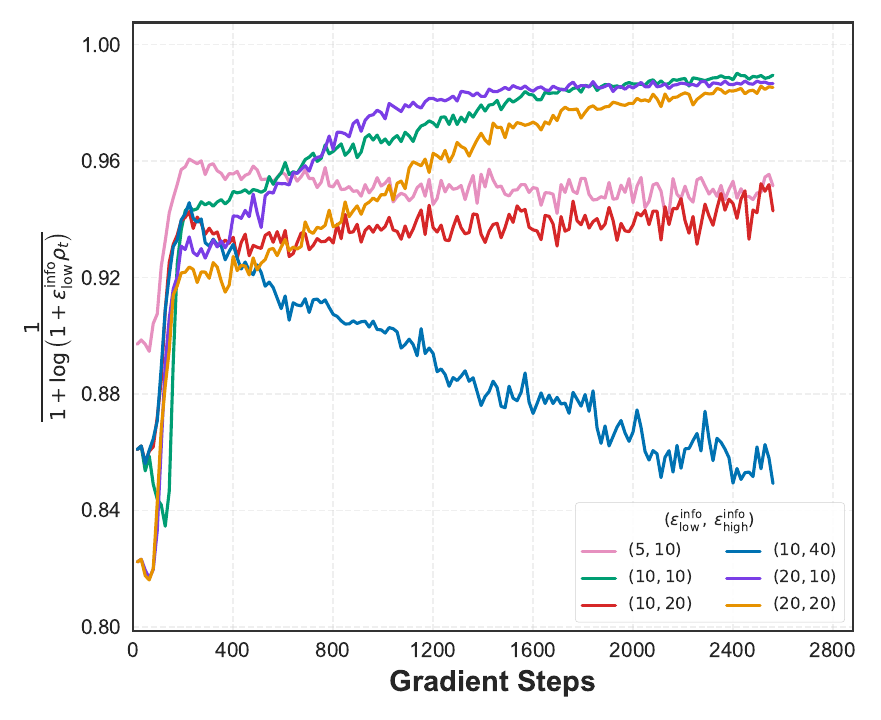}
        \vspace{-0.05in}
        \caption{Mean Lower Clipping Bound.}
        \label{lu_app_fig:mean_lower_clipping_bound}
    \end{subfigure}%
}
    \caption{
        Ablation of $\epsilon_{\mathrm{low}}^{\mathrm{info}}$ and
$\epsilon_{\mathrm{high}}^{\mathrm{info}}$ on the Qwen3-4B Base Model.
    }
    \label{app_fig:low_upper_ablation}

    \vspace{-0.15in}
\end{figure*}

\paragraph{Ablation of Adaptive Clipping Parameters.}
Figure~\ref{app_fig:low_upper_ablation} compares different combinations of
$\epsilon_{\mathrm{low}}^{\mathrm{info}}$ and
$\epsilon_{\mathrm{high}}^{\mathrm{info}}$ on the Qwen3-4B Base Model.
Increasing $\epsilon_{\mathrm{high}}^{\mathrm{info}}$ widens the upper
clipping range and is generally associated with higher predictive entropy,
whereas increasing $\epsilon_{\mathrm{low}}^{\mathrm{info}}$ lowers the
lower clipping bound and tends to reduce predictive entropy over training.
Configurations such as $(5,10)$ and $(10,20)$, where
$\epsilon_{\mathrm{high}}^{\mathrm{info}}$ is approximately twice
$\epsilon_{\mathrm{low}}^{\mathrm{info}}$, exhibit comparatively stable
entropy dynamics while maintaining competitive reasoning performance.
Such stability is consistent with the moderate variation of the information
clock considered in our theoretical analysis, and may also help preserve a
favorable balance between exploration and exploitation during RL optimization. Response length remains broadly stable across the tested clipping configurations,
with substantially smaller variation than that observed in the $\gamma$ ablation
in Figure~\ref{fig:ablation_study_gamma}. This suggests that $\gamma$ provides
a more direct and effective control knob for modulating response length. Considering the
overall reasoning performance and training dynamics, we use
$\epsilon_{\mathrm{low}}^{\mathrm{info}}=10$ and
$\epsilon_{\mathrm{high}}^{\mathrm{info}}=20$ in the main experiments.

\section{Additional Results on Llama Models}
\label{app:additional_experiments_llama}

\begin{table*}[t]
\centering
\caption{Performance comparison of \methodabb{} with PPO on Llama-3.2-3B Instruct model. For each prompt, we independently sample 16 responses and report the average accuracy, denoted as Mean@16.}
\label{tab:llama_results}

\begingroup
\setlength{\tabcolsep}{10pt}
\renewcommand{\arraystretch}{1.12}

\begin{tabular}{
l
cc
>{\columncolor[HTML]{D7E8E8}}c
>{\columncolor[HTML]{D7E8E8}}c
}
\toprule

\multirow{2}{*}{\textbf{Benchmark}}
& \multicolumn{2}{c}{\textbf{PPO}}
& \multicolumn{2}{c}{\textbf{\methodabb{}}} \\

\cmidrule(lr){2-3}
\cmidrule(lr){4-5}

& \textbf{Score} & \textbf{Len.}
& \textbf{Score} & \textbf{Len.} \\

\midrule

\multicolumn{5}{c}{\cellcolor[HTML]{EAF0FC}\textit{Llama-3.2-3B Instruct Model}} \\
\midrule

AMC23
& 39.2 & 648.1
& 42.3 & 511.9 \\

AIME24
& 11.9 & 906.0
& 14.0 & 516.4 \\

AIME25
& 0.0 & 665.2
& 0.8 & 374.7 \\

AIME26
& 0.2 & 542.8
& 0.6 & 365.7 \\

BeyondAIME
& 0.6 & 593.3
& 1.9 & 396.1 \\

\midrule
\textbf{Average}
& 10.4 & 671.1
& 11.9 & 433.0 \\

\bottomrule
\end{tabular}

\endgroup
\end{table*}

\begin{figure*}[ht]
    \centering

    \begin{subfigure}[b]{0.315\textwidth}
        \centering
        \includegraphics[width=\linewidth]{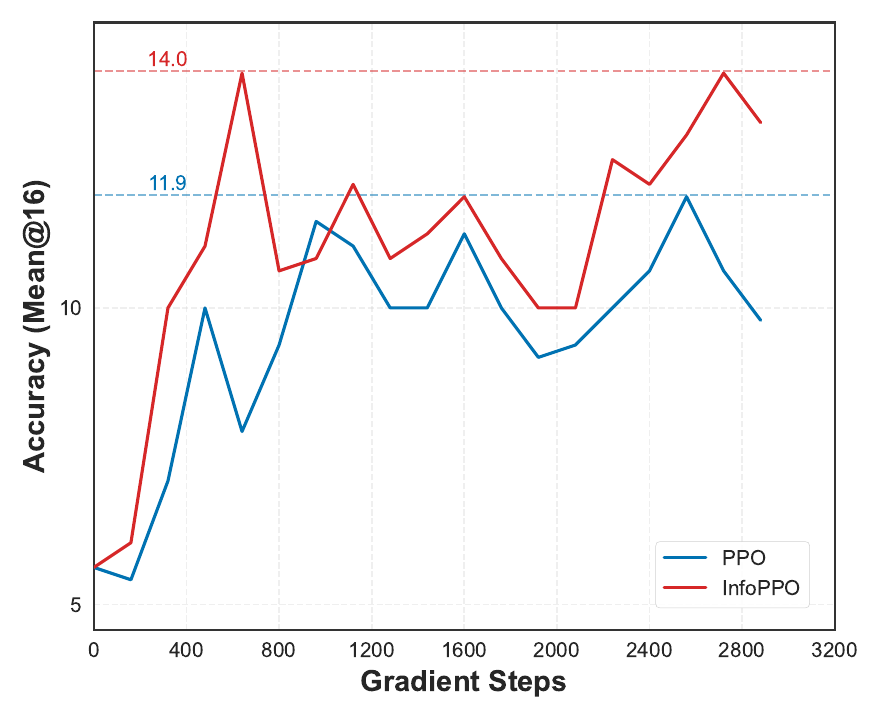}
        \vspace{-0.05in}
        \caption{AIME24 Accuracy.}
        \label{llama_app_fig:info_time_aime24}
    \end{subfigure}
    \hfill
    \begin{subfigure}[b]{0.315\textwidth}
        \centering
        \includegraphics[width=\linewidth]{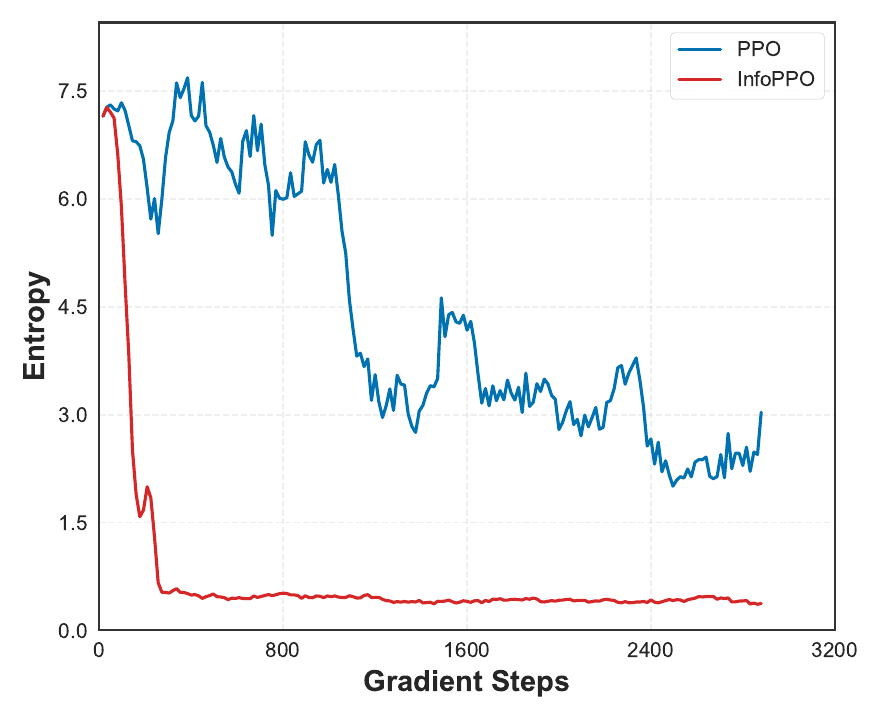}
        \vspace{-0.05in}
        \caption{Raw Predictive Entropy.}
        \label{llama_app_fig:info_time_entropy}
    \end{subfigure}
    \hfill
    \begin{subfigure}[b]{0.315\textwidth}
        \centering
        \includegraphics[width=\linewidth]{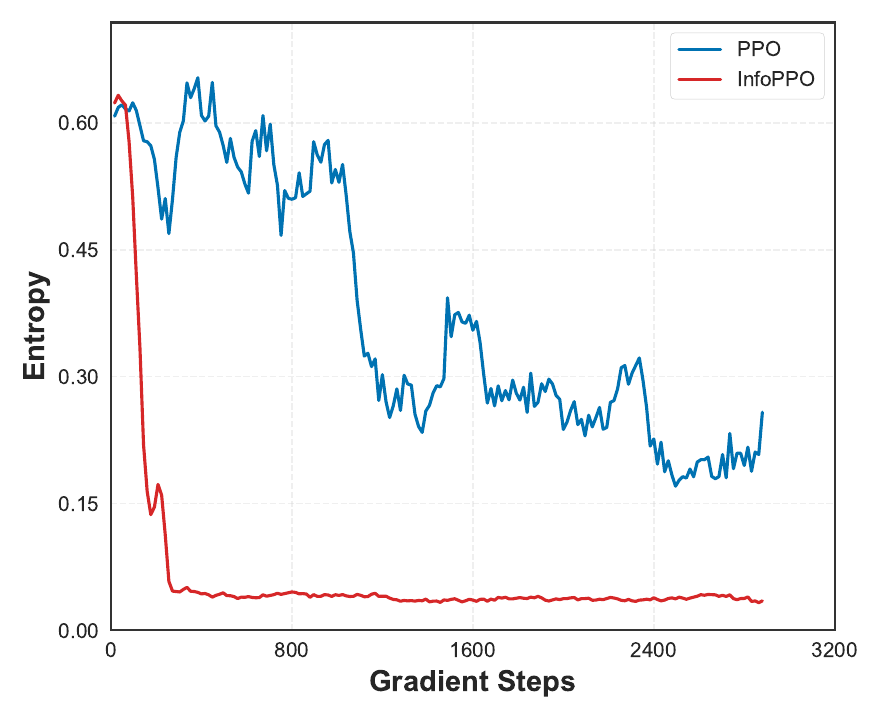}
        \vspace{-0.05in}
        \caption{Normalized Entropy.}
        \label{llama_app_fig:info_time_norm_entropy}
    \end{subfigure}

    \par\vspace{0.08in}
    \noindent

    \begin{subfigure}[b]{0.315\textwidth}
        \centering
        \includegraphics[width=\linewidth]{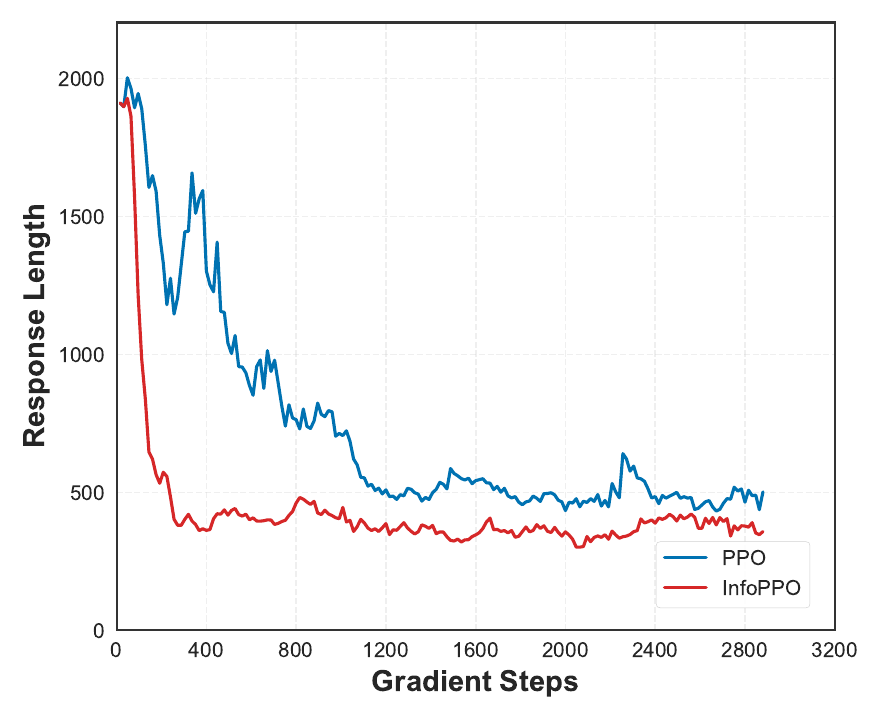}
        \vspace{-0.05in}
        \caption{Response Length.}
        \label{llama_app_fig:info_time_response}
    \end{subfigure}
    \hfill
    \begin{subfigure}[b]{0.315\textwidth}
        \centering
        \includegraphics[width=\linewidth]{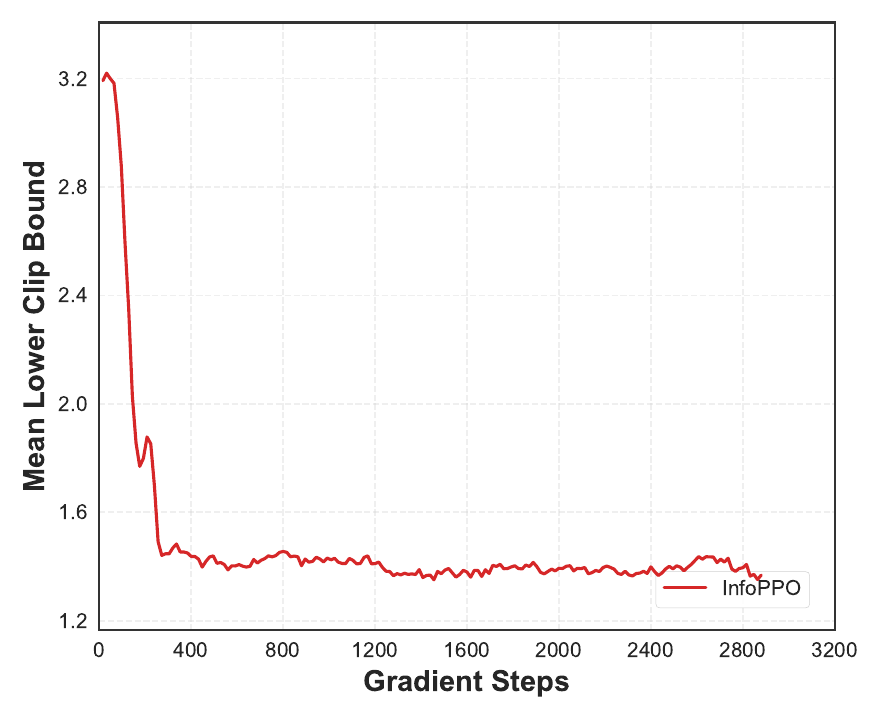}
        \vspace{-0.05in}
        \caption{Mean Upper Clipping Bound.}
        \label{llama_app_fig:info_time_upper_bound}
    \end{subfigure}
    \hfill
    \begin{subfigure}[b]{0.315\textwidth}
        \centering
        \includegraphics[width=\linewidth]{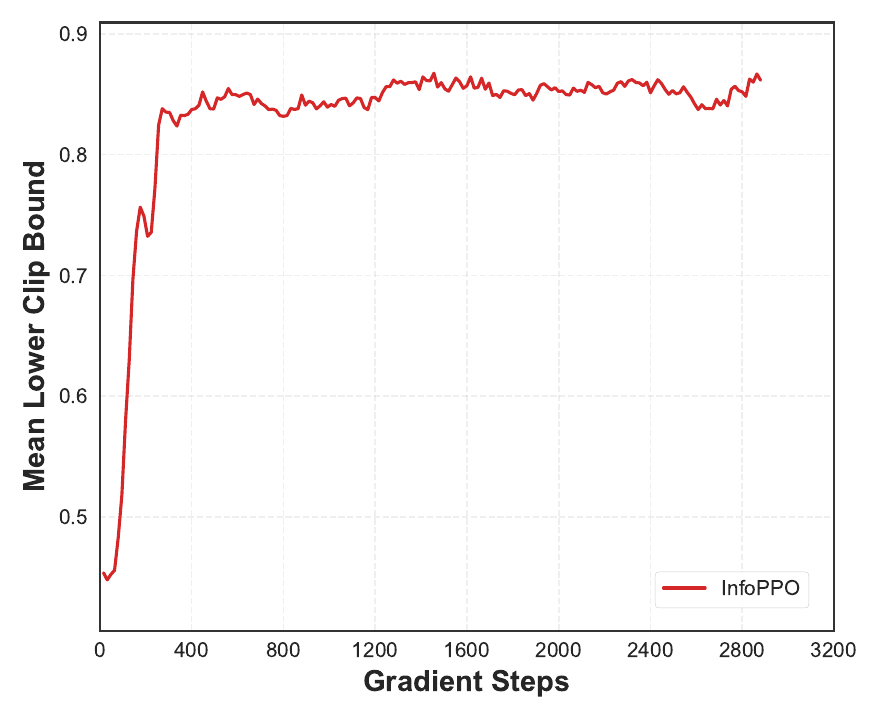}
        \vspace{-0.05in}
        \caption{Mean Lower Clipping Bound.}
        \label{llama_app_fig:info_time_lower_bound}
    \end{subfigure}

    \caption{
        Training dynamics on the Llama-3.2-3B Instruct Model.
    }
    \label{app_fig:llama_training_dynamics}

    \vspace{-0.15in}
\end{figure*}

To further examine whether the benefits of \methodabb{} extend beyond the
Qwen3 model family, we additionally evaluate PPO and \methodabb{} on the
Llama-3.2-3B Instruct Model under the same overall experimental protocol as
the main experiments.

Table~\ref{tab:llama_results} summarizes the final evaluation results.
\methodabb{} achieves clear improvements over PPO on AMC23 and AIME24,
while also obtaining higher accuracy on BeyondAIME. On AIME25 and AIME26,
both methods achieve very low absolute accuracy, making the comparison less informative on these particularly
challenging benchmarks. 

Figure~\ref{app_fig:llama_training_dynamics} further shows the training
dynamics on Llama-3.2-3B Instruct. On AIME24, \methodabb{} reaches a higher
final accuracy than PPO while exhibiting stable predictive-entropy and adaptive
clipping dynamics after the initial stage of training. Notably, both methods
show a substantial decrease in response length during training. This response-length contraction
may therefore be partly associated with the training dynamics of the
Llama-3.2-3B Instruct Model itself, since a similar and sustained decline is
also observed under PPO with $\gamma=1$. Taken together, these results provide
additional evidence that the optimization behavior of \methodabb{} remains
well behaved on a different model family.

\section{Additional Experimental Results}
\label{app:additional_experiments}

\begin{figure*}[t]
\vskip -0.2in
\begin{center}
\centerline{\includegraphics[width=1.2\textwidth]{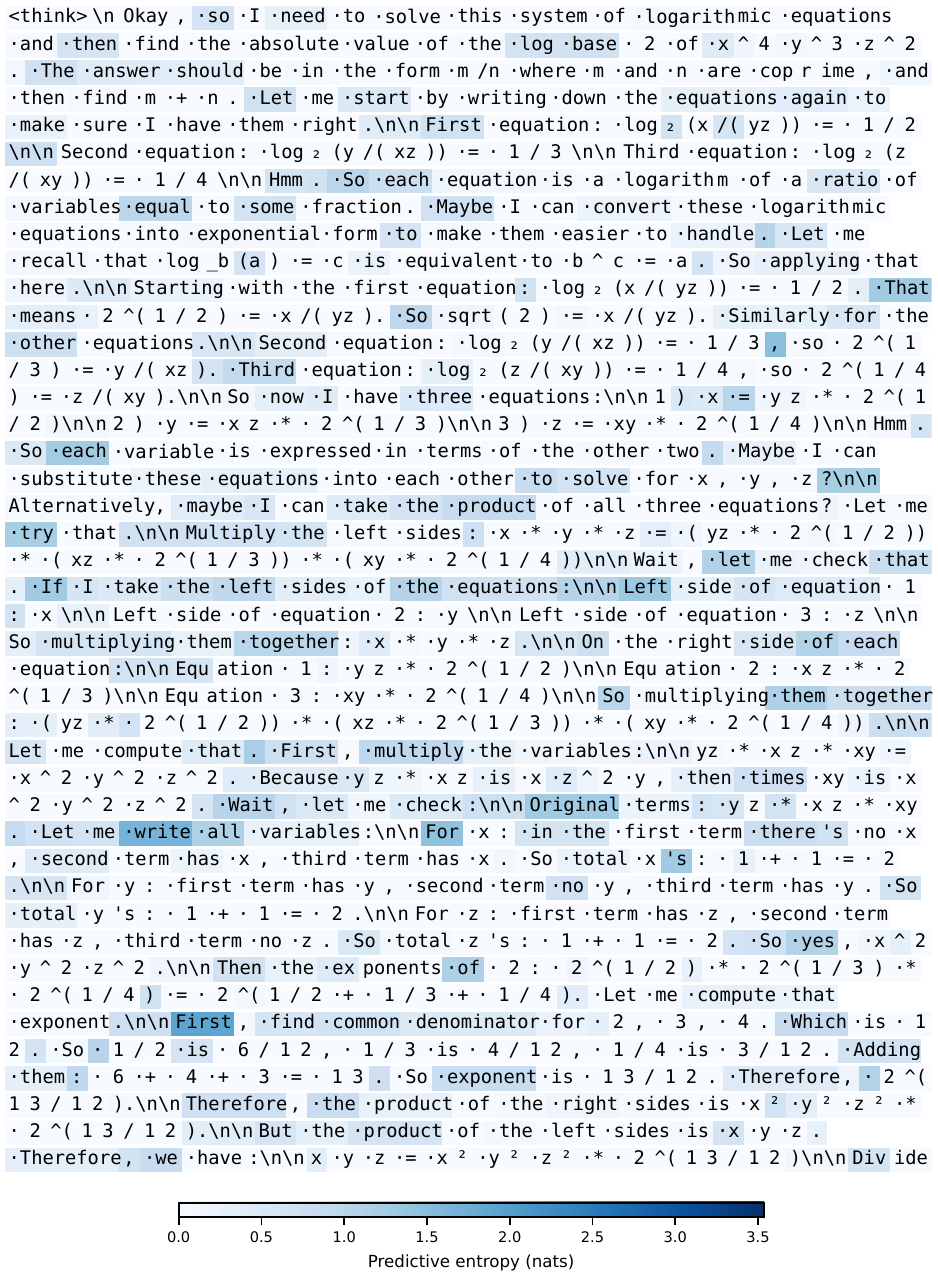}}
\caption{State-wise predictive entropy along a reasoning trajectory (Part 1/6).}
\label{app_fig:res_entropy_1}
\end{center}
\end{figure*}

\begin{figure*}[t]
\vskip -0.2in
\begin{center}
\centerline{\includegraphics[width=1.2\textwidth]{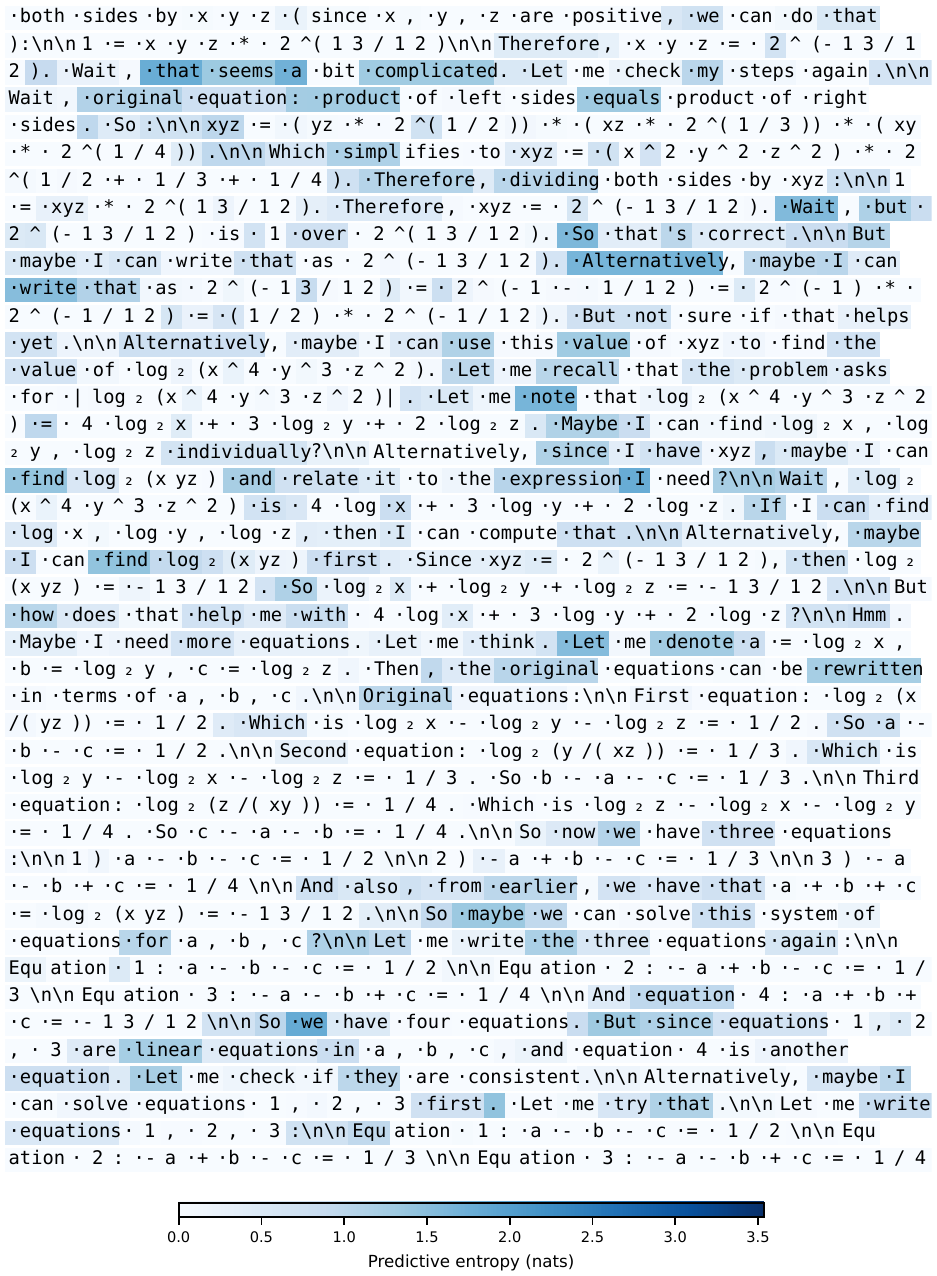}}
\caption{State-wise predictive entropy along a reasoning trajectory (Part 2/6).}
\label{app_fig:res_entropy_2}
\end{center}
\end{figure*}

\begin{figure*}[t]
\vskip -0.2in
\begin{center}
\centerline{\includegraphics[width=1.2\textwidth]{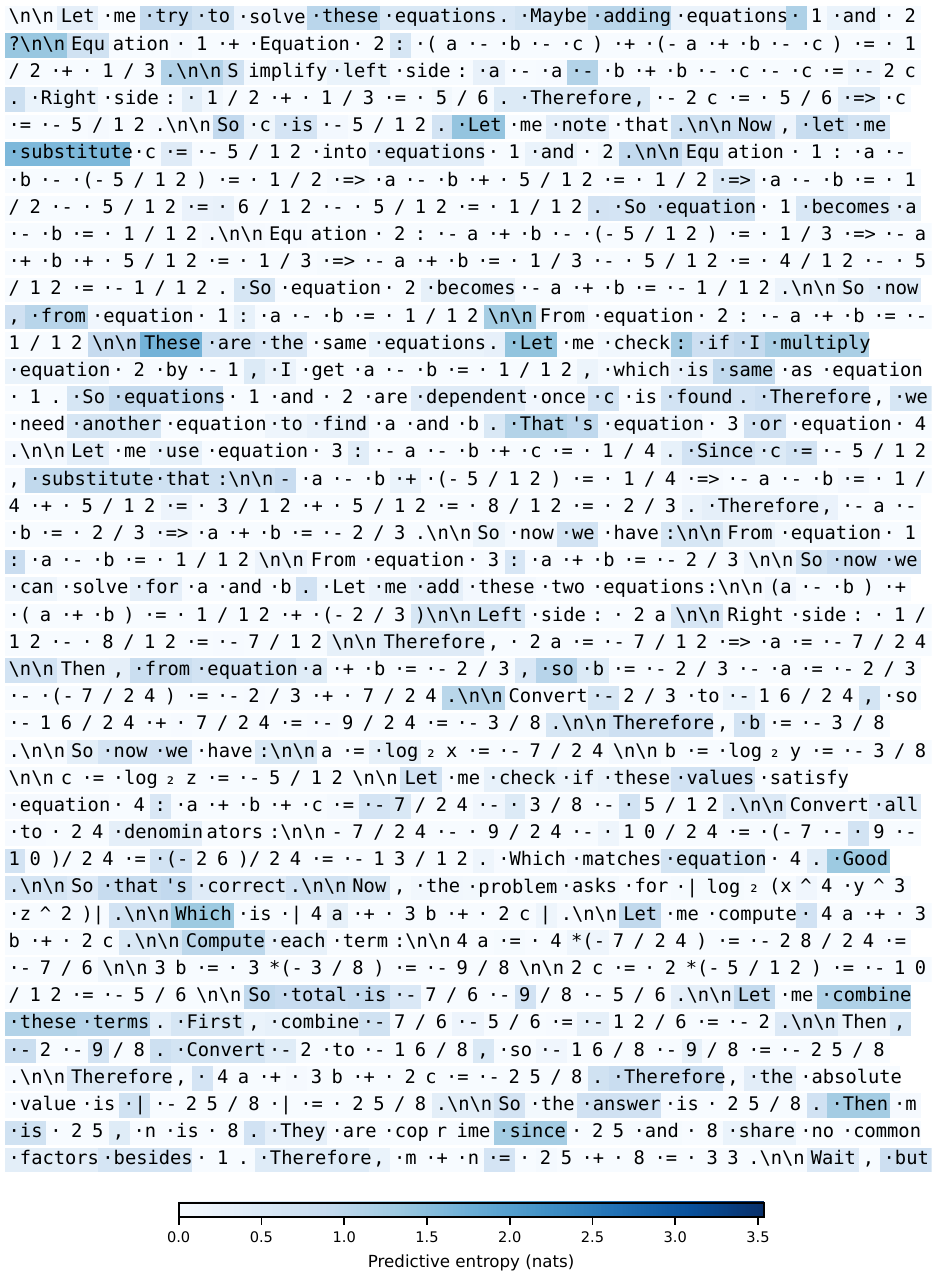}}
\caption{State-wise predictive entropy along a reasoning trajectory (Part 3/6).}
\label{app_fig:res_entropy_3}
\end{center}
\end{figure*}

\begin{figure*}[t]
\vskip -0.2in
\begin{center}
\centerline{\includegraphics[width=1.2\textwidth]{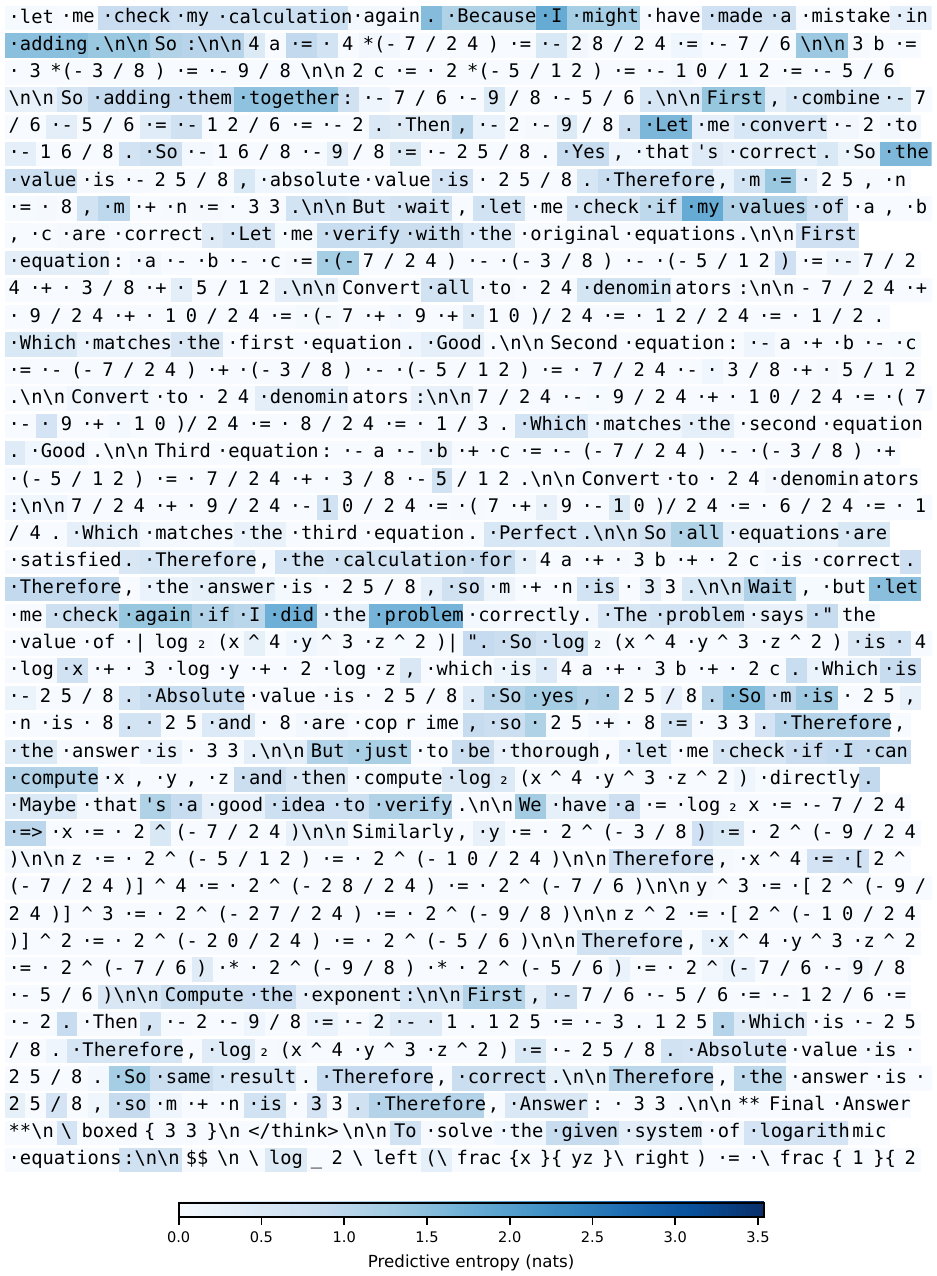}}
\caption{State-wise predictive entropy along a reasoning trajectory (Part 4/6).}
\label{app_fig:res_entropy_4}
\end{center}
\end{figure*}

\begin{figure*}[t]
\vskip -0.2in
\begin{center}
\centerline{\includegraphics[width=1.2\textwidth]{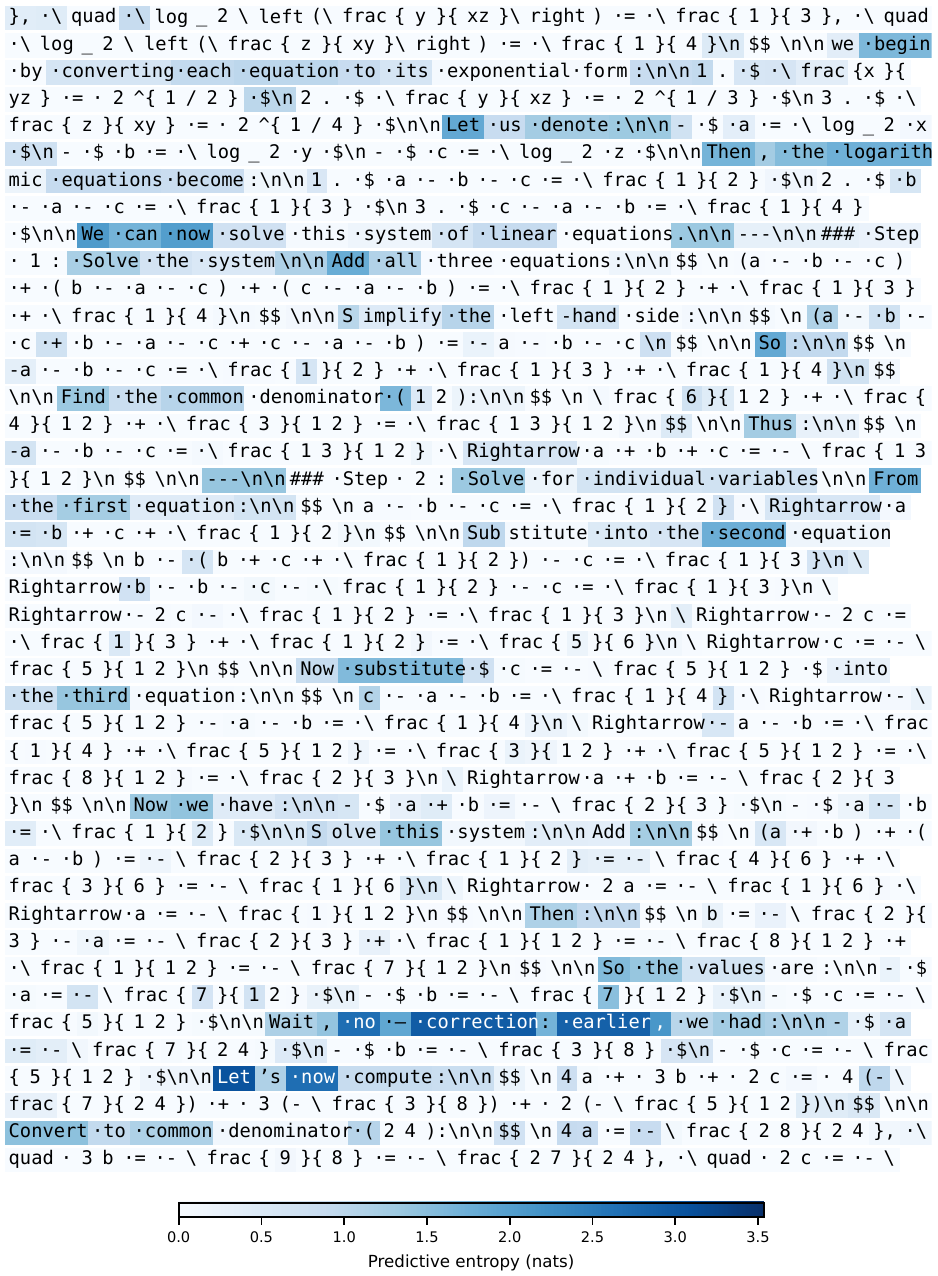}}
\caption{State-wise predictive entropy along a reasoning trajectory (Part 5/6).}
\label{app_fig:res_entropy_5}
\end{center}
\end{figure*}

\begin{figure*}[!t]
\vskip -0.2in
\begin{center}
\centerline{\includegraphics[width=1.2\textwidth]{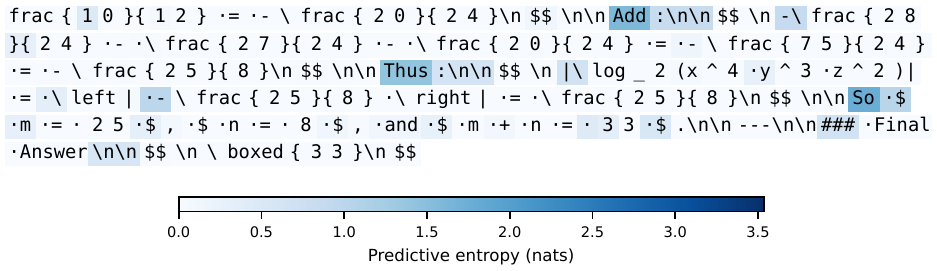}}
\caption{State-wise predictive entropy along a reasoning trajectory (Part 6/6).}
\label{app_fig:res_entropy_6}
\end{center}
\end{figure*}

\end{document}